\documentclass[onecolumn,12pt,letterpaper]{IEEEtran}
\usepackage[margin=1in]{geometry}
\usepackage{etoolbox}
\newtoggle{conference}
\togglefalse{conference} % Set to true for conference

\makeatletter
\newcommand{\pushright}[1]{\ifmeasuring@#1\else\omit\hfill$\displaystyle#1$\fi\ignorespaces}
\newcommand{\pushleft}[1]{\ifmeasuring@#1\else\omit$\displaystyle#1$\hfill\fi\ignorespaces}
\makeatother

\usepackage[utf8]{inputenc}
\usepackage[T1]{fontenc}
\usepackage{url}
\usepackage{ifthen}
\usepackage{cite}
\usepackage[cmex10]{amsmath} % Use the [cmex10] option to ensure complicance
\usepackage{hyperref}       % hyperlinks
\usepackage{booktabs}       % professional-quality tables
\usepackage{amsfonts}       % blackboard math symbols
\usepackage{nicefrac}       % compact symbols for 1/2, etc.
\usepackage{microtype}      % microtypography

\usepackage{amssymb, bm, epsfig, psfrag,amsthm}
\usepackage{algorithm,algorithmic}
\usepackage{bbm}
\usepackage[usenames]{color}
\usepackage{tikz}
\usepackage{enumerate}
\usetikzlibrary{shapes,arrows}
\usetikzlibrary{decorations.pathreplacing,calc,shapes.misc,shapes.callouts,shapes.arrows,positioning}
\usetikzlibrary{patterns}
\usepackage{url,hyperref}
\usepackage{xspace}
\usepackage{enumerate}
\usepackage{mathtools}
\usepackage{epstopdf}
\usepackage{sidecap}
\usepackage{graphicx}
\usepackage[normalem]{ulem} % provides strikeout command \sout

\newcommand{\citep}[1]{\cite{#1}}

\def\beq{\begin{equation}}
\def\eeq{\end{equation}}
\def\beqa{\begin{eqnarray}}
\def\eeqa{\end{eqnarray}}
\def\beqan{\begin{eqnarray*}}
\def\eeqan{\end{eqnarray*}}

\def\R{{\mathbb{R}}}

\DeclareMathOperator*{\argmin}{arg\,min}
\DeclareMathOperator*{\argmax}{arg\,max}
\DeclareMathOperator{\diag}{Diag}

\def\x{\bm{x}}

\newtheorem{theorem}{Theorem}
\newtheorem{lemma}{Lemma}

\theoremstyle{definition}

\newtheorem{definition}{Definition}
\newtheorem{assumption}{Assumption}

\def\phat{\wh{\p}}
\def\qhat{\wh{\q}}

\def\xhat{\wh{\x}}

\def\zhat{\widehat{\z}}

\def\arr{\rightarrow}
\def\Exp{\mathbb{E}}

\def\Cov{\mathrm{Cov}}

\def\alphabar{\overline{\alpha}}
\def\etabar{\overline{\eta}}
\def\gammabar{\overline{\gamma}}

\def\km{k\! - \!}
\def\kp{k\! + \!}
\def\lp{\ell\! + \!}
\def\lm{\ell\! - \!}
\def\Lm{L\! - \!}

\newcommand{\zero}{\mathbf{0}}

\newcommand{\bbf}{\mathbf{b}}

\newcommand{\fbf}{\mathbf{f}}
\newcommand{\gbf}{\mathbf{g}}

\newcommand{\pbf}{\mathbf{p}}
\newcommand{\pbfhat}{\widehat{\mathbf{p}}}
\newcommand{\qbf}{\mathbf{q}}
\newcommand{\qbfhat}{\widehat{\mathbf{q}}}
\newcommand{\rbf}{\mathbf{r}}

\newcommand{\sbf}{\mathbf{s}}

\newcommand{\wbf}{\mathbf{w}}

\newcommand{\xbf}{\mathbf{x}}

\newcommand{\ybf}{\mathbf{y}}

\newcommand{\zbf}{\mathbf{z}}

\newcommand{\zbfhat}{\widehat{\mathbf{z}}}
\newcommand{\Abf}{\mathbf{A}}
\newcommand{\Bbf}{\mathbf{B}}

\newcommand{\Gbf}{\mathbf{G}}

\newcommand{\Kbf}{\mathbf{K}}

\newcommand{\Pbf}{\mathbf{P}}

\newcommand{\Qbf}{\mathbf{Q}}
\newcommand{\Rbf}{\mathbf{R}}

\newcommand{\Ubf}{\mathbf{U}}

\newcommand{\Vbf}{\mathbf{V}}

\newcommand{\Wbf}{\mathbf{W}}

\def\betabf{{\boldsymbol \beta}}

\def\lambdabar{\overline{\lambda}}
\def\Lambdabar{\overline{\Lambda}}

\def\xibf{{\boldsymbol \xi}}

\def\varphibf{{\boldsymbol \varphi}}

\def\Psibf{\boldsymbol{\Psi}}

\newcommand{\thetabar}{{\overline{\theta}}}

\newcommand{\phibf}{{\bm{\phi}}}

\newcommand{\mubar}{\overline{\mu}}

\newcommand{\tran}{^{\text{\sf T}}}

\def\eqd{\stackrel{d}{=}}
\def\PLeq{\stackrel{PL(2)}{=}}
\def\Norm{{\mathcal N}}
\def\Range{\mathrm{Range}}

\def\alphabar{\overline{\alpha}}

\newcommand*\dif{\mathop{}\!\mathrm{d}} % differential in integrals
\newcommand{\bkt}[1]{{\left< #1 \right>}}
\def\Gset{\mathfrak{G}}
\def\Gsetbar{\overline{\mathfrak{G}}}

\RequirePackage{xcolor,bm,setspace}
\RequirePackage{mathrsfs}

\providecommand{\old}[1]{ }

\providecommand{\mmse}{\text{\sf mmse}}
\providecommand{\map}{\text{\sf map}}
\providecommand{\mest}{\text{\sf m-est}}

\providecommand{\mc}{\mathcal}
\providecommand{\ie}{\rm i.e.}
\providecommand{\T}{^\top}
\providecommand{\wb}{\overline}
\providecommand{\wt}{\widetilde}
\providecommand{\wh}{\widehat}
\providecommand{\norm}[1]{\left\|#1\right\|}
\providecommand{\Real}{\mathbb{R}}

\providecommand{\xhat}{\wh{\x}}
\providecommand{\zhat}{\wh{\z}}
\providecommand{\phat}{\wh{\p}}
\providecommand{\qhat}{\wh{\q}}

\providecommand{\i}{\bm{i}}
\providecommand{\j}{\bm{j}}

\providecommand{\p}{\mathbf{p}}
\providecommand{\q}{\mathbf{q}}
\renewcommand{\r}{\bm{r}}
\providecommand{\s}{\bm{s}}

\providecommand{\w}{\mathbf{w}}
\providecommand{\x}{\bm{x}}
\providecommand{\y}{\bm{y}}
\providecommand{\z}{\mathbf{z}}

\providecommand{\A}{\bm{A}}
\providecommand{\I}{\bm{I}}

\providecommand{\U}{\bm{U}}
\providecommand{\V}{\mathbf{V}}
\providecommand{\W}{\mathbf{W}}

\providecommand{\bbf}{\mathbf{b}}

\providecommand{\fbf}{\mathbf{f}}
\providecommand{\gbf}{\mathbf{g}}

\providecommand{\pbf}{\mathbf{p}}
\providecommand{\qbf}{\mathbf{q}}
\providecommand{\rbf}{\mathbf{r}}
\providecommand{\sbf}{\mathbf{s}}

\providecommand{\wbf}{\mathbf{w}}
\providecommand{\xbf}{\mathbf{x}}
\providecommand{\ybf}{\mathbf{y}}
\providecommand{\zbf}{\mathbf{z}}
\providecommand{\Abf}{\mathbf{A}}
\providecommand{\Bbf}{\mathbf{B}}

\providecommand{\Gbf}{\mathbf{G}}

\providecommand{\Kbf}{\mathbf{K}}

\providecommand{\Pbf}{\mathbf{P}}
\providecommand{\Qbf}{\mathbf{Q}}
\providecommand{\Rbf}{\mathbf{R}}

\providecommand{\Ubf}{\mathbf{U}}
\providecommand{\Vbf}{\mathbf{V}}
\providecommand{\Wbf}{\mathbf{W}}

\providecommand{\phibf}{\mbf{\phi}}

\providecommand{\mcN}{\mathcal{N}}

\title{Inference with Deep Generative Priors\\ in High Dimensions}

{
    \author{
    Parthe Pandit, Mojtaba Sahraee-Ardakan, Sundeep Rangan, Philip Schniter, and Alyson~K.~Fletcher
        \thanks{P. Pandit, M. Sahraee-Ardakan, and A.~K.~Fletcher
        (email: \{parthepandit,msahraee,akfletcher\}@ucla.edu) are with
        the Departments of Statistics and Electrical and Computer Engineering,
        the University of California, Los Angeles, CA, 90095\@.
        Their work was supported in part by the National Science Foundation under Grants 1254204 and 1738286, and the Office of Naval Research under Grant N00014-15-1-2677\@.
        S. Rangan (email: srangan@nyu.edu) is with
        the Department of Electrical and Computer Engineering,
        New York University, Brooklyn, NY, 11201\@.
        His work was supported in part by the National Science Foundation
        under Grants 1116589, 1302336, and 1547332,
        as well as the industrial affiliates of NYU WIRELESS\@.
        P. Schniter (email: schniter.1@osu.edu) is with 
        the Department of Electrical and Computer Engineering,
        The Ohio State University, Columbus, OH, 43210\@.
        His work was supported in part by the National Science Foundation
        under Grant 1716388.}
        \thanks{Portions of this paper were presented at the IEEE International Symposium on Information Theory in 2018 \cite{fletcher2018inference} and 2019 \cite{pandit2019asymptotics}.}
    }
}

\begin{document}

% %% Required by JSAIT
\doublespacing
% %%

\maketitle

%%%%%% START sections/abstract %%%%%%%%%%%%%

\begin{abstract}
Deep generative priors offer powerful models for complex-structured data, such as images, audio, and text.  
Using these priors in inverse problems typically requires estimating 
the input and/or hidden signals in a multi-layer deep neural network from observation of its output. 
While these approaches have been successful in practice, rigorous performance analysis 
is complicated by the non-convex nature of the underlying optimization problems.
This paper presents a novel algorithm, 
Multi-Layer Vector Approximate Message Passing (ML-VAMP), for inference in multi-layer stochastic neural networks.
\mbox{ML-VAMP} can be configured to compute maximum a priori (MAP) or 
approximate minimum mean-squared error (MMSE) estimates for these networks.
We show that the performance of ML-VAMP can be exactly predicted 
in a certain high-dimensional random limit.
Furthermore, under certain conditions, ML-VAMP yields estimates that achieve 
the minimum (i.e., Bayes-optimal) MSE as predicted by the replica method.
In this way, ML-VAMP provides a computationally efficient method for
multi-layer inference with an exact performance characterization and 
testable conditions for optimality in the large-system limit. 
\end{abstract}

%%%%%% END sections/abstract %%%%%%%%%%%%%

%%%%%% START sections/intro %%%%%%%%%%%%%
\section{Introduction}\label{sec:intro}

\subsection{Inference with Deep Generative Priors}
We consider inference in an $L$-layer  stochastic neural network of the form
\begin{subequations}  \label{eq:nntrue}
\begin{align}
    \z^0_\ell &= \Wbf_{\ell}\z^0_{\lm1} + \bbf_{\ell} + \xibf_\ell,
    \quad &\ell&=1,3,\ldots,\Lm1,
        \label{eq:nnlintrue} \\
    \z^0_{\ell} &= \phibf_\ell(\z^0_{\lm1},\xibf_{\ell}), \quad &\ell& = 2,4,\ldots,L,
        \label{eq:nnnonlintrue} 
%  \text{Estimate} & \text{$\{\zhat_\ell\}_{\ell=0}^{L-1}$ given $\z_L^0,\{\Wbf_{2\ell-1},\bbf_{2\ell-1},\phibf_{2\ell}\}_{\ell=1}^{L/2}$, and i.i.d.\ priors on $\z_0^0$ and $\{\xibf_\ell\}_{\ell=0}^{L-1}$}
\end{align}
\end{subequations}
where $\z^0_0$ is the network input, 
$\{\z^0_\ell\}_{\ell=1}^{\Lm1}$ are hidden-layer signals, and
$\y:=\z^0_L$ is the network output.  
The odd-indexed layers \eqref{eq:nnlintrue} are (fully connected) affine linear layers
with weights $\Wbf_\ell$, biases $\bbf_\ell$, and additive noise vectors $\xibf_\ell$.
The even-indexed layers  \eqref{eq:nnnonlintrue} involve separable and possibly nonlinear functions $\phibf_\ell$ that are randomized%
\footnote{The role of the noise $\xi_{\ell,i}$ in $\phi_\ell$ is allowed to be generic (e.g., additive, multiplicative, etc.). The relationship between $z^0_{\ell,i}$ and $z^0_{\lm1,i}$ will be modeled using the conditional density $p(z^0_{\ell,i}|z^0_{\lm1,i}) = \int \delta\big(z^0_{\ell,i}-\phi_\ell(z^0_{\lm1,i},\xi_{\ell,i})\big) p(\xi_{\ell,i}) \dif\xi_{\ell,i}$.}
by the noise vectors $\xibf_\ell$.
By ``separable,'' we mean that $[\phibf_\ell(\z,\xibf)]_i=\phi_\ell(z_i,\xi_i)~\forall i$, where $\phi_\ell$ is some scalar-valued function, such as a sigmoid or ReLU,
and where $z_i$ and $\xi_i$ represent the $i$th component of $\z$ and $\xibf$.
We assume that the input $\z^0_0$ and noise vectors $\xibf_\ell$ are mutually independent, that each contains i.i.d.\ entries, and that the number of layers, $L$, is even.
A block diagram of the network is shown in the top panel of Fig.~\ref{fig:nn_ml_vamp}.
The \emph{inference problem} is to estimate the input and hidden signals $\{\z_\ell\}_{\ell=0}^{\Lm1}$ from an observation of the network output $\y$.
That is,
\begin{align}
    \text{Estimate $\{\z_{\ell}\}_{\ell=0}^{L-1}$ given  $\y$ and $\{\Wbf_{2k-1},\bbf_{2k-1},\phibf_{2k}\}_{k=1}^{L/2}$.}
    \label{eq:problem} 
\end{align}
For inference, we will assume that network parameters (i.e., the weights $\Wbf_\ell$, biases $\bbf_\ell$, and activation functions $\phibf_\ell$) are all known, as are  the distributions of the input $\z^0_0$ and the noise terms $\xibf_\ell$.
Hence, we do \emph{not} consider the network learning problem.
The superscript ``$0$'' on $\z^0_\ell$ indicates that this is the ``true" value of $\z_\ell$, to be distinguished from the estimates of $\z_\ell$ produced during inference denoted by $\zhat_\ell$.

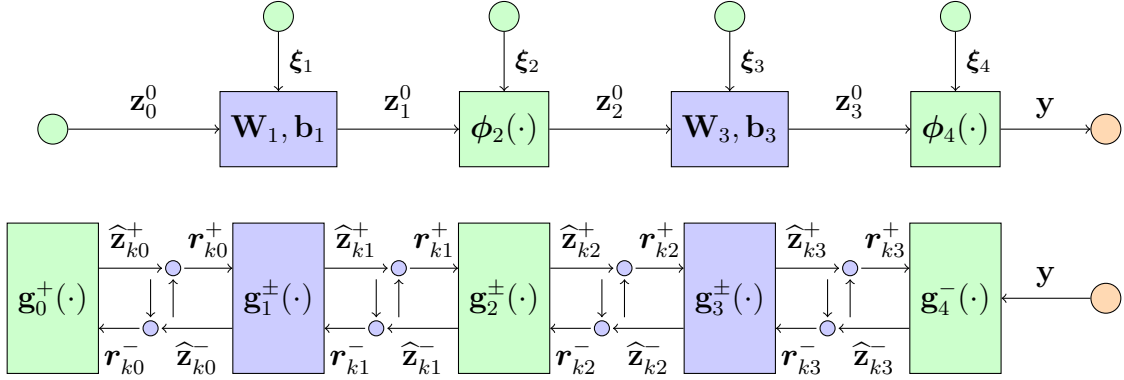
\begin{figure*}
\centering

%%%%%%% START tikz/feedforward  %%%%%%
\begin{tikzpicture}

    \pgfmathsetmacro{\sep}{3};
    \pgfmathsetmacro{\yoff}{0.4};
    \pgfmathsetmacro{\xoffa}{0.3};
    \pgfmathsetmacro{\xoffb}{0.6};

    \tikzstyle{var}=[draw,circle,fill=green!20,node distance=2.5cm];
    \tikzstyle{yvar}=[draw,circle,fill=orange!30,node distance=2.5cm];
    \tikzstyle{conn}=[draw,circle,fill=green!40,radius=0.02cm];
    \tikzstyle{linest}=[draw,fill=blue!20,minimum size=1cm,
        minimum height=2cm, node distance=2.25cm]
    \tikzstyle{nlest}=[draw,fill=green!20,minimum size=1cm,
        minimum height=2cm, node distance=2.25cm]
    \tikzstyle{linblock}=[draw,fill=blue!20, minimum size=1cm, node distance=\sep cm];
    \tikzstyle{nlblock}=[draw,fill=green!20, minimum size=1cm, node distance=\sep cm];

    % Neural net
    \node [var] (z0) {};
    \node [linblock, right of=z0] (W1) {$\Wbf_1,\bbf_1$};
    \node [nlblock, right of=W1] (phi2) {$\phibf_2(\cdot)$};
    \node [linblock, right of=phi2] (W3) {$\Wbf_3,\bbf_3$};
    \node [nlblock, right of=W3] (phi4) {$\phibf_4(\cdot)$};
    \node [yvar,right of=phi4, node distance=2cm] (y) {};
    % \xi var nodes
    \node [var, above of=W1, yshift=-1cm] (xi1) {};
    \node [var, above of=phi2, yshift=-1cm] (xi2) {};
    \node [var, above of=W3, yshift=-1cm] (xi3) {};
    \node [var, above of=phi4, yshift=-1cm] (xi4) {};
    % Draw arrows
    \path[->] (z0) edge  node [above] {$\z^0_0$} (W1);
    \path[->] (W1) edge  node [above] {$\z^0_1$} (phi2);
    \path[->] (phi2) edge  node [above] {$\z^0_2$} (W3);
    \path[->] (W3) edge  node [above] {$\z^0_3$} (phi4);
    \path[->] (phi4) edge  node [above] {$\ybf$} (y);
    % \xi arrows
    \path[->] (xi1) edge node [right] {\small$\xibf_1$} (W1);
    \path[->] (xi2) edge node [right] {\small$\xibf_2$} (phi2);
    \path[->] (xi3) edge node [right] {\small$\xibf_3$} (W3);
    \path[->] (xi4) edge node [right] {\small$\xibf_4$} (phi4);

    % ML-VAMP message passing
    \node [nlest,below of=z0] (h0) {$\gbf^+_0(\cdot)$};
    \node [linest,below of=W1] (h1) {$\gbf^{\pm}_1(\cdot)$};
    \node [nlest,below of=phi2] (h2) {$\gbf^{\pm}_2(\cdot)$};
    \node [linest,below of=W3] (h3) {$\gbf^{\pm}_3(\cdot)$};
    \node [nlest,below of=phi4] (h4) {$\gbf^-_4(\cdot)$};
    \node [yvar,right of=h4, node distance=2cm] (y2) {};
    \path [draw,->] (y2) edge node [above] {$\ybf$} (h4.east);

    \foreach \i/\j in {0/1,1/2,2/3,3/4} {
        \node [right of=h\i,xshift=\xoffa cm,yshift=\yoff cm] (conn0\i) {};
        \node [right of=h\i,xshift=\xoffb cm,yshift=\yoff cm] (conn1\i) {};
        \node [right of=h\i,xshift=\xoffa cm,yshift=-\yoff cm] (conn2\i) {};
        \node [right of=h\i,xshift=\xoffb cm,yshift=-\yoff cm] (conn3\i) {};
        %\filldraw [blue] (conn0\i) circle (0.1cm);
        \draw [fill=blue!20] (conn1\i) circle (0.1cm);
        \draw [fill=blue!20] (conn2\i) circle (0.1cm);
        %\draw [blue] (conn3\i) circle (0.1cm);

        \path [draw,->] (conn0\i) -- (conn2\i);
        \path [draw,->] (conn3\i) -- (conn1\i);

        \path[->] ([yshift=\yoff cm]h\i.east) edge node  [above]
            {$\widehat{\z}^+_{k\i}$} (conn1\i);
        \path[->] (conn1\i) edge node  [above]
            {$\r^+_{k\i}$} ([yshift=\yoff cm]h\j.west);
        \path[->] ([yshift=-\yoff cm]h\j.west) edge node [below]
            {$\widehat{\z}^-_{k\i}$} (conn2\i);
        \path[->] (conn2\i) edge node  [below]
            {$~\r^-_{k\i}$} ([yshift=-\yoff cm]h\i.east);

    }
        %\node [conn,right of=h0,xshift=\xoffb cm,yshift=\yoff cm] (connp01) {};

    %\path[->] ([yshift=\lineoff cm]h0.east) edge node (c0) [above]
    %    {$\r^+_{k0}$} ([yshift=\yoff cm]h1.west);
    %\path[->] ([yshift=-\lineoff cm]h1.west) edge node [below]
    %    {$\r^-_{k0}$} ([yshift=-\lineoff cm]h0.east);
    %\path[->] ([yshift=\lineoff cm]h1.east) edge node [above]
    %    {$\r^+_{k1}$} ([yshift=\lineoff cm]h2.west);
    %\path[->] ([yshift=-\lineoff cm]h2.west) edge node [below]
    %    {$\r^-_{k1}$} ([yshift=-\lineoff cm]h1.east);
    %\path[->] ([yshift=\lineoff cm]h2.east) edge node [above]
    %    {$\r^+_{k2}$} ([yshift=\lineoff cm]h3.west);
    %\path[->] ([yshift=-\lineoff cm]h3.west) edge node [below]
    %    {$\r^-_{k2}$} ([yshift=-\lineoff cm]h2.east);
    %\path[->] ([yshift=\lineoff cm]h3.east) edge node [above]
    %    {$\r^+_{k3}$} ([yshift=\lineoff cm]h4.west);
    %\path[->] ([yshift=-\lineoff cm]h4.west) edge node [below]
    %    {$\r^-_{k3}$} ([yshift=-\lineoff cm]h3.east);

\end{tikzpicture}

%%%%%%% END tikz/feedforward  %%%%%%

\caption{Top panel: Feedfoward neural network mapping an input $\z_0$ to output $\y=\z^0_{4}$ in the case of $L=4$ layers.
Bottom panel: ML-VAMP estimation functions $\gbf_\ell^\pm(\cdot)$ and estimation quantities $\r_{k\ell}^\pm$ and $\zhat_{k\ell}^\pm$ at iteration $k$.} 
%algorithm for recovering estimates for the input and hidden states from the output $\y$.}
\label{fig:nn_ml_vamp}
\end{figure*}

The inference problem \eqref{eq:problem} arises in the following state-of-the-art approach to inverse problems.
In general, solving an ``inverse problem" means recovering some signal $\x$ from a measurement $\y$ that depends on $\x$.
For example, in compressed sensing (CS) \cite{eldar2012compressed}, the measurements are often modeled as $\y=\Abf\x+\xibf$ with known $\Abf$ and additive white Gaussian noise (AWGN) $\xibf$, and the signal is often modeled as a sparse linear combination of elements from a known dictionary, i.e., $\x=\Psibf\zbf$ for some sparse coefficient vector $\zbf$.
To recover $\x$, one usually computes a sparse coefficient estimate $\zbfhat$ using a LASSO-type convex optimization  \cite{tibshirani1996regression} and then uses it to form a signal estimate $\xhat$, as in
\begin{equation}
\xhat = \Psibf\zbfhat 
\quad\text{for}\quad
\zbfhat = \arg\min_{\zbf} \left\{ \frac{1}{2}\|\y-\Abf\Psibf\zbf\|^2 + \lambda\|\zbf\|_1 \right\}
\label{eq:cs} ,
\end{equation}
where $\lambda>0$ is a tunable parameter.
The CS recovery approach \eqref{eq:cs} can be interpreted as a \emph{two-layer} version of the inference problem: 
the first layer implements signal generation via $\x=\Psibf\zbf$, while the second layer implements the measurement process $\y=\Abf\zbf + \xibf$. 
Equation~\eqref{eq:cs} then performs maximum a posteriori inference (see the discussion around \eqref{eq:map_estimator}) to recover estimates of $\zbf$ and $\x$.
%Note that $\cbfhat$ in \eqref{eq:cs} is the MAP estimate of $\cbf$ under an i.i.d.\ Laplacian prior.

Although CS has met with some success, it has a limited ability to exploit the complex structure of natural signals, such as images, audio, and video.
This is because the model ``$\x=\Psibf\zbf$ with sparse $\zbf$'' is overly simplistic; it is a \emph{one-layer} generative model.
Much more sophisticated modeling is possible with multi-layer priors, as demonstrated in recent works on variational autoencoders (VAEs) \citep{rezende2014stochastic,kingma2013auto},
generative adversarial networks (GANs) \cite{radford2015unsupervised,salakhutdinov2015learning},
and deep image priors (DIP) \cite{ulyanov2018deep,van2018compressed}.
These models have had tremendous success in modeling richly structured data, such as images and text. 

A typical application of solving an inverse problem using a deep generative model is shown in Fig.~\ref{fig:inpaint_mod}.
This figure considers the classic problem of \emph{inpainting} \cite{bertalmio2000image}, for which reconstruction with DIP has been particularly successful \cite{yeh2016semantic,bora2017compressed}.
Here, a noise-like signal $\zbf_0^0$ drives a three-layer generative network to produce an image $\x^0$.
%In this latter approach, data with a complex structure, such as an image $\x^0$, is represented as the output of a multi-layer generative network (i.e., the first three layers in Fig.~\ref{fig:inpaint_mod}).   
The generative network would have been trained on an ensemble of images similar to the one being estimated using, e.g., VAE or GAN techniques.
The measurement process, which manifests as occlusion in the inpainting problem, is modeled using one additional layer of the network, which produces the measurement $\y$.
Inference is then used to recover the image $\x^0$ (i.e., the hidden-layer signal $\zbf_3^0$) from $\y$. 
%In this case, the inference can be used to estimate the original image ($\x^0$) from the occluded version $\y$, which is precisely the inpainting problem in this case.
In addition to inpainting, this deep-reconstruction approach can be applied to other \emph{linear} inverse problems (e.g., CS, de-blurring, and super-resolution) as well as  \emph{generalized-linear} \cite{mccullagh1989generalized} inverse problems (e.g., classification, phase retrieval, and estimation from quantized outputs).
We note that the inference approach provides an alternative to designing and training a separate reconstruction network, such as in \cite{mousavi2015deep,metzler2017learned,borgerding2017amp}.

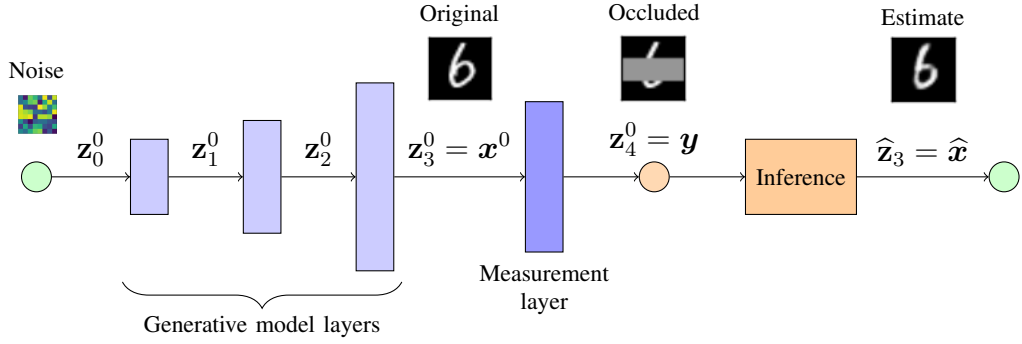
\begin{figure*}
\centering

%%%%%%% START tikz/inpaint  %%%%%%%%

\begin{tikzpicture}

    \pgfmathsetmacro{\sep}{3};
    \pgfmathsetmacro{\yoff}{0.4};
    \pgfmathsetmacro{\xoffa}{0.3};
    \pgfmathsetmacro{\xoffb}{0.6};

    \tikzstyle{var}=[draw,circle,fill=green!20,node distance=1.5cm];
    \tikzstyle{yvar}=[draw,circle,fill=orange!30,node distance=1cm];
    \tikzstyle{encblock}=[draw,fill=blue!20, node distance=1.5 cm];
    \tikzstyle{measblock}=[draw,fill=blue!40, node distance=1.5 cm];
    \tikzstyle{infblock}=[draw,fill=orange!40, node distance=1.5 cm];

    % Neural net
    \node [var] (z0) {};
    \node [encblock, right of=z0, minimum width=0.5cm,
        minimum height=1cm] (W1) {};
    \node [encblock, right of=W1, minimum width=0.5cm,
        minimum height=1.5cm] (W2) {};
    \node [encblock, right of=W2, minimum width=0.5cm,
        minimum height=2.5cm] (W3) {};
    \node [measblock, right of=W3, minimum width=0.5cm,
        minimum height=2cm, node distance=2.25 cm] (W4) {};
    \node [yvar,right=1 cm of W4] (y) {};
    \node [infblock,right=1 cm of y, minimum height=1cm] (inf) {
        \footnotesize Inference};
    \node [var,right=1.75cm of inf] (zhat) {};

    % Draw arrows
    \path[->] (z0) edge  node [above] {$\z^0_0$} (W1);
    \path[->] (W1) edge  node [above] {$\z^0_1$} (W2);
    \path[->] (W2) edge  node [above] {$\z^0_2$} (W3);
    \path[->] (W3) edge  node (z3) [above] {$\z^0_3=\x^0$} (W4);
    \path[->] (W4) edge  node [above] {} (y);
    \path[->] (y) edge  node [above] {} (inf);
    \path[->] (inf) edge  node (zhat3) [above] {$\zhat_3 = \xhat$} (zhat);

    % Draw pictures
    \node (ytext) [above of=y, yshift=-0.5 cm] {$\z_4^0=\y$};
    \node [above of=z0] {
        \begin{tabular}{c} \footnotesize Noise \\
        \includegraphics[width=1cm]{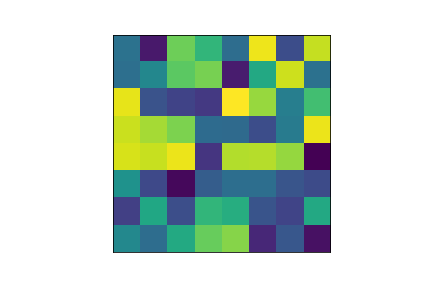}
        \end{tabular} };
    \node [above of=z3, yshift=0.25cm](Original) {
        \begin{tabular}{c} \footnotesize Original \\
        \includegraphics[width=1cm]{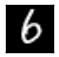}
        \end{tabular} };
    \node [ ](c) at (ytext|- Original){
        \begin{tabular}{c} \footnotesize Occluded \\
        \includegraphics[width=1cm]{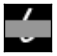}
        \end{tabular} };
    \node [above of=zhat3,yshift=0.25cm] {
    \begin{tabular}{c} \footnotesize Estimate \\
        \includegraphics[width=1cm]{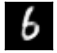}
        \end{tabular} };
        
    % Text
    \node (W1below) [below left =1.25 cm and 0.1 cm of W1.west] {};
    \node (W3below) [below right =1.25 cm and 0.1 cm of W3.east] {};
    \draw  [decorate,yshift=-1cm,
    decoration={brace,amplitude=10pt,mirror}]
        (W1below) -- (W3below)  node[black,midway,yshift=-0.6cm] {\footnotesize Generative model layers};
    %\node [below of=W2,yshift=-0.5cm] 
    %    {\footnotesize \begin{tabular}{c}
    %    Generative model \\ layers \end{tabular}};
    \node [below of=W4,yshift=-0.5cm] 
        {\footnotesize \begin{tabular}{c}
        Measurement \\ layer \end{tabular}};
\end{tikzpicture}

%%%%%%% END tikz/inpaint  %%%%%%%%

\caption{Motivating example:  Inference for inpainting 
\cite{yeh2016semantic,bora2017compressed}.
An image $\x^0$ is modeled as the output of a generative model
driven by white noise $\z_0^0$, and an occluded measurement $\y$ is generated by one additional layer.  Inference is then used to recover
the image $\x$ from the measurement $\y$.}
\label{fig:inpaint_mod}
\end{figure*}

When using deterministic deep generative models, the unknown signal $\x^0$ can be modeled as 
$\x^0 = \mc G(\zbf_0^0)$,
where $\mc G$ is a trained deep neural network and $\z_0^0$ is a realization 
of an i.i.d.\ random vector, typically with a Gaussian distribution.
Consequently, to recover $\x^0$ from a linear-AWGN measurement of the form 
$\y=\Abf\x^0+\xibf$, the compressed-sensing approach in \eqref{eq:cs} can be extended to 
a regularized least-squares problem~\cite{chang2017one} of the form
\begin{align}
\xhat=\mc G(\zhat_0)
\quad\text{for}\quad
\zhat_0 := \arg\min_{\z_0}\ \left\{ \frac{1}{2}\norm{\y-\A\mc G(\z_0)}^2 + \lambda \norm{\z_0}^2 \right\}
\label{eq:likelihood_min} .
\end{align}
%Note that $\zhat_0$ in \eqref{eq:likelihood_min} is the maximum a priori (MAP)
%estimate of $\z_0^0$ under the likelihood $p(\y|\z_0^0)=\Norm(\y;\A\mc{G}(\z_0^0),\lambda\Ibf)$.
In practice, the optimization in \eqref{eq:likelihood_min} is solved using a gradient-based method. 
This approach can be straightforwardly implemented with deep-learning software packages and has been used, with excellent results, in \cite{yeh2016semantic,bora2017compressed,hand2017global,kabkab2018task,shah2018solving,tripathi2018correction,mixon2018sunlayer}.
The minimization \eqref{eq:likelihood_min} has also been useful in interpreting the semantic meaning of hidden signals in deep networks \cite{mahendran2015understanding,yosinski2015understanding}.
VAEs \citep{rezende2014stochastic,kingma2013auto} and certain GANs \cite{dumoulin2016adversarially} can also produce decoding networks that sample from the posterior density,
and sampling methods such as Markov-chain Monte Carlo (MCMC) algorithms and Langevin diffusion \cite{cheng2018sharp, welling2011bayesian} can also be employed.

\subsection{Analysis via Approximate Message Passing (AMP)}
While reconstruction with deep generative priors has seen tremendous practical success, its performance is not fully understood.  
Optimization approaches such as \eqref{eq:likelihood_min} are typically non-convex and difficult to analyze.
As we discuss below, most results available today only provide bounds, and these bounds are often be overly conservative (see Section~\ref{sec:prior}).

Given a network architecture and statistics on the unknown signals, fundamental information-theoretic questions include: 
What are the precise limits on the accuracy of estimating the hidden signals $\{\zbf^0_\ell\}_{\ell=0}^{L-1}$ from the measurements $\y$?
How well do current estimation methods perform relative to these limits? 
Is is possible to design computationally efficient yet optimal methods?

To answer these questions, this paper considers deep inference via approximate message passing (AMP), a powerful approach for analyzing estimation problems in certain high-dimensional random settings.  
Since its origins in understanding linear inverse problems in compressed sensing
\cite{DonohoMM:09,DonohoMM:10-ITW1},
AMP has been extended to an impressive range of estimation and learning tasks, including 
generalized linear models \cite{rangan2011generalized},
models with parametric uncertainty \cite{fletcher2017learning}, 
structured priors\cite{fletcher2018plug}, and 
bilinear problems\cite{sarkar2019bilinear}.  
For these problems, AMP-based methods have been able to provide computationally efficient algorithms with precise high-dimensional analyses.
Often, AMP approaches yield optimality guarantees in cases where all other known approaches do not.

\subsection{Main Contributions}
In this work, we develop a  multi-layer version of a AMP for inference in deep networks. 
The proposed approach builds on the recent vector AMP (VAMP) method 
of \cite{rangan2019vamp}, which is itself closely related to expectation propagation (EP) \cite{minka2001expectation,takeuchi2017rigorous}, expectation-consistent approximate inference (EC) \cite{opper2005expectation,fletcher2016expectation}, S-AMP \cite{cakmak2014samp}, and orthogonal AMP \cite{ma2017orthogonal}.
The proposed method is called \emph{multi-layer VAMP}, or ML-VAMP.  
As will be described in detail below, ML-VAMP estimates the hidden signals in a deep network by cycling through a set of relatively simple \emph{estimation functions} $\{\gbf^\pm_\ell\}_{\ell=0}^{L}$.
The information flow in ML-VAMP is shown in the bottom panel of  Fig.~\ref{fig:nn_ml_vamp}.
The ML-VAMP method is similar to the multi-layer AMP method of \cite{manoel2017multi} but can handle a more general class of matrices in the linear layers.
In addition, as we will describe below,
%through selection of the estimation functions,
the proposed ML-VAMP algorithm can be configured for either MAP or MMSE estimation.
We will call these approaches MAP-ML-VAMP and MMSE-ML-VAMP.

We establish several key results on the ML-VAMP algorithm:
\begin{itemize}
    \item We show that, for both MAP and MMSE inference, the fixed points of the ML-VAMP algorithm correspond to stationary points of variational formulations of these estimators. This allows the interpretation of ML-VAMP as a Lagrangian algorithm with adaptive step-sizes in both cases. These findings are given in Theorems \ref{thm:mapfix} and \ref{thm:fixed_point_mmse} and are
    similar to previous results for AMP \cite{krzakala2014variational,rangan2016fixed}. Section \ref{sec:fixed} describes these results.
    
    \item We prove that, in a certain large system limit (LSL), the behavior of ML-VAMP is exactly described by a deterministic recursion called the \emph{state evolution} (SE).
    This SE analysis is a multi-layer extension of similar results \cite{BayatiM:11,javanmard2013state,rangan2019vamp} for AMP and VAMP.
    The SE equations enable \emph{asymptotically exact} predictions of macroscopic behaviors of the hidden-layer estimates for \emph{each iteration} of the ML-VAMP algorithm. This allows us to obtain error bounds even if the algorithm is run for a finite number of iterations. The SE analysis, given in Theorem \ref{thm:main_result}, is the main contribution of the paper, and is discussed in Section \ref{sec:seevo}. 
    
    \item Since the original conference versions of this paper \cite{fletcher2018inference,pandit2019asymptotics}, formulae for the minimum mean-squared error (MMSE) for inference in deep networks have been conjectured in 
    \cite{reeves2017additivity,gabrie2018entropy,barbier2019optimal}. As discussed in Section~\ref{sec:replica}, these formulae are based on heuristic techniques, such as the replica method from statistical physics, and have been rigorously proven in special cases \cite{reeves2016replica}.  
    Remarkably, we show that the mean-squared-error (MSE) of ML-VAMP exactly matches the predicted MMSE in certain cases.
    
    \item Using numerical simulations, we verify the predictions of the main result from Theorem \ref{thm:main_result}. In particular, we show that the SE accurately predicts the MSE even for networks that are not considered large by today's standards.
    We also perform experiments with the MNIST handwritten digit dataset. Here we consider the inference problem using learned networks, for which the weights do not satisfy the randomness assumptions required in our analysis. 
\end{itemize}

In summary, ML-VAMP provides a computationally efficient method for inference in deep networks whose performance can be exactly predicted in certain high-dimensional random settings.  Moreover, in these settings, the MSE performance of ML-VAMP can match the existing predictions of the MMSE.

\subsection{Prior Work}
\label{sec:prior}

There has been growing interest in studying learning and inference problems in high-dimensional, random settings.
One common model is the so-called \emph{wide network}, where the dimensions of the input, hidden layers, and output 
are assumed to grow with a fixed linear scaling, and the weight matrices are modeled as realizations of random matrices.
This viewpoint has been taken in \cite{neal2012bayesian,giryes2016deep,hanin2018start,choromanska2015loss}, 
in several works that explicitly use AMP methods \cite{manoel2017multi,li2018random,gabrie2018entropy,reeves2017additivity},
and in several works that use closely related random-matrix techniques \cite{schoenholz2016deep,novak2018bayesian}.

The existing work most closely related to ours is that by Manoel et al.~\cite{manoel2017multi}, 
which developed a multi-layer version of the original AMP algorithm \cite{DonohoMM:09}.  
The work \cite{manoel2017multi} provides a state-evolution analysis of multi-layer inference 
in networks with entrywise i.i.d.\ Gaussian weight matrices.  
%Similar to the VAMP analysis in \cite{rangan2019vamp},
In contrast, our results apply to the larger class of rotationally invariant matrices 
(see Section~\ref{sec:seevo} for details), which includes i.i.d.\ Gaussian matrices case as a special case.  

Several other recent works have also attempted to characterize the performance of reconstruction using deep priors in random settings.
For example, when $\z_0^0\in\Real^k$ and $\A\in\Real^{m\times n}$ is a realization of an i.i.d.\ Gaussian matrix with $m=\Omega(kL\log n)$, Bora et al.~\cite{bora2017compressed} showed that an $L$-layer network $\mc G$ with ReLU activations can provide provably good reconstruction of $\x^0\in{\rm Range}(\mc G)$ from measurements $\y=\A\x^0+\xibf$.
For the same problem,  \cite{hand2017global} and \cite{huang2018provably} show that, for $\Wbf_\ell\in\Real^{N_\ell\times N_{\ell-1}}$ generated entrywise i.i.d.\ Gaussian and $N_\ell = \Omega(N_{\ell-1}\log N_{\ell-1})$, one can derive bounds on reconstruction error that hold with high probability under similar conditions on $m$. 
Furthermore, they also show that the cost function of \eqref{eq:likelihood_min} has stationary points in only two disjoint regions of the $\z_0$ space, and both are closely related to the true solution $\z_0^0$. 
In \cite{lei2019inverting}, the authors use a layer-wise reconstruction scheme to prove reconstruction error bounds when $N_\ell = \Omega(N_{\ell-1})$, i.e., the network is expansive, but with a constant factor as opposed to the logarithmic factor in \cite{huang2018provably}. 

Our results, in comparison, provide an asymptotically exact characterization of the reconstruction error---not just bounds.
Moreover, our results hold for arbitrary hidden-dimension ratios $N_\ell/N_{\ell-1}$, which can be less than, equal to, or greater than one.
On the other hand, our results hold only in the large-system limit, whereas the other results above hold in the finite-dimensional regime. 
Nevertheless, we think that it should be possible to derive a finite-dimensional version of our analysis (in the spirit of \cite{rush2018finite}) that holds with high probability. 
Also, our experimental results suggest that our large-system-limit analysis is a good approximation of behavior at moderate dimensions.

Some of the material in this paper appeared in conference versions \cite{fletcher2018inference,pandit2019asymptotics}.
The current paper includes all the proofs, simulation details, and provides a unified treatment of both MAP and MMSE estimation.

%%%%%% END sections/intro %%%%%%%%%%%%%

%%%%%% START sections/mlvamp_algorithm %%%%%%%%%%%%%
\section{Multi-layer Vector Approximate Message Passing}

\subsection{Problem Formulation}
We consider inference in a probabilistic setting where,
in \eqref{eq:nntrue}, $\zbf^0_0$ and $\xibf_\ell$ are modeled as random vectors with known densities.  
Due to the Markovian structure of $\{\z_\ell\}$ in \eqref{eq:nntrue},
%, which is evident from Fig.~\ref{fig:nn_ml_vamp} (TOP), 
the posterior distribution $p(\z|\y)$, where $\z:=\{\z_0\}_{\ell=0}^{L-1}$, factorizes as
\begin{align}\label{eq:posterior_factorization}
    p(\z|\y) 
    \propto p(\z,\y) 
    = p(\z,\z_L) 
    = p(\z_0)\prod_{\ell=1}^L p(\z_\ell|\z_{\ell-1}),
\end{align}
where the form of $p(\z_\ell|\z_{\ell-1})$ is determined by $\Wbf_\ell$, $\bbf_\ell$, and the distribution of $\xibf_\ell$ for odd $\ell$; and by $\phibf_\ell$ and the distribution of $\xibf_\ell$ for even $\ell$.
We will assume that $\z_\ell\in\Real^{N_\ell}$, where $N_\ell$ can vary across the layers $\ell$.

Similar to other graphical-model methods \cite{wainwright2008graphical}, we consider
two forms of estimation:  MAP estimation and MMSE estimation.  
The \emph{maximum a priori}, or \textbf{MAP}, estimate is defined as 
\beq
    \zhat_{\map} := \arg\max_{\z}\ p(\z|\y) 
    \label{eq:map_estimator}.
\eeq
Although we will focus on MAP estimation,
most of our results will apply to general $M$-estimators \cite{huber2011robust} of the form,
\begin{align*}
    \zhat_{\mest} := \arg\min_{\z}  \left\{
    \mathscr{L}_0(\z_0) +
    \sum_{\ell=1}^{L} \mathscr{L}_\ell(\z_\ell,\z_{\ell-1}) \right\}
\end{align*}
for loss functions $\mathscr{L}_\ell$.  The MAP estimator corresponds to the
loss function $\mathscr{L}_\ell = -\ln p(\z_\ell|\z_{\ell-1})$.

We will also consider the minimum mean-squared error, or \textbf{MMSE}, estimate, defined as
\begin{align}
    \zhat_{\mmse} &:= \Exp[\z|\y] = \int \z\, p(\z|\y) \dif\z
    \label{eq:mmse_estimator} .
\end{align}
To compute the MMSE estimate, we first compute the posterior marginals $p(\z_\ell|\ybf)$.  
We will also be interested in estimating the posterior marginals $p(\z_\ell|\ybf)$.
From estimates of the posterior marginals, one also compute other estimates, 
such as the mininum mean-absolute error (MMAE) estimate, i.e., the median of the posterior marginal.

%%%%%%%%%%%%%%%%%%%%%%%%%%%%%%%%%%%
\begin{algorithm}[t]
\setstretch{1.1}
\caption{Multi-layer Vector Approximate Message Passing (ML-VAMP)}

%%%%%%% START algo/mlvamp  %%%%%%%%
\begin{algorithmic}[1]  \label{algo:ml-vamp}
\REQUIRE{Estimation functions $\gbf_0^+$, $\gbf_L^-$, and $\gbf_\ell^\pm$ for $\ell=1,\ldots,\Lm1$.}%
\STATE{Set $\r^-_{0\ell}=\zero$ and initialize parameters $\theta_{0\ell}^-$ for $\ell=0,1,\ldots,\Lm1$.}
\FOR{$k=0,1,\dots,N_{\rm it}-1$}

    \STATE{// \texttt{Forward Pass} }
    \STATE{$\zhat^+_{k0} = \gbf_0^+(\r^-_{k0},\theta^+_{k0})$}
        \label{line:zp0}
    \STATE{$\alpha^+_{k0} = \bkt{\partial \gbf_0^+(\r^-_{k0},\theta^+_{k0})/
            \partial \r^-_{k\ell}}$}
            \label{line:alphap0}
    \STATE{$\r^+_{k0} = (\zhat^+_{k0} - \alpha^+_{k0}\r^-_{k0})/(1-\alpha^+_{k0})$}
            \label{line:rp0}
    \FOR{$\ell=1,\ldots,\Lm1$}
        \STATE{$\zhat^+_{k\ell} =
        \gbf_\ell^+(\r^-_{k\ell},\r^+_{k,\lm1},\theta_{k\ell}^+)$}
        \label{line:zp}
        \STATE{$\alpha^+_{k\ell} = \bkt{{\partial
            \gbf_\ell^+}(\r^-_{k\ell},\r^+_{k,\lm1},\theta_{k\ell}^+)/{\partial \r^-_{\ell}}}$}
            \label{line:alphap}
        % \STATE{$\gamma^+_{k\ell} = \gamma^-_{k\ell}(\tfrac{1}{\alpha_{k\ell}^+} - 1)$}
        %     \label{line:gammap}
        \STATE{$\r^+_{k\ell} = (\zhat^+_{k\ell} - \alpha^+_{k\ell}\r^-_{k\ell})/
            (1-\alpha^+_{k\ell})$}       \label{line:rp}
    \ENDFOR
    \STATE{}

    \STATE{// \texttt{Backward Pass} }
    \STATE{$\zhat^-_{k,\Lm1} =
        \gbf_{L}^-(\r^+_{k,\Lm1},\theta^-_{kL})$}   \label{line:znL}
    \STATE{$\alpha^-_{k+1,\Lm1} = \bkt{\partial
        \gbf_{L}^-(\r^+_{k,\Lm1},\theta^-_{kL})/ \partial \r^+_{k,\Lm1}}$}
        \label{line:alphanL}
   \STATE{$\r^-_{\kp1,\Lm1} = (\zhat^-_{k,\Lm1}
            - \alpha^-_{k,{\Lm1}}\r^+_{k,\Lm1})/ (1-\alpha^-_{k,\Lm1}$)}
            \label{line:rnL}
    \FOR{$\ell=\Lm1,\ldots,1$}
        \STATE{$\zhat^-_{k,\ell-1} =
            \gbf_{\ell}^-(\r^-_{k+1,\ell},\r^+_{k,\ell-1},\theta^-_{k\ell})$}
            \label{line:zn}
        \STATE{$\alpha^-_{k+1,\ell-1} = \bkt{{\partial
        \gbf_{\ell}^-}(\r^-_{k+1,\ell},\r^+_{k,\ell-1},\theta^-_{k\ell})/{\partial \r^+_{\ell-1}}}$}
            \label{line:alphan}
        \STATE{$\r^-_{\kp1,\ell-1} = (\zhat^-_{k,\ell-1}
            - \alpha^-_{k,\ell-1}\r^+_{k,\ell-1})/ (1-\alpha^-_{k,\ell-1})$}
            \label{line:rn}
    \ENDFOR

\ENDFOR
\end{algorithmic}
%%%%%%% END algo/mlvamp  %%%%%%%%
\end{algorithm}
%%%%%%%%%%%%%%%%%%%%%%%%%%%%%%%%%%%
\setlength{\abovedisplayskip}{6pt}
\setlength{\belowdisplayskip}{6pt}

\subsection{The ML-VAMP Algorithm}

Similar to the generalized EC (GEC) \cite{fletcher2016expectation} and generalized VAMP \cite{schniter2016vector} algorithms, 
the ML-VAMP algorithm attempts to compute MAP or MMSE estimates using a sequence of forward-pass and backward-pass updates.
The steps of the algorithm are specified in Algorithm~\ref{algo:ml-vamp}.
The quantities updated in the forward pass are denoted by superscript $+$,
and those updated in the backward pass are denoted by superscript $-$. 
The update formulae can be derived similarly to those for the GEC algorithm \cite{fletcher2016expectation}, 
using expectation-consistent approximations of the Gibbs free energy inspired by \cite{opper2005expectation}. 
The ML-VAMP algorithm splits the estimation of $\z=\{\z_\ell\}_{\ell=1}^{L-1}$ into smaller problems 
that are solved by the \emph{estimation functions} $\{\gbf_\ell^\pm\}_{\ell=1}^{L-1}$, $\gbf_0^+$ and $\gbf_L^-$.
(See Figure~\ref{fig:nn_ml_vamp}, bottom panel.)
As described below, the form of $\gbf_\ell^\pm$ depends on whether the goal is MAP or MMSE estimation.
During the forward pass, the estimators $\gbf^+_\ell$ are invoked, whereas in the backward pass, $\gbf^-_\ell$ are invoked. 
Similarly, the ML-VAMP algorithm maintains two copies, $\zhat^+$ and $\zhat^-$, of the estimate of $\z$.
For $\ell=1,2,\ldots,\Lm1$, each pair of estimators $(\gbf_\ell^+,\gbf_\ell^-)$ takes as input 
$\r_{\ell-1}^+$ and $\r_\ell^-$ to update the estimates $\zhat_\ell^+$ and $\zhat^-_{\ell-1}$, respectively.
Similarly, $\gbf_0^+$ and $\gbf_L^-$ take inputs $\r_0^-$ and $\r_{L-1}^+$ to update $\zhat^0$ and $\zhat_{L-1}^-$, respectively.
The estimation functions also take parameters $\theta_{\ell}^\pm$.

\subsection{MAP and MMSE Estimation Functions}
The form of the estimation functions $\{\gbf_\ell^\pm\}_{\ell=0}^{L-1}$ depends on 
%the probabilistic models assumed and 
whether the goal is to perform MAP or MMSE estimation.
In either case, the parameters are given by
\beq \label{eq:thetagam}
    \theta_{k0}^+ = \gamma_{k0}^-,
    \qquad
    \theta_{k\ell}^+ =  (\gamma^-_{k\ell},\gamma_{k,\lm1}^+),
    \qquad
    \theta_{k\ell}^- = (\gamma^-_{k+1,\ell},\gamma_{k,\lm1}^+),
    \qquad
    \theta_{kL}^- = \gamma_{k,L-1}^+,
\eeq
where $\gamma_{k\ell}^\pm$ and $\eta_{k\ell}^\pm$ are scalars updated at iteration $k\geq 0$ and all $\ell=0,1,\ldots,\Lm1$ as follows:
\begin{align} \label{eq:gamupdate}
%\begin{split}
    \gamma^+_{k\ell} = \eta^+_{k\ell} - \gamma^-_{k\ell}, \qquad
    \gamma^-_{\kp1,\ell} = \eta^-_{k+1,\ell} - \gamma^+_{k\ell}, \qquad
		    \eta^+_{k\ell} = \gamma^-_{k\ell}/\alpha^+_{k\ell} \qquad
            \eta^-_{k+1,\ell} = \gamma^+_{k\ell}/\alpha^-_{k+1,\ell}.
%\end{split}
\end{align}
%Of particular interest is the operating region where $\alpha_{k\ell}^\pm\in(0,1)$ or, equivalently, $0<\gamma_{k\ell}^\pm<\eta^\pm_{k\ell}$. 
%We note that the parameter $\eta_{k\ell}^\pm$ is defined here for the purpose of analysis which we elaborate upon in Section \ref{sec:seevo}. 
%In short, for MMSE-ML-VAMP, this parameter is closely related to the mean squared error of the estimates $\zhat_{k\ell}^\pm$, as described in Theorem~\ref{thm:MMSE_SE}.
Given these parameters, both the MAP and MMSE estimation functions are defined from the \textit{belief} function
\begin{align}\label{eq:def_belief}
    \MoveEqLeft b_\ell(\z_\ell,\z_{\ell-1}|\r_\ell^-,\r_{\ell-1}^+,\gamma_\ell^-,\gamma_{\ell-1}^+)
    \propto p(\z_\ell|\z_{\ell-1})\exp(-\frac{\gamma_{\ell}^-}2\norm{\z_\ell-\r_\ell^-}^2-\frac{\gamma_{\ell-1}^+}2\norm{\z_{\ell-1}-\r_{\ell-1}^+}^2)
\end{align}
for $\ell=1,\ldots \Lm1$. 
Similarly, $b_L(\z_L,\z_{L-1})\propto p(\y|\z_{L-1})\exp(-\tfrac{\gamma_{L-1}^+}2\|\z_{L-1}-\r_{L-1}^+\|^2)$, and $b_0(\z_0,\z_{-1})\propto p(\z_0)\exp(-\tfrac{\gamma_{0}^-}2\|\z_0-\r_0^-\|^2).$
When performing MMSE inference, we use
\begin{align}\label{eq:MMSE_update}
    (\zhat_{\ell}^+,\zhat_{\ell-1}^-)_\mmse = \gbf_{\ell,\mmse}^\pm(\r_\ell^-,\r^+_{\ell-1};\gamma_\ell^-,\gamma^+_{\ell-1})= \Exp[(\z_{\ell},\z_{\ell-1})|b_\ell] ,
\end{align}
where $\Exp[\cdot|b_\ell]$ denotes expectation with respect to the distribution $b_\ell$.
Similarly, for MAP inference, we use
\begin{align}\label{eq:MAP_update}
    (\zhat_{\ell}^+,\zhat_{\ell-1}^-)_\map=\gbf_{\ell,\map}^\pm(\r_\ell^-,\r^+_{\ell-1};\gamma_\ell^-,\gamma^+_{\ell-1}) = \argmax_{\z_\ell,\z_{\ell-1}} b_\ell(\z_\ell,\z_{\ell-1}|\r_\ell^-,\r^+_{\ell-1},\gamma_\ell^-,\gamma^+_{\ell-1}) .
\end{align}
Notice that \eqref{eq:MAP_update} corresponds to the proximal operator of $-\ln p(\z_\ell|\z_{\ell-1})$. 
We will use ``MMSE-ML-VAMP'' to refer to ML-VAMP with the MMSE estimation functions \eqref{eq:MMSE_update},
and ``MAP-ML-VAMP'' to refer to ML-VAMP with the MAP estimation functions \eqref{eq:MAP_update}.

\subsection{Computational Complexity}
A key feature of the ML-VAMP algorithm is that, for the neural network \eqref{eq:nntrue}, 
the MMSE and MAP estimation functions \eqref{eq:MMSE_update} and \eqref{eq:MAP_update} 
are computationally easy to compute.
To see why, first recall that, for the even layers $\ell=2,4,\ldots L$,
the map $\phibf_\ell$ in \eqref{eq:nnnonlintrue} is assumed separable and the noise $\xibf_\ell$ is assumed i.i.d.  
As a result, $\z_\ell$ is conditionally independent given $\z_{\ell-1}$, i.e.,
$p(\z_\ell|\z_{\ell-1})=\prod_i p(z_{\ell,i}|z_{\ell-1,i})$.
Thus, for even $\ell$, the belief function $b_\ell$ in \eqref{eq:def_belief} also factors into a product of the form $b_\ell(\z_\ell,\z_{\ell-1}) = \prod_i b_\ell(z_{\ell,i},z_{\ell-1,i})$, 
implying that the MAP and MMSE versions of $\gbf_\ell^\pm$ are both coordinate-wise separable. 
In other words, the MAP and MMSE estimation functions can be computed using $N_\ell$ scalar MAP or MMSE estimators.

Next consider \eqref{eq:nnlintrue} for $\ell=1,3,\ldots,L-1$, i.e., the linear layers.
Assume that $\bm\xi_{\ell}\sim\mc N(\bm0,\I\nu_{\ell}^{-1})$ for some
precision (i.e., inverse variance) $\nu_{\ell} > 0$.  
Then $p(\z_{\ell}|\z_{\ell-1}) \propto \tfrac{\nu_{\ell}}2\norm{\z_{\ell}-\W_{\ell} \z_{\ell-1}-\bbf_{\ell}}^2$. 
In this case, the MMSE and MAP estimation functions \eqref{eq:MMSE_update} and \eqref{eq:MAP_update} are identical,
and both take the form of a standard least-squares problem.
Similar to the VAMP algorithm~\cite{rangan2019vamp}, the least-squares solution---which must be recomputed at each iteration $k$---is can be efficiently computed using a single singular value decomposition (SVD) that is computed once, before the iterations begin.
In particular, we compute the SVD
\begin{align}\label{eq:SVD}
\W_{\ell}=\V_\ell\diag(\s_\ell)\V_{\ell-1},
\end{align}
where $\V_\ell\in\Real^{N_\ell\times N_\ell}$  and $\V_{\ell-1}\in\Real^{N_{\ell-1}\times N_{\ell-1}}$ are orthogonal and 
$\diag(\s_\ell)\in\Real^{N_\ell\times N_{\ell-1}}$ is a diagonal matrix that contains the singular values of $\W_\ell$. 
Let $\overline{\bbf}_\ell:=\V_\ell\T\bbf_\ell$. 
Then for odd $\ell$, the updates \eqref{eq:MMSE_update} and \eqref{eq:MAP_update} both correspond to quadratic problems, 
which can be simplified by exploiting the rotational invariance of the $\ell_2$ norm. 
Specifically, one can derive that 
\begin{subequations}\label{eq:glintrans}
\begin{align}\label{eq:glintrans1p}
   \zhat_\ell^+=\gbf^+_\ell(\r^-_{\ell},\r^+_{\lm1},\gamma_{\ell}^-,\gamma_{\lm1}^+)&=
  \Vbf_\ell{\Gbf}^+_\ell(\Vbf_{\ell}\T\r^-_{\ell},\Vbf_{\ell-1}\r^+_{\lm1},
  \wb{\sbf}_\ell,\wb{\bbf}_\ell,\gamma_{\ell}^-,\gamma_{\lm1}^+) \\
  \label{eq:glintrans1n}
   \zhat_{\ell-1}^-=\gbf^-_\ell(\r^-_{\ell},\r^+_{\lm1},\gamma_{\ell}^-,\gamma_{\lm1}^+)   &=
  \Vbf_{\lm1}\tran{\Gbf}^-_\ell(\Vbf_{\ell}\T\r^-_{\ell},\Vbf_{\ell-1}\r^+_{\lm1},
  \wb{\sbf}_\ell,\wb{\bbf}_\ell,   \gamma_{\ell}^-,\gamma_{\lm1}^+),
\end{align}
\end{subequations}
where \emph{transformed denoising functions} ${\Gbf}^{\pm}_\ell(\cdot)$ are componentwise extensions of $G_\ell^{\pm}(\cdot)$, defined as
\begin{align}\label{eq:Glintrans_componentwise}
    \MoveEqLeft \begin{bmatrix}
        {G}_\ell^+(u_\ell,u_{\lm1},s_\ell,\wb{b}_\ell,\gamma_{\ell}^-,\gamma_{\lm1}^+)\\
        {G}_\ell^-(u_\ell,u_{\lm1},s_\ell,\wb{b}_\ell,\gamma_{\ell}^-,\gamma_{\lm1}^+)
    \end{bmatrix}
     :=  \begin{bmatrix}
        -\nu_\ell s_\ell & \gamma_\ell^- +\nu_\ell\\
        \gamma_{\lm1}^+ + \nu_\ell s_\ell^2 & -\nu_\ell s_\ell
        \end{bmatrix}^{-1}
        \begin{bmatrix}
        \gamma_{\ell}^- u_{\ell} + \nu_\ell \wb{b}_\ell\\
        \gamma_{\lm1}^+ u_{\lm1} - \nu_\ell s_\ell \wb{b}_\ell
        \end{bmatrix}.
\end{align}
A detailed derivation of equations \eqref{eq:glintrans} and \eqref{eq:Glintrans_componentwise} is given in  \cite[Appendix B]{fletcher2017inference}. Note that the argument $\wb\s_\ell$ in \eqref{eq:glintrans1p} is $N_\ell$ dimensional, whereas in \eqref{eq:glintrans1n} it is $N_{\ell-1}$ dimensional, \ie, appropriate zero-padding is applied. Keeping this subtlety in mind, we use $\wb\s_\ell$ to keep the notation simple.

From Algorithm~\ref{algo:ml-vamp}, we see that each pass of the MAP-ML-VAMP or MMSE-ML-VAMP algorithm requires solving 
(a) scalar MAP or MMSE estimation problems for the non-linear, separable layers; 
and (b) least-squares problems for the linear layers.
In particular, no high-dimensional integrals or high-dimensional optimizations are involved.
%%%%%% END sections/mlvamp_algorithm %%%%%%%%%%%%%

%%%%%% START sections/fixed_points %%%%%%%%%%%%%
\section{Fixed Points of ML-VAMP} \label{sec:fixed}
Our first goal is to characterize the fixed points of Algorithm \ref{algo:ml-vamp}. 
To this end,
let $\r^+_\ell,\r^-_\ell,\zhat_\ell$ with parameters $\alpha^+_\ell,\alpha^-_\ell,\gamma_\ell^+,\gamma_\ell^-,\eta_\ell$ 
be a fixed point of the ML-VAMP algorithm, where we have dropped the iteration subscript $k$. 
At a fixed point, we do not need to distinguish between $\zhat^+_\ell$ and $\zhat^-_\ell$, nor between $\eta_\ell^+$ and $\eta_\ell^-$, since the updates in \eqref{eq:gamupdate} imply that
\begin{align}\label{eq:fixed_point}
   \eta_\ell^+=\eta_\ell^- = \gamma_\ell^+ + \gamma_\ell^- = : \eta_\ell,
   \quad 
   \alpha^+_\ell=\tfrac{\gamma^-_\ell}{\eta_\ell},   
   \quad 
   \alpha^-_\ell=\tfrac{\gamma^+_\ell}{\eta_\ell},
   \quad\text{and}\quad  
   \alpha_\ell^+ + \alpha_\ell^- = 1.
\end{align}
Applying these relationships to lines \ref{line:rp} and \ref{line:rn} of Algorithm~\ref{algo:ml-vamp} gives
\begin{align}\label{eq:z_equal}
\zhat_\ell^+=\zhat_\ell^- = \frac{\gamma_\ell^+\r_\ell^++\gamma_\ell^-\r_\ell^-}{\gamma_\ell^++\gamma_\ell^-} = : \zhat_\ell .
 \end{align}

\subsection{Fixed points of MAP-ML-VAMP and connections to ADMM}
Our first results relates the MAP-ML-VAMP updates to an ADMM-type minimization of the MAP objective \eqref{eq:map_estimator}.  
For this we use \emph{variable splitting}, where we replace each variable $\zbf_\ell$ with two copies, $\zbf_\ell^+$ and $\zbf^-_\ell$.
Then, we define the objective function
\begin{align}
   \MoveEqLeft F(\zbf^+,\zbf^-) := -\ln p(\zbf^+_0)  - \sum_{\ell=1}^{\Lm1} \ln p(\zbf^+_\ell|\zbf^-_{\lm1})
   -\ln p(\ybf|\zbf_{\Lm1}^{-}) \label{eq:Fsplit}
\end{align}
over the variable groups $\zbf^+ := \{ \zbf^+_\ell \}_{\ell=1}^{L-1}$ and $\zbf^- := \{ \zbf^-_\ell \}_{\ell=1}^{L-1}$.
The optimization \eqref{eq:map_estimator} is then equivalent to
\begin{align} \label{eq:Fmincon}
    \min_{\zbf^+,\zbf^-} \ F(\zbf^+,\zbf^-)
    \quad\text{subject to}\quad
    \zbf^+_\ell=\zbf^-_\ell,\ ~\forall\ \ell=0,1,\ldots, \Lm1.
\end{align}
Corresponding to this constrained optimization, we define the augmented Lagrangian
\begin{align}\label{eq:Lagdef}
    \mc L(\zbf^+,\zbf^-,\sbf) = F(\zbf^+,\zbf^-)
    + \sum_{\ell={0}}^{\Lm1} \eta_\ell\sbf\tran_\ell(\zbf^+_\ell-\zbf_\ell^-)+
    \sum_{\ell=0}^{L-1}\frac{\eta_\ell}{2}\|\zbf_{\ell}^+-\zbf_{\ell}^-\|^2,
\end{align}
where $\sbf:=\{\sbf_\ell\}$ is a set of dual parameters, $\gamma_\ell^{\pm}>0$ are weights,
and $\eta_\ell = \gamma^+_\ell+\gamma^-_\ell$.   Now, for $\ell=1,\ldots,L-2$, define
\begin{align}\label{eq:Laug}
\begin{split}
    \mc L_\ell(\zbf_{\lm1}^-,\zbf_\ell^+;\zbf_{\lm1}^+,\zbf_\ell^-,\sbf_{\lm1},\sbf_\ell)
    := &-\ln p(\zbf_\ell^+|\zbf_{\lm1}^-)
     +\eta_\ell\sbf\tran_\ell\zbf^+_\ell -\eta_{\lm1}\sbf\tran_{\lm1}\zbf_{\lm1}^- \\
    &+ \frac{\gamma^+_{\lm1}}{2}\|\zbf^-_{\lm1} - \zbf^+_{\lm1}\|^2
    + \frac{\gamma^-_\ell}{2}\|\zbf^+_{\ell} - \zbf^-_{\ell}\|^2,  
\end{split}
\end{align}
which represents the terms in the Lagrangian $\mc L(\cdot)$ in \eqref{eq:Lagdef}
that contain $\zbf_{\lm1}^-$ and $\zbf_\ell^+$.
Similarly, define $\mc L_0(\cdot)$ and $\mc L_{\Lm1}(\cdot)$ using $p(\zbf_0^+)$ and $p({\bf y}| \zbf^+_{L-1})$, respectively.
\iftoggle{conference}{}{One can then verify that
\beq\nonumber
\mc L(\zbf^+,\zbf^-,\sbf) = \sum_{\ell=0}^{L-1}\mc L_\ell(\zbf_{\lm1}^-,\zbf_\ell^+;\zbf_{\lm1}^+,\zbf_\ell^-,\sbf_{\lm1},\sbf_\ell).
\eeq
}

\begin{theorem}[MAP-ML-VAMP] \label{thm:mapfix}
Consider the iterates of Algorithm~\ref{algo:ml-vamp} with MAP estimation functions
\eqref{eq:MAP_update} for fixed $\gamma_\ell^{\pm}>0$.
Suppose lines \ref{line:alphap} and \ref{line:alphan} are replaced with fixed values $\alpha^{\pm}_{k\ell}=\alpha^{\pm}_{\ell}\in(0,1)$ from \eqref{eq:fixed_point}.
Let $\sbf^{-}_{k\ell} := \alpha^{+}_{k\ell}(\zbfhat_{\km1,\ell}^--\rbf^-_{k\ell})$ 
and $\sbf^{+}_{k\ell} := \alpha^{-}_{k\ell}(\rbf^+_{k\ell}-\zbfhat_{k\ell}^+)$.
Then, for $\ell=0,\ldots,\Lm1$, the forward pass iterations satisfy
\begin{subequations}
\begin{align}
    \underline{\hspace{.3cm}}\,,\zbfhat^+_{k\ell} &= \argmin_{(\zbf_{\lm1}^-,\zbf^+_{\ell})}\
        %\nonumber \\
        \mc L_\ell(\zbf^-_{\lm1},\zbf^+_\ell;\zbfhat^+_{k,\lm1},\zbfhat^-_{\km1,\ell},\sbf_{k,\lm1}^+,\sbf_{k\ell}^-)  \label{eq:admmHp}\\
    \sbf_{k\ell}^+ &= \sbf_{k\ell}^- + \alpha^+_\ell(\zbfhat^+_{k\ell}-\zbfhat^-_{\km1,\ell}) \label{eq:admmsp} ,
\end{align}
\end{subequations}
whereas the backward pass iterations satisfy
\begin{subequations}
\begin{align}
    \zbfhat^-_{k,\lm1},\,\underline{\hspace{.3cm}}\, &= \argmin_{(\zbf_{\lm1}^-,\zbf^+_{\ell})}\
        \mc L_\ell(\zbf^-_{\lm1},\zbf^+_\ell;\zbfhat^+_{k,\lm1},\zbfhat^-_{k\ell},
            \sbf_{k,\lm1}^+,\sbf_{\kp1,\ell}^-)  \label{eq:admmHn} \\
 \sbf_{\kp1,\lm1}^- &= \sbf_{k,\lm1}^+ + \alpha^-_{\lm1}(\zbfhat^+_{k,\lm1}
        -\zbfhat^-_{k,\lm1}).
    \label{eq:admmsn}
\end{align}
\end{subequations}
Further, any fixed point of Algorithm 1 corresponds to a critical point of the Lagrangian \eqref{eq:Lagdef}.
\end{theorem}
\begin{proof} 
See Appendix \ref{app:fixed_points}
\end{proof}

Theorem~\ref{thm:mapfix} shows that the fixed-$\{\alpha_\ell^{\pm}\}$ version of ML-VAMP is an ADMM-type algorithm for solving the optimization problem \eqref{eq:Fmincon}.
In the case that $\alpha_\ell^+=\alpha^-_\ell$, this algorithm is known as the Peaceman-Rachford Splitting variant of ADMM and its convergence has been studied extensively; see \cite[eqn. (3)]{he2016application} and \cite{he2014strictly}, and the references therein.
Different from ADMM, the full ML-VAMP algorithm adaptively updates $\{\alpha_{k\ell}^{\pm}\}$ in a way that exploits the local curvature of the objective in \eqref{eq:MAP_update}. 
Note that, in \eqref{eq:admmHp} and \eqref{eq:admmHn}, we compute the joint minimizers over $(\zbf^+_{\lm1},\zbf^+_\ell)$, but only use one of them at a time.

% \begin{theorem} \label{thm:mapfix}  Consider the outputs of the
% of ML-VAMP (Algorithm~\ref{algo:ml-vamp}) with the MAP estimation functions
% \eqref{eq:gmap} with \emph{fixed values} $\gamma^{\pm}_\ell>0$.
% Then, there exists a sequence of dual parameters $\sbf_{k\ell}^{\pm}$ such that
% the updates in the forward pass are given by,
% \begin{subequations}
% \begin{align}
%     \MoveEqLeft \zbfhat^+_{k\ell} = \argmin_{\zbf_\ell^+} \min_{\zbf^-_{\lm1}}
%         \nonumber \\
%         & H_\ell(\zbf^-_{\lm1},\zbfhat^+_{k,\lm1},\zbfhat^-_{\km1,\ell},\zbf^+_\ell,
%             \sbf_{k,\lm1}^+,\sbf_{k\ell}^-)  \label{eq:admmHp} \\
%     \MoveEqLeft \sbf_{k\ell}^+ = \sbf_{k\ell}^- + \gamma^+_\ell(\zbfhat^+_{k\ell}-\zbfhat^-_{\km1,\ell}). \label{eq:admmsp}
% \end{align}
% \end{subequations}
% for $\ell=1,\ldots,L-2$
% Similarly, in the reverse pass,
% \begin{subequations}
% \begin{align}
%     \MoveEqLeft \zbfhat^-_{k,\lm1} = \argmin_{\zbf_{\lm1}^-} \min_{\zbf^+_{\ell}}
%         \nonumber \\
%         &H_\ell(\zbf^-_{\lm1},\zbfhat^+_{k,\lm1},\zbfhat^-_{k\ell},\zbf^+_\ell,
%             \sbf_{k,\lm1}^+,\sbf_{\kp1,\ell}^-)  \label{eq:admmHn} \\
%     \MoveEqLeft \sbf_{\kp1,\lm1}^- = \sbf_{k,\lm1}^+ + \gamma^-_{\lm1}(\zbfhat^+_{k,\lm1}
%         -\zbfhat^-_{k,\lm1}).
%     \label{eq:admmsn}
% \end{align}
% \end{subequations}
% Similar equations hold for $\ell=0$ and $\Lm1$.  In addition, any fixed point
% of the updates corresponds to a critical point of the Lagrangian \eqref{eq:Lagdef}.
% \end{theorem}

\subsection{Fixed Points of MMSE-ML-VAMP and Connections to Free-Energy Minimization}

Recall that $\z:=\{\z_\ell\}_{\ell=0}^{L-1}$ and let $\mc B$ denote the set of density functions $b(\z)$ factorizable as 
$f_0(\z_0)f_L(\z_{L-1})\prod_{\ell=1}^{L-1} f_\ell(\z_\ell,\z_{\ell-1})$.
Notice that the true posterior $p(\z|\y)$ from \eqref{eq:posterior_factorization} belongs to this set. 
Essentially, this $\mc B$ captures the chain structure of the factor graph visible in the top panel of Fig.~\ref{fig:nn_ml_vamp}.
For chain-structured (and, more generally, tree-structured) graphs, one can express any $b\in\mc B$ as \cite{yedidia2005constructing} (see also \cite[Sec. III C]{pereyra2015survey} for a succinct description) 
\begin{align}\label{eq:tree_marginal_property}
b(\z) = \frac{\prod_{\ell=1}^{L-1} f_{\ell}(\z_\ell,\z_{\ell-1})}{\prod_{\ell=1}^{L-2}q_\ell(\z_\ell)},
\end{align}
where $\{f_\ell(\z_\ell,\z_{\lm1})\}$ and $\{q_\ell(\z_\ell)\}$ are marginal density functions of $b(\z)$.
As marginal densities, they must satisfy the consistent-marginal equations 
\begin{align}\label{eq:consistent_marginals}
    b(\z_\ell)=\int f_\ell(\z_\ell,\z_{\lm1})\dif\z_{\lm1} 
    = q_\ell(\z_\ell)  
    = \int f_{\ell+1}(\z_{\ell+1},\z_{\ell})\dif\z_{\lp1},
    \quad \forall\ \ell=1,\ldots, \Lm1.
\end{align}
Because $p(\z|\ybf)\in\mc B$, we can express it using variational optimization as
\begin{align}\label{eq:variational}
p(\z|\y) = \arg\min_{b\in\mc B} D_{\mathsf{KL}}(b(\z)\|p(\z|\y)),
\end{align}
where $D_{\mathsf{KL}}(b(\z)\| p(\z|\ybf)) :=\int b(\z)\ln\frac{b(\z)}{p(\z|\y)}\dif\z$ is the KL divergence.
Plugging $b(\z)$ from \eqref{eq:tree_marginal_property} into \eqref{eq:variational}, we obtain 
\begin{align}\label{eq:BFE}
        p(\z|\y)=\arg\min_{b\in \mc B} \left\{ \sum_{\ell=1}^L D_{\mathsf{KL}}(f_\ell(\z_\ell,\z_{\ell-1})\|p(\z_\ell|\z_{\ell-1}))  + \sum_{\ell=0}^{L-1} h(q_\ell(\z_\ell)) \right\}
    \quad \text{s.t.\ \eqref{eq:consistent_marginals}},
\end{align}
%also called the Gibbs free energy minimization,
where $h(q_\ell(\z_\ell)):=-\int q_\ell(\z_\ell) \ln q_\ell(\z_\ell) \dif\z_\ell$ is the differential entropy of $q_\ell$. 
%$D_{\mathsf{KL}}(f_\ell(\z_\ell,\z_{\lm1})\| p(\z_\ell|\z_{\lm1})))  :=\int f_\ell(\z_\ell,\z_{\lm1})\ln\tfrac{f_\ell(\z_\ell,\z_{\lm1})}{p(\z_\ell|\z_{\lm1})}\dif\z_\ell \dif\z_{\lm1}$ 
The cost function in \eqref{eq:BFE} is often called the Bethe free energy \cite{yedidia2005constructing}.
% \begin{align}\label{eq:Bethe_def}
% F_{\mathsf{Bethe}}\left(\{b_{\ell}\},\{q_\ell\}\right) := \sum_{\ell=0}^L D_{\mathsf{KL}}(b_\ell(\z_\ell,\z_{\ell-1})\|p(\z_\ell|\z_{\ell-1}))  + \sum_{\ell=0}^{L-1} H(q_\ell(\z_{\ell}))
% \end{align}
In summary, because $\mc B$ is tree-structured, Bethe-free-energy minimization yields the exact posterior distribution \cite{yedidia2005constructing}. 

The constrained minimization \eqref{eq:BFE} is computationally intractable, because both the optimization variables $\{f_\ell,q_\ell\}$ and the pointwise linear constraints \eqref{eq:consistent_marginals} are infinite dimensional.
%(often called the codimension in variational optimization).
Rather than solving for the exact posterior, we might instead settle for an approximation obtained by relaxing the marginal constraints \eqref{eq:consistent_marginals} to the following moment-matching conditions, for all $\ell=0,1,\ldots \Lm1$:
\begin{align}\label{eq:moment_matching}
\begin{split}
    \Exp[\z_\ell|f_\ell]=\Exp[\z_\ell|q_\ell],&\qquad
     \Exp[\z_{\ell}|f_{\ell+1}]=\Exp[\z_{\ell}|q_{\ell}],\\
    \Exp\Big[\norm{\z_\ell}^2\Big|f_\ell\Big]=\Exp\Big[\norm{\z_\ell}^2\Big|q_{\ell}\Big],&\qquad
     \Exp\Big[\norm{\z_{\ell}}^2\Big|f_{\ell+1}\Big]=\Exp\Big[\norm{\z_{\ell}}^2\Big|q_{\ell}\Big].
\end{split}
\end{align}
This approach is known as expectation-consistent (EC) approximate inference \cite{opper2005expectation}.
Because the constraints on $f_\ell$ and $q_\ell$ in \eqref{eq:moment_matching} are finite dimensional, standard Lagrangian-dual methods can be used to compute the optimal solution. 
Thus, the EC relaxation of the Bethe free energy minimization problem \eqref{eq:BFE}, i.e., 
\begin{align}\label{eq:BFE_EC}
\min_{f_\ell}\max_{q_\ell} \left\{ \sum_{\ell=1}^{L-1} D_{\mathsf{KL}}(f_\ell(\z_\ell,\z_{\ell-1})\|p(\z_\ell|\z_{\ell-1}))  + \sum_{\ell=0}^{L-1} h(q_\ell(\z_{\ell})) \right\}
    \quad \text{s.t.\ \eqref{eq:moment_matching}} ,
\end{align}
yields a tractable approximation to $p(\z|\y)$. 
% KKT conditions for the optimization problem in \eqref{eq:BFE_EC} then implies that the the stationary distributions $\{f_\ell^*,q_\ell^*\}$ are of the form

We now establish an equivalence between the fixed points of the MMSE-ML-VAMP algorithm and the first-order stationary points of \eqref{eq:BFE_EC}. 
The statement of the theorem uses the belief functions $b_\ell$ defined in \eqref{eq:def_belief}.
\begin{theorem}[MMSE-ML-VAMP]\label{thm:fixed_point_mmse}
Consider a fixed point 
$\left(\{\r^\pm_\ell\},\{\zhat_\ell\},\{\gamma_\ell^\pm\}\right)$ 
of Algorithm \ref{algo:ml-vamp} with MMSE estimation functions \eqref{eq:MMSE_update}. Then $\{\gamma_\ell^+\r^+_\ell,\gamma_\ell^-\r^-_\ell,\tfrac{\gamma_\ell^+}2,\tfrac{\gamma_\ell^-}2\}$, are Lagrange multipliers for \eqref{eq:moment_matching} such that KKT conditions are satisfied for the problem \eqref{eq:BFE_EC} at primal solutions $\{f_\ell^*,q_\ell^*\}$. Furthermore, the marginal densities take the form $f_\ell^*(\cdot)\propto b_\ell(\cdot|\r_\ell^-,\r_{\ell-1}^+,\gamma_{\ell}^-,\gamma_{\ell}^-,\gamma_{\ell-1}^+)$ and $q_\ell^* = \Norm(\zhat_\ell,\bm{I}/\eta_{\ell})$,
with $\zhat_\ell$ and $\eta_\ell$ given in \eqref{eq:fixed_point}-\eqref{eq:z_equal}.
% \begin{subequations}
% \begin{align} \label{eq:BFE_EC_solution}
% f_\ell^*(\z_\ell, \z_{\ell-1}) 
% &\propto  p(\z_\ell|\z_{\ell-1})
% \exp\left(-\frac{\gamma_\ell^-}{2}\norm{\z_\ell-\r_\ell^-}_2^2\right)
% \exp\left(-\frac{\gamma_{\ell-1}^+}{2}\norm{\z_{\ell-1}-\r_{\ell-1}^+}_2^2\right)\\
% q_\ell^*(\z_\ell) 
% &\propto \exp\left(-\frac{\eta_\ell}{2}\norm{\z_\ell-\hat{z}_\ell}_2^2\right),
% \label{eq:approx_marginal_posterior}
% \end{align}
% and $\{\r_\ell^-, \r_{\ell-1}^+, \gamma_\ell^-, \gamma_{\ell-1}^+, \eta_\ell, \wh{\z}_\ell\}$ that satisfy the constraints in \eqref{eq:moment_matching}.
% \end{subequations}
\end{theorem}
\begin{proof}
See Appendix~\ref{app:fixed_points}.
\end{proof}

The above result shows that MMSE-ML-VAMP is essentially an algorithm to iteratively solve for the parameters 
$\left(\{\r^\pm_\ell\},\{\zhat_\ell\},\{\gamma_\ell^\pm\}\right)$ 
that characterize the EC fixed points. 
Importantly, $q_\ell^*(\z_\ell)$ and $f^*(\z_\ell,\z_{\ell-1})$ serve as an approximate marginal posteriors for $\z_\ell$ and $(\z_{\ell},\z_{\ell-1})$. 
This enables us to not only compute the MMSE estimate (i.e., posterior mean), but also other estimates like the MMAE estimate (i.e., the posterior median), or quantiles of the marginal posteriors. 
Remarkably, in certain cases, these approximate marginal-posterior statistics become exact. 
This is one of the main contributions of the next section.

%\begin{proof}
%See Appendix~\ref{app:fixed_points}.
%\end{proof}

%%%%%% END sections/fixed_points %%%%%%%%%%%%%

%%%%%% START sections/LSL %%%%%%%%%%%%%
\section{Analysis in the Large-System Limit} \label{sec:seevo}

\subsection{LSL model}
In the previous section, we established that, for any set of deterministic matrices $\{\Wbf_\ell\}$,
MAP-ML-VAMP solves the MAP problem and MMSE-ML-VAMP solves the EC variational inference problem as the iterations $k\rightarrow\infty$.
In this section, we extend the analysis of \cite{BayatiM:11,rangan2019vamp} to 
the rigorously study the behavior of ML-VAMP at any iteration $k$ for classes
of random matrices $\{\Wbf_\ell\}$ in a certain large-system limit (LSL).
The model is described in the following set of assumptions.

\medskip \noindent
\paragraph*{System model}  We consider a sequence of systems indexed by $N$.
For each $N$, 
let $\z_\ell = \z_\ell^0(N)\in \Real^{N_\ell(N)}$ be ``true'' vectors generated by
neural network \eqref{eq:nntrue} for layers $\ell=0,\ldots, L$, such that layer widths satisfy $\lim_{N\rightarrow \infty}N_\ell(N)/N = \beta_\ell\in(0,\infty)$. 
Also, let the weight matrices $\Wbf_\ell$ in \eqref{eq:nnlintrue} each have an SVD given by \eqref{eq:SVD}, where $\{\V_\ell\}$ are 
drawn uniformly from the set of orthogonal matrices in $\Real^{N_\ell\times N_\ell}$ and independent across $\ell$. 
The distribution on the singular values $\sbf_\ell$ will be described below.

Similar to the VAMP analysis \cite{rangan2019vamp}, the assumption here is that weight matrices
$\Wbf_\ell$ are rotationally invariant, meaning that $\Vbf\Wbf_\ell$ and $\Wbf_\ell\Vbf$
are distributed identically to $\Wbf_\ell$.  Gaussian i.i.d.\ $\Wbf_\ell$ as considered in the
original ML-AMP work of \cite{manoel2017multi} satisfy this rotationally invariant assumption, but the 
rotationally invariant model is more general.  In particular, as described in \cite{rangan2019vamp},
the model can have arbitrary coniditoning which is known to be a major failure mechanism of AMP 
methods.

\medskip \noindent
\paragraph*{ML-VAMP algorithm}  We assume that we generate estimates $\zbfhat^{\pm}_{k\ell}$
from the ML-VAMP algorithm, Algorithm~\ref{algo:ml-vamp}.  Our analysis
will apply to general estimation functions, $\gbf_\ell(\cdot)$, not necessarily 
the MAP or MMSE estimators.  However, we require two technical conditions:
For the non-linear estimators, 
$\gbf_\ell^\pm$ for $\ell=2,4,\ldots L-2$, and $\gbf^+_0$, $\gbf_L^-$ act componentwise. Further, these estimators and their derivatives $\tfrac{\gbf_\ell^+}{\partial z_{\ell}^-}$,$\tfrac{\gbf_\ell^-}{\partial z_{\ell-1}^+}$,$\tfrac{\gbf_0^+}{\partial z_{0}^-}$,$\tfrac{\gbf_L^-}{\partial z_{L-1}^+}$ are uniformly Lipschitz continuous.
The technical 
definition of uniformly Lipschitz continuous is given in Appendix~\ref{sec:empirical}.
For the linear layers, $\ell=1,3,\ldots L-1$, 
we assume we apply estimators $\gbf_{\ell}^\pm$ of the form \eqref{eq:glintrans} 
where $\Gbf_\ell^\pm$ act componentwise. Further, $\Gbf_\ell^\pm$ 
along with its derivatives are uniformly Lipschitz continuous.
We also assume that 
the activation functions $\bm\phi_\ell$ in equation \eqref{eq:nnnonlintrue} 
are componentwise separable and Lipschitz continuous.
To simplify the analysis, we will also assume the estimation function
parameters $\theta_{k\ell}^\pm$ converge to fixed limits,
\beq \label{eq:thetalim}
    \lim_{N\rightarrow\infty}\theta^{\pm}_{k\ell}(N) = \thetabar_{k\ell}^{\pm},
\eeq
for values $\thetabar_{k\ell}^{\pm}$.
Importantly, in this assumption, we assume that the limiting parameter
values  $\thetabar_{k\ell}^{\pm}$ are fixed and not data dependent.  
However, data dependent parameters can also be modeled \cite{rangan2019vamp}.

\medskip \noindent
\paragraph*{Distribution of the components}
We follow the framework of Bayati-Montanari and describe the statistics on the
unknown quantities via their empirical convergence -- see Appendix~\ref{sec:empirical}.
For $\ell=1,3,\ldots L-1,$ define $\wb{\bbf}_\ell := \Vbf_\ell\tran\bbf_\ell$ and
$\wb{\xibf}_\ell := \Vbf_\ell\tran \xibf_\ell$. 
We assume that the sequence of 
true vectors $\z_0^0$, singular values $\s_\ell$, bias vectors $\wb{\bbf}_\ell$, 
and noise realizations $\wb{\bm\xi}_\ell$ empirically converge as
    \begin{subequations}\label{eq:varinit}
    \begin{align}
    \label{eq:varinitnl}
\MoveEqLeft    \lim_{N \arr \infty} \left\{ z^0_{0,n} \right\} \PLeq Z^0_0, \quad
    \lim_{N \arr \infty} \left\{ \xi_{\ell,n} \right\} \PLeq \Xi_\ell, \qquad \ell=2,4,\ldots,L\\
    \label{eq:varinitlin}
\MoveEqLeft    \lim_{N \arr \infty} \left\{ ({s}_{\ell,n},\wb{b}_{\ell,n},\wb{\xi}_{\ell,n}) \right\}
        \PLeq ({S}_\ell, \wb{B}_\ell,\wb{\Xi}_\ell),\qquad \ell=1,3,\ldots,L-1,
        \end{align}
\end{subequations}
to random variables $Z_0^0, \Xi_\ell, S_\ell, \wb{B}_\ell, \wb{\Xi}_\ell$.
We will also 
assume that the singular values are bounded, i.e., $s_{\ell,n}<S_{\ell,{\rm max}}~\forall n$. 
Also, the initial vectors $\r_{0\ell}^-$ converge as,
\begin{align}\label{eq:main_result_initialization}
\begin{split}
    \lim_{N\arr \infty} \left\{ [{\r}^-_{0\ell}-\z_{\ell}^0]_n \right\}
        \PLeq {Q}^-_{0\ell},\qquad \ell=0,2,\ldots,L,\\
        \left\{ [\V_\ell\T({\r}^-_{0\ell}-\z_{\ell}^0)]_n \right\}
        \PLeq {Q}^-_{0\ell}, \qquad \ell=1,3,\ldots,L-1,
\end{split}
\end{align}
where $(Q_{0\ell}^-,Q_{1\ell}^-,\ldots Q_{L-1,\ell}^-)$ 
is jointly Gaussian independent of  $Z_0^0$, $\{\Xi_\ell\}$, $\{S_\ell,\wb B_\ell,\wb{\Xi}_\ell\}$.

\medskip \noindent
\paragraph*{State Evolution}  
Under the above assumptions, our main result is to show that the asymptotic distribution
of the quantities from ML-VAMP algorithm converge to certain distributions.
The distributions are described by a set of deterministic parameters
$\{\Kbf_{k\ell}^+,\tau_{k\ell}^-,\wb\alpha_{k\ell}^\pm,\wb\gamma_{k\ell}^\pm,\wb\eta_{k\ell}^\pm\}$. 
The evolve according to a scalar recursion called the state evolution (SE), given in Algorithm \ref{algo:mlvamp_se} in Appendix~\ref{app:mlvamp_se}.
We assume
$\wb\alpha_{k\ell}^\pm\in(0,1)$ for all iterations $k$ and $\ell=0,1,\ldots \Lm1$.

\subsection{SE Analysis in the LSL}
Under these assumptions, we can now state our main result.

\begin{theorem} \label{thm:main_result}  Consider the system
under the above assumptions.
For any componentwise pseudo-Lipschitz function $\bm\psi$ of order 2, iteration index $k$, 
and layer index $\ell=2,4,\ldots L-2$, 
\begin{align}
\MoveEqLeft    \lim_{N\rightarrow\infty}\Big<\bm\psi\left(\z_{\ell-1}^0,\zhat_{k,\ell-1}^-,\zhat_{k\ell}^+\right)\Big> \xrightarrow{a.s.}\nonumber\\
\MoveEqLeft \qquad\qquad\qquad \Exp\left[\psi\left(\mathsf A ,g_\ell^-(\mathsf C +\mathsf A ,\mathsf B +\mathsf A,\wb\gamma_{k\ell}^-,\wb\gamma_{k,\ell-1}^+ ),g_\ell^+(\mathsf C +\mathsf A ,\mathsf B +\mathsf A,\wb\gamma_{k\ell}^-,\wb\gamma_{k,\ell-1}^+ )\right)\right],
\label{eq:even_pl2}
\\
\MoveEqLeft\lim_{N\rightarrow\infty}\big<\bm\psi(\z,_0^0\zhat_{k0}^+)\big> \xrightarrow{a.s.}\Exp\left[\psi(g^+_0(\mathsf{C} '+Z_0^0),\wb\gamma_0^-)\right],\quad \\
\MoveEqLeft    \lim_{N\rightarrow\infty}\big<\bm\psi(\z_{L-1}^0,\zhat_{k,L-1}^-)\big> \xrightarrow{a.s.} \Exp\left[\psi(\mathsf{A} ',g_{L}^-(\mathsf{B} '+\mathsf{A} ',\wb\gamma_{L-1}^+))\right],  
\end{align}
where $(\mathsf A,\mathsf B)\sim \mc N(0,\bm\Kbf_{k\ell}^+)$ and $\mathsf C\sim \mc N(0,\tau_{k\ell}^-)$ are mutually independent and independent of $\Xi_\ell$; $(\mathsf A',\mathsf B')\sim \mc N(0,\bm\Kbf_{kL}^+)$ is independent of $\Xi_L$ and $\mathsf C'\sim \mc N(0,\tau_{k0}^-)$ is independent of $Z_0^0$.
Similarly for any layer index $\ell=1,3,\ldots,\Lm1$, we have
\begin{align}
   \MoveEqLeft \lim_{N\rightarrow\infty} \bkt{\bm\psi\left(\V_{\ell-1}\z_{\ell-1}^0,\V_{\ell-1}\zhat_{k,\ell-1}^-,\V_{\ell}\T\zhat_{k\ell}^+\right)} \xrightarrow{a.s.} \nonumber\\
    \MoveEqLeft \quad\Exp\left[\psi\left(\mathsf A ,G_\ell^-(\mathsf C +\mathsf A ,\mathsf B +\mathsf A, S_\ell,\wb B_\ell,\wb\gamma_{k\ell}^-,\wb\gamma_{k,\ell-1}^+),G_\ell^+(\mathsf C +\mathsf A ,\mathsf B +\mathsf A,S_\ell,\wb B_\ell,\wb\gamma_{k\ell}^-,\wb\gamma_{k,\ell-1}^+ )\right)\right],\label{eq:odd_pl2}
\end{align}
where $(\mathsf A,\mathsf B)\sim \mc N(0,\bm\Kbf_{k\ell}^+)$ and $\mathsf C\sim \mc N(0,\tau_{k\ell}^-)$ are mutually independent and independent of $(S_\ell,\wb B_\ell,\wb\Xi_\ell)$. Furthermore, if $\gammabar^\pm_{k\ell},\etabar^\pm_{k\ell},$ are defined analogous to \eqref{eq:gamupdate} using $\wb\alpha_{k\ell}^\pm$, then for all $\ell$,
\begin{align}\label{eq:alpha_convergence}
    \lim_{N\rightarrow\infty} (\alpha^\pm_{k,\ell},\gamma^\pm_{k,\ell},\eta^\pm_{k,\ell})\xrightarrow{a.s.}(\wb\alpha^\pm_{k,\ell},\wb\gamma^\pm_{k,\ell},\wb\eta^\pm_{k,\ell}).
\end{align}
\end{theorem}
\begin{proof}
See Appendix \ref{app:proof_of_main_result}.
\end{proof}

\medskip
The key value of Theorem~\ref{thm:main_result} is that we can \emph{exactly
characterize} the asymptotic joint distribution of the true vectors $\z_\ell^0$
and the ML-VAMP estimates $\zbfhat^{\pm}_{k\ell}$.
The asymptotic joint distribution, can be used to compute 
various key quantities.  For example,
suppose we wish to compute the mean squared error (MSE).
Let $\psi(z^0,\zhat)=(z^0-\zhat)^2$, whereby 
$\bkt{\bm\psi(\z^0_\ell,\zbfhat^-_\ell)}=\frac1N\norm{\z^0_\ell-\zbfhat^-_\ell}^2$.
Observe that $\psi$ is a pseudo-Lipschitz function of order 2, whereby we can apply Theorem \ref{thm:main_result}. 
Using \eqref{eq:even_pl2}, we get the asymptotic MSE on the $k^{\rm th}$-iteration estimates for $\ell=2,4,\ldots \Lm2$:
\begin{align*}
    \lim_{N_{\ell-1}\rightarrow\infty} \tfrac1{N_{\ell-1}}\norm{\zhat_{k,\ell-1}^- -\z_{\ell-1}^0}^2 
    &\xrightarrow{a.s.} \Exp\left[\left(g_\ell^-(\mathsf C+\mathsf A,\mathsf B+\mathsf A,\wb\gamma_{k\ell}^-,\wb\gamma_{k,\ell-1}^+)-\mathsf A\right)^2\right],\\
    \lim_{N_{\ell}\rightarrow\infty} \tfrac1{N_{\ell}}\norm{\zhat_{k\ell}^+ -\z_{\ell}^0}^2 
    &\xrightarrow{a.s.} \Exp\left[\left(g_\ell^+(\mathsf C+\mathsf A,\mathsf B+\mathsf A,\wb\gamma_{k\ell}^-,\wb\gamma_{k,\ell-1}^+)-\phi_\ell(\mathsf A,\Xi_\ell)\right)^2\right],
\end{align*}
where we used the fact that $\phi_\ell$ is pseudo-Lipschitz of order 2, and $\z_\ell^0 = \phi_\ell(\z^0_{\ell-1},\xibf_\ell)$ from \eqref{eq:nnnonlintrue}. 
Similarly, using \eqref{eq:odd_pl2}, we get the $k$th-iteration MSE for $\ell=1,3,\ldots \Lm1$:
\begin{align*}
    \MoveEqLeft\lim_{N_{\ell-1}\rightarrow\infty} \tfrac1{N_{\ell-1}}\norm{\zhat_{k,\ell-1}^- -\z_{\ell-1}^0}^2 
    =\lim_{N_{\ell-1}\rightarrow\infty} \tfrac1{N_{\ell-1}}\norm{\V_{\ell-1}(\zhat_{k,\ell-1}^- -\z_{\ell-1}^0)}^2 \\
     \MoveEqLeft\qquad\qquad\qquad\qquad\xrightarrow{a.s.} \Exp\left[\left(G_\ell^-(\mathsf C+\mathsf A,\mathsf B+\mathsf A,S_\ell,\wb B_\ell,\gamma_{k,l}^+,\gamma_{k,\ell-1}^-)-\mathsf A\right)^2\right].\\
    \MoveEqLeft\lim_{N_{\ell}\rightarrow\infty} \tfrac1{N_{\ell}}\norm{\zhat_{k\ell}^+ -\z_{\ell}^0}^2 
    =\lim_{N_\ell\rightarrow\infty} \tfrac1{N_\ell}\norm{\V_\ell\T(\zhat_{k\ell}^+ -\z_{\ell}^0)}^2  \\
    \MoveEqLeft\qquad\qquad\qquad\qquad\xrightarrow{a.s.} \Exp\left[\left(G_\ell^+(\mathsf C+\mathsf A,\mathsf B+\mathsf A,\wb\gamma_{kl}^+,\wb\gamma_{k,\ell-1}^-)-S_\ell\mathsf A-\wb B_\ell\right)^2\right] ,
\end{align*}
where we used the rotational invariance of the $\ell_2$ norm, and the fact that equation \eqref{eq:nnlintrue} is equivalent to $\V_\ell\T\z_\ell^0=\diag(\s_\ell)\V_{\ell-1}\z_{\ell-1}^0+\wb\bbf_\ell$ using the SVD \eqref{eq:SVD} of the weight matrices $\W_\ell$.

At the heart of the proof lies a key insight: Due to the randomness of the unitary matrices $\V_\ell$,
the quantities $(\z_\ell^0,\r_{k\ell}^--\z_\ell^0,\r_{k,\ell-1}^+-\z_{\ell-1}^0)$ are asymptotically jointly
Gaussian for even $\ell$, with the asymptotic covariance matrix of $\{(z^0_{\ell-1,n},r^+_{k,\ell-1,n}-z^0_{\ell-1,n},r_{k\ell,n}^--z^0_{\ell,n})\}$ given by 
$\left[\begin{smallmatrix}\Kbf_{k\ell}^+ & \bm{0}\\ \bm{0} & \tau_{k\ell}^- \end{smallmatrix}\right]$,
where $\Kbf_{k\ell}\in \Real^{2\times 2}$ and $\tau_{k\ell}^-$ is a scalar. After establishing the asymptotic Gaussianity of $(\z_\ell^0,\r_{k\ell}^--\z_\ell^0,\r_{k,\ell-1}^+-\z_{\ell-1}^0)$, since $\zhat_{\ell}$ and $\zhat_{\ell-1}$ are componentwise functions of this triplet, we have the PL(2) convergence result in \eqref{eq:even_pl2}.
Similarly, for odd $\ell$, we can show that $\left(\V_{\ell-1}\z_{\ell-1}^0,\V_{\ell-1}\r_{k,\ell-1}^+,\V_{\ell}\T\r_{k\ell}^-\right)$ is asymptotically Gaussian. For these $\ell$, $\V_{\ell-1}\zhat_{k,\ell-1}^-$ and $\V_{\ell}\T\zhat_{k\ell}^+$ are functions of the triplet, which gives the result in \eqref{eq:odd_pl2}.\label{rem:normality}

Due to the asymptotic normality mentioned above, the inputs $(\r_{\ell}^-,\r_{\ell-1}^+)$ to the estimators $\gbf_\ell^\pm$ are the true signals $(\z_{\ell-1}^0,\z_{\ell}^0)$ plus additive white Gaussian noise (AWGN). 
Hence, the estimators $\gbf_\ell^\pm$ act as denoisers, and ML-VAMP effectively reduces the inference problem \ref{eq:problem} into a sequence of linear transformations and denoising problems.
The denoising problems are solved by $\gbf_\ell^\pm$ for even $\ell$, and by $\Gbf_\ell^\pm$ for odd $\ell$. 
\subsection{MMSE Estimation and Connections to the Replica Predictions}
\label{sec:replica}

We next consider the special case of using MMSE estimators corresponding to the 
true distributions.  In this case, the SE equations simplify considerably
using the following \emph{MSE functions}:  
let $\zbfhat^-_{\ell-1}$, $\zbfhat^+_{\ell}$ be the MMSE
estimates of $\zbf^0_{\ell-1}$ and $\zbf^0_{\ell}$ from the variables
$\rbf^+_{\ell-1},\rbf^-_\ell$ under the joint density \eqref{eq:def_belief}.
Let ${\mathcal E}^{\pm}(\cdot)$ be the corresponding
mean squared errors,
\beq \label{eq:Ecal}
    {\mathcal E}^+_\ell(\gammabar^+_{\ell-1},\gammabar^-_{\ell}) := 
    \lim_{N \rightarrow \infty} \frac{1}{N} \Exp \left\|\zbf^0_{\ell}-\zbfhat^+_{\ell}
        \right\|^2, \quad
    {\mathcal E}^-_{\ell-1}(\gammabar^+_{\ell-1},\gammabar^-_{\ell}) := 
    \lim_{N \rightarrow \infty} \frac{1}{N} \Exp \left\|\zbf^0_{\ell-1}-\zbfhat^-_{\ell-1}
        \right\|^2.
\eeq

\begin{theorem}[MSE of MMSE-ML-VAMP]\label{thm:MMSE_SE}
Consider the system under the assumptions of Theorem \ref{thm:main_result}, 
with MMSE estimation functions $\gbf_\ell^\pm,\gbf_0^+,\gbf_{L}^-$ from \eqref{eq:MMSE_update} 
for the belief estimates in \eqref{eq:def_belief} with $\gamma^+_{k\ell}=\wb\gamma^{\pm}_{k\ell}$
from the state-evolution equations.  Then, the state evolution equations reduce to
\beq \label{eq:mlvamp_se}
    \gammabar^+_{k\ell} = \frac{1}{
        {\mathcal E}^+_\ell(\gammabar^-_{k\ell},\gammabar^+_{k,\ell-1})} - \gammabar^-_{k\ell},
    \quad
    \gammabar^-_{k+1,\ell} = \frac{1}{
        {\mathcal E}^-_\ell(\gammabar^-_{k+1,\ell+1},\gammabar^+_{k\ell})} - 
        \gammabar^+_{k\ell},
\eeq
where $1/\wb\eta^+_{k\ell}={\mathcal E}^+_\ell(\gammabar^-_{k\ell},\gammabar^+_{k,\ell-1})$ is the MSE of the estimate $\zhat_{k\ell}^+$.
\end{theorem}
\begin{proof}
See Appendix \ref{app:proof_of_main_result}.
\end{proof}

Since the estimation functions in Theorem~\ref{thm:MMSE_SE} 
are the MSE optimal functions for 
true densities, we will call this selection of estimation functions the \emph{MMSE matched
estimators}.  Under the assumption of MMSE matched estimators, the theorem
shows that the MSE error has a simple set of recursive expressions.

It is useful to compare the predicted MSE with the predicted optimal values.
The works \cite{reeves2017additivity,gabrie2018entropy} 
postulate the optimal MSE for inference in deep networks under the LSL model described above
using the replica method from statistical physics.
Interestingly, it is shown in \cite[Thm.2]{reeves2017additivity} that the predicted minimum MSE
satisfies equations that exactly agree with the fixed points of the updates
\eqref{eq:mlvamp_se}.
Thus, when the fixed points of \eqref{eq:mlvamp_se} are unique, ML-VAMP
with matched MMSE estimators provably achieves the Bayes optimal MSE predicted
by the replica method.
Although the replica method is not rigorous,
this MSE predictions have been indepedently proven for the Gaussian case in 
\cite{reeves2017additivity} and certain two layer networks in \cite{gabrie2018entropy}.
This situation is similar to several other works relating the MSE of AMP with replica 
predictions \cite{reeves2016replica,krzakala2012statistical}.
The consequence is that, if the replica method is correct,
ML-VAMP provides a computationally efficient method for inference
with testable conditions under which it achieves the Bayes optimal MSE.

%%%%%% END sections/LSL %%%%%%%%%%%%%

%%%%%% START sections/simulations %%%%%%%%%%%%%
\section{Numerical Simulations} \label{sec:sim}

We now numerically investigate the MAP-ML-VAMP and MMSE-ML-VAMP algorithms using two sets of experiments, 
where in each case the goal was to solve an estimation problem of the form in \eqref{eq:problem} 
using a neural network of the form in \eqref{eq:nntrue}.
%We solve the estimation problems \eqref{eq:map_estimator} and \eqref{eq:mmse_estimator} for two experiments using the ML-VAMP algorithm, where a problem posed as \eqref{eq:problem} and the true signals satisfy \eqref{eq:nntrue}. 
We used the Python~3.7 implementation of the ML-VAMP algorithm available on GitHub.\footnote{See \url{https://github.com/GAMPTeam/vampyre}.}

The first set of experiments uses random draws of a synthetic network to validate the claims made about the ML-VAMP state-evolution (SE) in Theorem \ref{thm:main_result}.
In addition, it compares MAP-ML-VAMP and MMSE-ML-VAMP to the MAP approach \eqref{eq:likelihood_min} using a standard gradient-based solver, ADAM \cite{kingma2014adam}.
The second set of experiments applies ML-VAMP to image inpainting, using images of handwritten digits from the widely used MNIST dataset. 
There, MAP-ML-VAMP and MSE-ML-VAMP were compared to Stochastic Gradient Langevin Dynamics (SGLD) \cite{welling2011bayesian}, an MCMC-based sampling method that approximates $\Exp[\z|\y]$, as well as to the optimization approach \eqref{eq:likelihood_min} using the ADAM solver.

%\subsection{Validation of the ML-VAMP State Evolution} \label{sec:se_sim}
\subsection{Performance on a Synthetic Network} \label{sec:se_sim}

We first considered a 7-layer neural network of the form in \eqref{eq:nntrue}.
The first six layers, with dimensions $N_0=20$, $N_1=N_2=100$, $N_3=N_4=500$, $N_5=N_6=784$, formed a 
(deterministic) deep generative prior driven by i.i.d.\ Gaussian $\z_0^0$.
The matrices $\Wbf_1,\Wbf_3,\Wbf_5$ and biases $\bbf_1,\bbf_3,\bbf_5$ were drawn i.i.d.\ Gaussian,
and the activation functions $\phi_2,\phi_4,\phi_6$ were ReLU.
The mean of the bias vectors $\bbf_\ell$ was chosen so that a fixed fraction, $\rho$, of the linear outputs were positive, 
so that only the fraction $\rho$ of the ReLU outputs were non-zero.
Because this generative network is random rather than trained, we refer to it as ``synthetic.''
The final layer, which takes the form $\ybf=\Abf\z_6^0+\xibf_6$, generates noisy, compressed measurements of $\z_6^0$.
Similar to \citep{RanSchFle:14-ISIT}, the matrix $\Abf\in\R^{M\times N_6}$ was constructed from the SVD $\A=\U\diag(\s)\V\tran$, where
the singular-vector matrices $\U$ and $\V$ were drawn uniformly from the set of orthogonal matrices,
and the singular values were geometrically spaced (i.e., $s_i/s_{i-1}=\kappa~\forall i$) to achieve a condition number of $s_1/s_{M}=10$.
It is known that such matrices cause standard AMP algorithms to fail \citep{RanSchFle:14-ISIT}, but not VAMP algorithms \cite{rangan2019vamp}.
The number of compressed measurements, $M$, was varied from 10 to 300,
and the noise vector $\xibf$ was drawn i.i.d.\ Gaussian with a variance set to achieve a signal-to-noise ratio of 
$10\log_{10}( \Exp\|\A\z_6^0\|^2/\Exp\|\xibf\|^2) =$ 30 dB. 

%with an $N_0=20$-dimensional i.i.d.\ Gaussian input $\zbf_0$,
%and has three hidden layers with 100, 500 and 784 units.  
%For the weight matrices and bias vectors in all but the final layer,
%$\Wbf_\ell$ and $\bbf_\ell$ are generated from the entrywise i.i.d. Gaussian ensemble. 
%ReLUs activation are applied to the hidden layers. 
%For $\ell=1,3,5$, the mean of the bias vectors $\bbf_\ell$ was chosen so that a fixed fraction, $\rho$, of the linear outputs $\z\ell$ would be positive.  Hence, in the ReLU outputs $\z_2^0,\z_4^0,\z_6^0$, only a fraction $\rho$ of the entries was non-zero.

%The final layer generated, we constructed the matrix similar to \citep{RanSchFle:14-ISIT} where $\A=\U\Diag(\s)\V\tran$,
%with $\U$ and $\V$ being random orthogonal matrices and $\sbf$ was logarithmically spaced 
%to obtain a desired condition number of $\kappa=10$.  
%It is known from \citep{RanSchFle:14-ISIT} that matrices with high condition numbers are precisely the matrices in which AMP algorithms fail.
%For the linear  measurements, $\y = \A\z_{6}+\w$, the noise level is set at $10\log_{10}( \Exp\|\A\zbf_6\|^2/\|\w\|^2) =$ 30 dB. 
%The number of measurements $M$, \ie, the number of rows of $\A$ is varied from 10 to 300.

To quantify the performance of ML-VAMP, we repeated the following 1000 times.
First, we drew a random neural network as described above.
Then we ran the ML-VAMP algorithm for 100 iterations, 
recording the normalized MSE (in dB) of the iteration-$k$ estimate of the network input, $\zhat_{k0}^\pm$:
\[
    \mathrm{NMSE}(\zhat_{k0}^\pm) := 10\log_{10}\left[ \frac{\|\z_0^0-\zhat_{k0}^\pm\|^2}{\|\z_0^0\|^2} \right] .
\]
Since ML-VAMP computes two estimates of $\z_0^0$ at each iteration, we consider each estimate as corresponding to a ``half iteration.''

%In this set of experiments, we ran the ML-VAMP algorithm on 1000 independent trials of deep neural networks with random weight matrices and compare the average performance over these trials with the predictions made by the State Evolution to validate the claim in Theorem \ref{thm:main_result} about the behaviour in the large system limit. We consider the normalized MSE (in dB) as the performance metric which is defined as
%Since each iteration of ML-VAMP involves a forward and reverse pass, we say that each iteration consists of two ``half-iterations", using the same terminology as turbo codes.

\begin{figure}
\centering
\includegraphics[width=0.45\columnwidth]{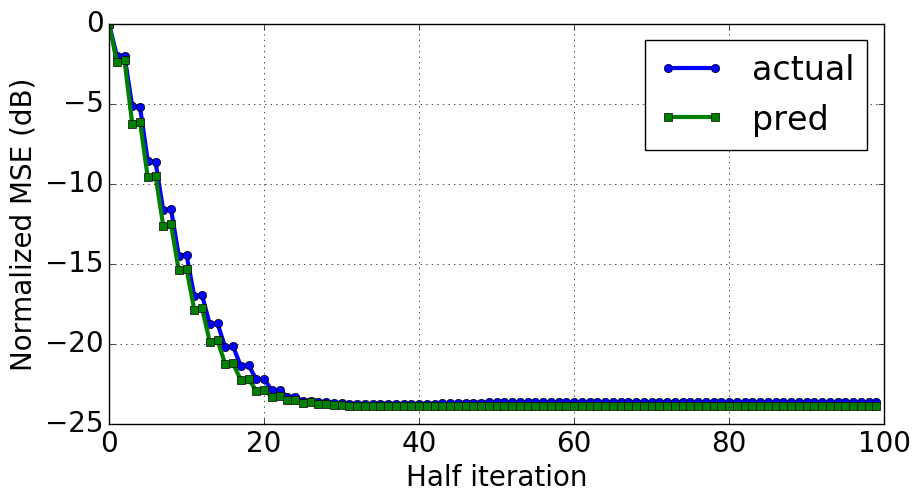}
\hfill
\includegraphics[width=0.45\columnwidth]{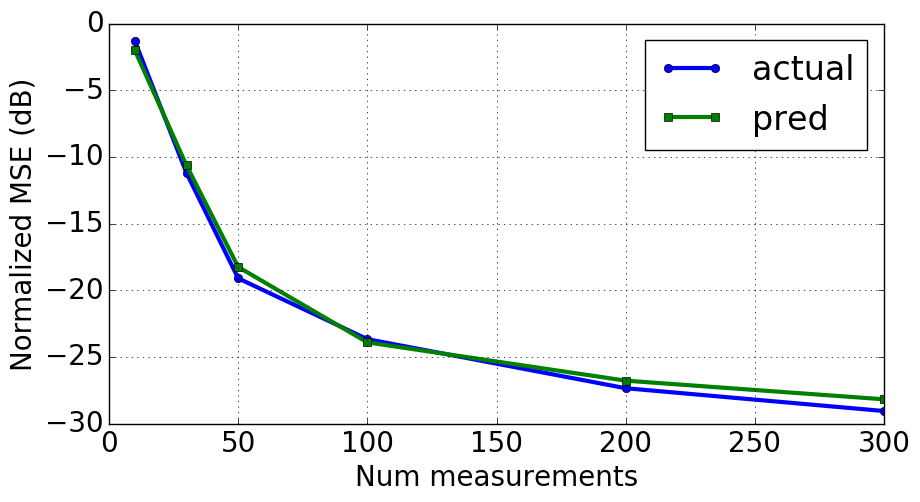}
\caption{NMSE of MMSE-ML-VAMP and its SE prediction when estimating the input to a randomly generated 7-layer neural network (see text of Section~\ref{sec:se_sim}).
Left panel:  Average NMSE versus half-iteration with $M=$ 100 measurements.
Right panel:  Average NMSE verus measurements $M$ at after $50$ iterations.}
\label{fig:randmlp_sim_mmse}
\end{figure}

\begin{figure}
\centering
\includegraphics[width=0.45\columnwidth]{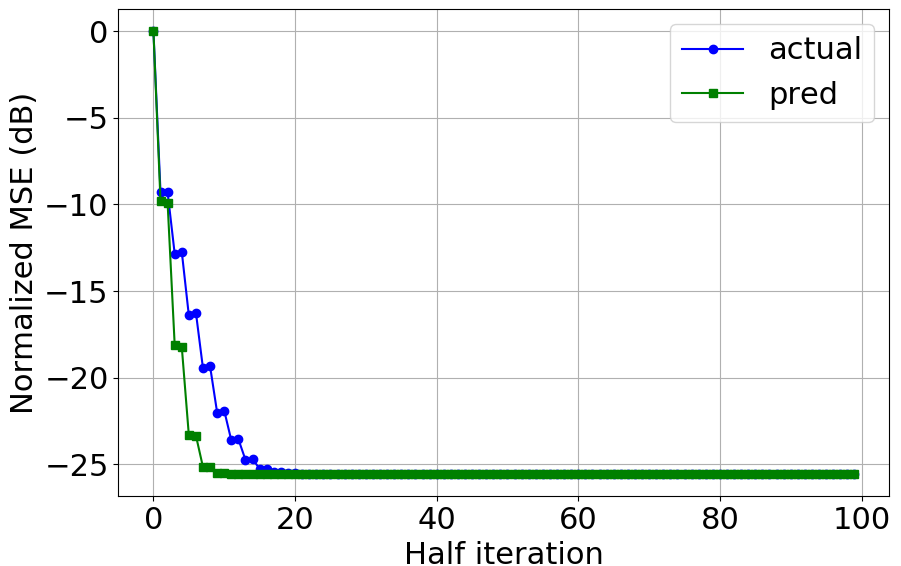}
\hfill
\includegraphics[width=0.45\columnwidth]{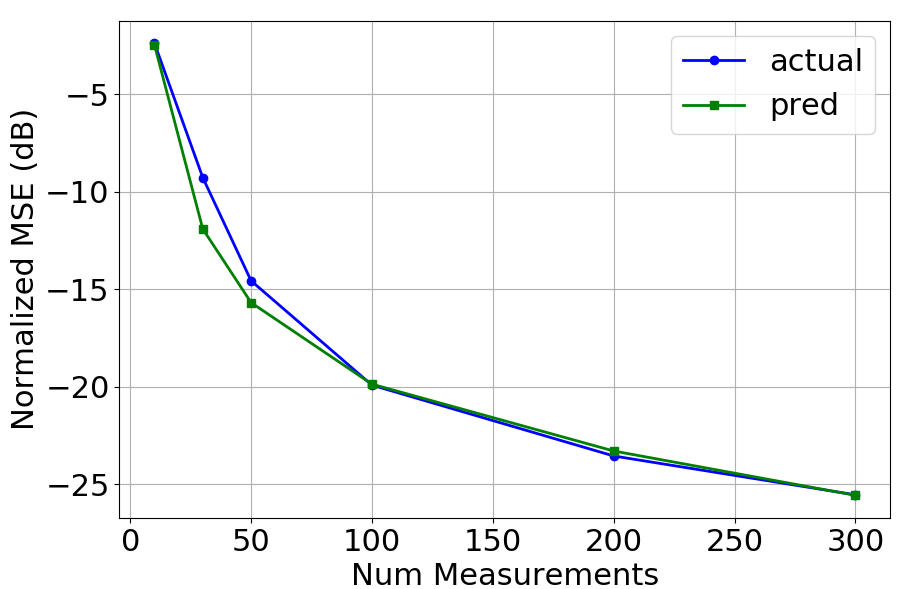}
\caption{Simulation with randomly generated neural network with MAP estimators from equation \eqref{eq:MAP_update}.
Left panel:  Normalized mean squared error (NMSE)
for ML-VAMP and the predicted MSE
as a function of the iteration with $M=$ 100 measurements.
Right panel:  Final NMSE (50 iterations) for ML-VAMP and the predicted MSE as a
function of the number of measurements, $M$. $\rho=0.9$}
\label{fig:randmlp_sim_map}
\end{figure}

\paragraph{Validation of SE Prediction} 
For MMSE-ML-VAMP,
the left panel of Fig.~\ref{fig:randmlp_sim_mmse} shows the NMSE versus half-iteration for $M=$ 100 compressed measurements.
The value shown is the average over 1000 random realizations. 
Also shown is the MSE predicted by the ML-VAMP state evolution.
Comparing the two traces, we see that the SE predicts the actual behavior of MMSE-ML-VAMP remarkably well, within approximately 1~dB\@.  
The right panel shows the NMSE after $k=$ 50 iterations (i.e., 100 half-iterations) for several numbers of measurements $M$.  
Again we see an excellent agreement between the actual MSE and the SE prediction.
In both cases we used the positive fraction $\rho=0.4$.
Analogous results are shown for MAP-ML-VAMP in Fig.~\ref{fig:randmlp_sim_map}.
There we see an excellent agreement between the actual MSE and the SE prediction for iterations $k\geq 15$ and all values of $M$.

\begin{figure}
\centering
\includegraphics[scale=.5]{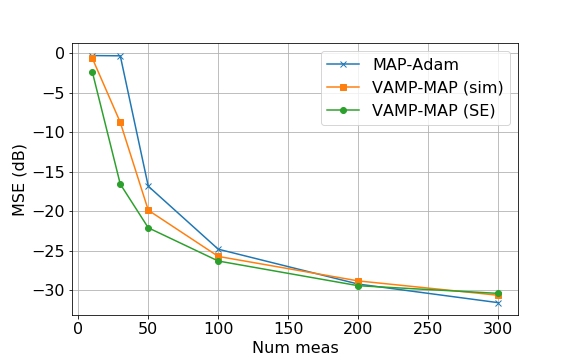}
\caption{Simulation with randomly generated neural network with MAP estimators from equation \eqref{eq:MAP_update}. Final NMSE for (a) MAP inference
computed by Adam optimizer; (b) MAP inference from ML-VAMP; (c) State evolution prediction.}
\label{fig:randmlp_sim}
\end{figure}

\paragraph{Comparison to ADAM}
We now compare the MSE of MAP-ML-VAMP and its SE to that the MAP approach \eqref{eq:likelihood_min} using the ADAM optimizer \cite{kingma2014adam}, as implemented in Tensorflow.
As before, the goal was to recover the input $\z_0^0$ to the 7-layer synthetic network from a measurement of its output.
%For a comparison of the performance of MAP estimation using ADAM (in Tensorflow) \cite{kingma2014adam} and ML-VAMP, for each value of $M$, we generate 40 random instances of the network and compute.
Fig.~\ref{fig:randmlp_sim} shows the median NMSE over 40 random network realizations for several values of $M$, the number of measurements. 
We see that, for $M\geq 100$, the performance of MAP-ML-VAMP closely matches its SE prediction, as well as the performance of the ADAM-based MAP approach \eqref{eq:likelihood_min}.
For $M<100$, there is a discrepancy between the MSE performance of MAP-ML-VAMP and its SE prediction, which is likely due to the relatively small dimensions involved.
Also, for small $M$, MAP-ML-VAMP appears to achieve slightly better MSE performance than the ADAMP-based MAP approach \eqref{eq:likelihood_min}.
Since both are attempting to solve the same problem, the difference is likely due to ML-VAMP finding better local minima.

\subsection{Image Inpainting: MNIST dataset}

To demonstrate that ML-VAMP can also work on a real-world dataset, we perform inpainting on the MNIST dataset. 
The MNIST dataset consists of 28 $\times$ 28 = 784 pixel images of handwritten digits, as shown in the first column of Fig.~\ref{fig:mnist_inpaint}.

To start, we trained a 4-layer (deterministic) deep generative prior model from 50\,000 digits using a variational autoencoder (VAE) \cite{kingma2013auto}. 
The VAE ``decoder'' network was designed to accept 20-dimensional i.i.d.\ Gaussian random inputs $\z_0$ with zero mean and unit variance,
and to produce MNIST-like images $\xbf$.
In particular, this network began with a linear layer with 400 outputs, followed by a ReLU activations, followed by a linear layer with 784 units, followed by sigmoid activations that forced the final pixel values to between 0 and 1.

Given an image, $\xbf$, our measurement process produced $\ybf$ by erasing rows 10-20 of $\xbf$, as shown in the second column of Fig.~\ref{fig:mnist_inpaint}.  
This process is known as ``occlusion.'' 
By appending the occlusion layer onto our deep generative prior, we got a 5-layer network that generates an occluded MNIST image $\y$ from a random input $\z_0$. 
The ``inpainting problem'' is to recover the image $\x=\z_4$ from the occluded image $\y$. 

% \textr{More details about the implementation and the baseline are provided in Appendix \ref{app:simulation_details}. [Remove?]}
\begin{figure}
\centering
\includegraphics[width=0.6\columnwidth]{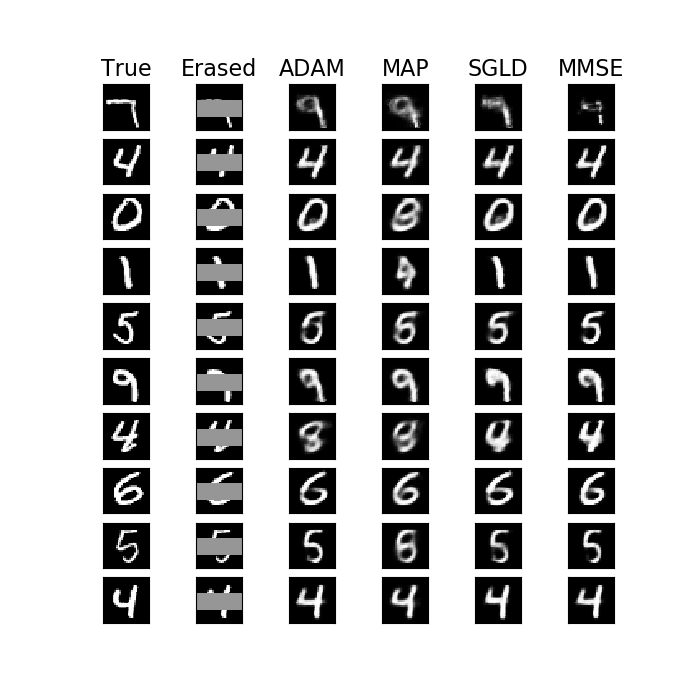}
\vspace{-0.5cm}
\caption{MNIST inpainting: Original 28$\times$28 images of handwritten digits (Col 1), with rows 10-20 are erased (Col 2). Comparison of reconstructions using MAP estimation with ADAM solver (Col 3), MAP estimation with ML-VAMP algorithm (Col 4), MMSE estimation with the SGLD approach (Col 5), and MMSE estimation with ML-VAMP algorithm (Col 6).}
\label{fig:mnist_inpaint}
\end{figure}

For this inpainting problem, we compared MAP-ML-VAMP and MMSE-ML-VAMP to the MAP estimation approach \eqref{eq:likelihood_min} using the ADAM solver, and to Stochastic Gradient Langevin Dynamics (SGLD) \cite{welling2011bayesian}, an MCMC-based sampling method that approximates $\Exp[\z|\y]$.
Example image reconstructions are shown in Fig.~\ref{fig:mnist_inpaint}. 
There we see that the qualitative performance of ML-VAMP is comparable to the baseline solvers.
%%%%%% END sections/simulations %%%%%%%%%%%%%

%%%%%% START sections/conclusion %%%%%%%%%%%%%
\section{Conclusion}
Inference using deep generative prior models provides a powerful tool for complex inverse problems. Rigorous theoretical analysis of these methods has been difficult due to the non-convex nature of the models. The ML-VAMP methodology for MMSE as well as MAP estimation provides a principled and computationally tractable method for performing the inference whose performance can be rigorously and precisely characterized in a certain large system limit.  The approach thus offers a new and potentially powerful approach for understanding and improving deep neural network based models for inference.

%%%%%% END sections/conclusion %%%%%%%%%%%%%

\appendices

%%%%%% START appendix/empirical_convergence %%%%%%%%%%%%%
\section{Empirical Convergence of Vector Sequences} \label{sec:empirical}
\label{app:empirical_convergence}
We follow the framework of Bayati and Montanari \citep{BayatiM:11}, which models
various sequences as deterministic, but with components converging empirically
to a distribution.  We start with a brief review of useful definitions.
Let $\xbf(N) = (\xbf_1,\ldots,\xbf_N)$ be a block vector with components $\xbf_n \in \R^r$
for some $r$.  Thus, the vector $\xbf(N)$
is a vector with dimension $rN$.
Given any function $g:\R^r \arr \R^s$, we define the
\emph{componentwise extension} of $g(\cdot)$ as the function,
\beq \label{eq:gcomp}
    \gbf(\xbf) := (g(\xbf_1),\ldots,g(\xbf_N)) \in \R^{Ns}.
\eeq
That is,
$\gbf(\cdot)$
applies the function
$g(\cdot)$
on each $r$-dimensional component.
Similarly,
we say $\gbf(\xbf)$ \emph{acts componentwise} on $\xbf$ whenever it is of the form \eqref{eq:gcomp}
for some function $g(\cdot)$.

Next consider a sequence of block vectors of growing dimension,
\[
    \xbf(N) = (\xbf_1(N),\ldots,\xbf_N(N)),
\qquad
N=1,\,2,\,\ldots,
\]
where each component $\xbf_n(N) \in \R^r$.
In this case, we will say that
$\xbf(N)$ is a \emph{block vector sequence that scales with $N$
under blocks $\xbf_n(N) \in \R^r$.}
When $r=1$, so that the blocks are scalar, we will simply say that
$\xbf(N)$ is a \emph{vector sequence that scales with $N$}.
Such vector sequences can be deterministic or random.
In most cases, we will omit the notational dependence on $N$ and simply write $\xbf$.

Now, given $p \geq 1$,
a function $f:\R^r \arr \R^s$ is called \emph{pseudo-Lipschitz continuous of order $p$},
if there exists a constant $C > 0$ such that for all $\xbf_1,\xbf_2 \in\R^r$,
\[
    \| f(\xbf_1)- f(\xbf_2) \| \leq C\|\xbf_1-\xbf_2\|\left[ 1 + \|\xbf_1\|^{p-1}
    + \|\xbf_2\|^{p-1} \right].
\]
Observe that in the case $p=1$, pseudo-Lipschitz continuity reduces to
usual Lipschitz continuity.
Given $p \geq 1$, we will say that the block vector sequence $\xbf=\xbf(N)$
\emph{converges empirically with $p$-th order moments} if there exists a random variable
$X \in \R^r$ such that
\begin{enumerate}[(i)]
\item $\Exp\|X\|_p^p < \infty$; and
\item for any $f : \R^r \arr \R$ that is pseudo-Lipschitz continuous of order $p$,
\beq \label{eq:PLp-empirical}
    \lim_{N \arr \infty} \tfrac{1}{N} \sum_{n=1}^N f(\xbf_n(N)) = \Exp\left[ f(X) \right].
\eeq
\end{enumerate}
In \eqref{eq:PLp-empirical}, we have
the empirical mean of the components $f(\xbf_n(N))$
of the componentwise extension $\fbf(\xbf(N))$
converging to the expectation $\Exp[ f(X) ]$.
In this case, with some abuse of notation, we will write
\beq \label{eq:plLim}
    \lim_{N \arr \infty} \left\{ \xbf_n \right\} \stackrel{PL(p)}{=} X,
\eeq
where, as usual, we have omitted the dependence on $N$ in $\xbf_n(N)$.
Importantly, empirical convergence can be defined on deterministic vector sequences,
with no need for a probability space.  If $\xbf=\xbf(N)$ is a random vector sequence,
we will often require that the limit \eqref{eq:plLim} holds almost surely.

Finally, we introduce the concept of \emph{uniform pseduo-Lipschitz continuity}.
Let $\phibf(\rbf,\gamma)$ be a function on $\rbf \in \R^r$ and $\theta \in \R^s$.
We say that $\phibf(\rbf,\theta)$ is \emph{uniformly Lipschitz continuous} in $\rbf$
at $\theta=\thetabar$ if there exists constants
$L_1,L_2 \geq 0$ and an open neighborhood $U$ of $\thetabar$ such that
\begin{subequations}
\begin{align} \label{eq:unifLip1}
    \|\phibf(\rbf_1,\theta)-\phibf(\rbf_2,\theta)\| \leq L_1\|\rbf_1-\rbf_2\|,\qquad\forall 
    \rbf_1,\rbf_2 \in \R^r,\theta \in U\\
    \label{eq:unifLip2}
    \|\phibf(\rbf,\theta_1)-\phibf(\rbf,\theta_2)\| \leq L_2\left(1+\|\rbf\|\right)
    \|\theta_1-\theta_2\|,\qquad\forall \rbf \in \R^r,\theta_1,\theta_2 \in U .
\end{align}
\end{subequations}

%%%%%% END appendix/empirical_convergence %%%%%%%%%%%%%

%%%%%% START appendix/state_evolution_equations %%%%%%%%%%%%%
\section{ML-VAMP State Evolution Equations}\label{app:mlvamp_se}

The state evolution (SE) recursively defines a set of scalar random variables that describe
the typical components of the vector quantities produced from the ML-VAMP algorithm.
The definition of the random variables are given in Algorithm \ref{algo:mlvamp_se}.
The algorithm steps mimic those in the ML-VAMP algorithm, Algorithm~\ref{algo:ml-vamp}, 
but with each update producing scalar random variables instead of vectors.
The updates use several functions:
\begin{subequations}\label{eq:se_update_functions}
\begin{align}
\MoveEqLeft f_0^0(w_0) = w_0,\qquad f^0_\ell(p^0_{\lm1},w_\ell) := f^0_\ell(p^0_{\lm1},\xi_\ell) := \phi_\ell(p^0_{\lm1},\xi_\ell), 
    \quad \ell=2,4,\ldots,L,    \label{eq:f0nonlin} \\
\MoveEqLeft     f^0_\ell(p^0_{\lm1},w_\ell) := f^0_\ell(p^0_{\lm1},(\bar{s}_\ell,\bar{b}_\ell,\bar{\xi}_\ell))
        = \bar{s}_\ell p^0_{\ell-1} + \bar{b}_\ell + \bar{\xi}_\ell, 
    \quad \ell=1,3,\ldots,\Lm1,    \label{eq:f0lin}\\
\MoveEqLeft h_\ell^\pm(p_{\ell-1}^0,p_{\ell-1}^+,q_\ell^-,w_\ell,\theta_{k\ell}^\pm) = g_\ell^\pm(q_\ell^-+q_{\ell}^0,p_{\ell-1}^++p_{\ell-1}^0,\theta_{k\ell}^\pm),\quad \ell=2,4,\ldots L-2,\\
\MoveEqLeft h_\ell^\pm(p_{\ell-1}^0,p_{\ell-1}^+,q_\ell^-,w_\ell,\theta_{k\ell}^\pm) = G_\ell^\pm(q_\ell^-+q_{\ell}^0,p_{\ell-1}^++p_{\ell-1}^0,\theta_{k\ell}^\pm),\quad \ell=1,3,\ldots L-1,\\
\MoveEqLeft h_0^+(q_{0}^-,w_{0}\theta_{k0}^+) = g^+_0(q_0^-+w_{0},\theta_{k0}^+),\quad h_L^-(p_{L-1}^0,p_{L-1}^+,w_{L},\theta_{kL}^-) = g^-_L(p_{L-1}^++p^0_{L-1},\theta_{kL}^-),\\
\MoveEqLeft f^{+}_{0}(q_{0}^-,w_0,\Lambda_{k0}^+) := \tfrac{1}{1-\alpha_{k\ell}^+}
     \left[
            h^{+}_{0}(q_{0}^-,w_0,\theta_{k0}^+) -w_0
            - \alpha_{k0}^+ q_0^- \right], \label{eq:fh0p} \\
\MoveEqLeft      f^{+}_{\ell}(p^0_{\lm1},p_{\lm1}^+,q_{\ell}^-,w_\ell,\Lambda_{k\ell}^+)
    := \tfrac{1}{1-\alpha_{k\ell}^+} 
      \left[
            h^{+}_{\ell}(p^0_{\lm1},p_{\lm1}^+,q_{\ell}^-,w_\ell,\theta_{k\ell}^+)
            - q^0_\ell - \alpha_{k\ell}^+ q_\ell^- \right], \label{eq:fhp} \\
\MoveEqLeft      f^{-}_{L}(p^0_{\Lm1},p_{\Lm1}^+,w_L,\Lambda_{kL}^-)
    := \tfrac{1}{1-\alpha_{k\ell}^-}
     \left[
            h^{-}_{L}(p^0_{\Lm1},p_{\Lm1}^+,w_L,\theta_{kL}^-) - p^0_{\Lm1}
            - \alpha_{k,\Lm1}^- p_{\Lm1}^+ \right], \label{eq:fhLn}    \\
\MoveEqLeft      f^{-}_{\ell}(p^0_{\lm1},p_{\lm1}^+,q_{\ell}^-,w_\ell,\Lambda_{k\ell}^-)
    := \tfrac{1}{1-\alpha_{k,\lm1}^-} 
     \left[
            h^{-}_{\ell}(p^0_{\lm1},p_{\lm1}^+,q_{\ell}^-,w_\ell,\theta_{k\ell}^-)
            - p^0_{\lm1}
            - \alpha_{k,\lm1}^- p_{\lm1}^+ \right]. \label{eq:fhn}
\end{align}
\end{subequations}

In addition define the \textit{perturbation} random variables $W_\ell$ (recall from \eqref{eq:varinit}) as
\begin{subequations}\label{eq:perturbations}
\begin{align}
    \MoveEqLeft W_0 = Z_0^0,\qquad\qquad    W_\ell = \Xi_\ell,\qquad \ell=2,4,\ldots,L-2,\\
    \MoveEqLeft W_\ell = (S_\ell,\wb B_\ell,\wb\Xi_\ell),\qquad \ell=1,3,\ldots,L-1.
\end{align}
\end{subequations}

\begin{algorithm}
\setstretch{1.1}
\caption{State Evolution for ML-VAMP}
\begin{algorithmic}[1]  \label{algo:mlvamp_se}

\REQUIRE{$f^0_\ell(\cdot),$ $f^\pm_{\ell}(\cdot)$ and $h^\pm_{\ell}(\cdot)$ from eqn. \eqref{eq:se_update_functions} and initial random variables:  $Z_0^0,$ $\{W_\ell, Q_{0\ell}^-\}$ 
from Section~\ref{sec:seevo} and \eqref{eq:perturbations}}

\STATE{// \texttt{Initial pass}}
    \label{line:qinit_se_mlvamp}
\STATE{$Q^0_0 = Z_0^0$, $\tau^0_0 = \Exp(Q^0_0)^2$ and $P^0_0 \sim \Norm(0,\tau^0_0)$} \label{line:p0init_se_mlvamp}
\FOR{$\ell=1,\ldots,\Lm1$}
    \STATE{$Q^0_\ell=f^0_\ell(P^0_{\lm1},W_\ell)$}
    \label{line:q0_init}
    \STATE{$P^0_\ell \sim \Norm(0,\tau^0_\ell)$,
            $\tau^0_\ell = \Exp(Q^0_\ell)^2$} \label{line:pinit_se_mlvamp}
\ENDFOR
\STATE{}

\FOR{$k=0,1,\dots$}
    \STATE{// \texttt{Forward Pass} }
    %\STATE{$\lambdabar^+_{k0} = T_{k0}^+(\mubar^+_{k0},\Lambdabar_{0k}^-), \quad
    %    \mubar^+_{k0} = \Exp(\varphi_{k0}^+(Q_{k0}^-,W_0,\Lambdabar_{0k}^-))$}    \label{line:mup0_se_mlvamp}
    %\STATE{$\Lambdabar_{k0}^+ = (\Lambdabar_{k1}^-,\lambdabar^+_{k0})$} \label{line:lamp0_se_mlvamp}
    \STATE{$\wh{Q}^+_{k0} = h^+_0(Q_{k0}^-,W_0,\wb\theta^+_{k0}))$} \label{line:qhat0_se_mlvamp}
    \STATE{$\alphabar_{k0}^+ = \Exp(\tfrac{\partial h^+_0}{\partial Q_{k0}^-}(Q_{k0}^-,W_0,\wb\theta^+_{k0})),\qquad \Lambdabar_{k0}^+ = (\alphabar^+_{k0},\wb\theta_{k0}^+)$}
%    \STATE{$$}
    \STATE{$Q_{k0}^+ = f^+_{0}(Q_{0}^-,W_0,\Lambdabar^+_{k0})$}  \label{line:q0_se_mlvamp}
    \STATE{$(P^0_0,P_{k0}^+) \sim \Norm(\zero,\Kbf_{k0}^+)$,
        $\qquad\Kbf_{k0}^+ := \Cov(Q^0_0,Q_{k0}^+)$} \label{line:p0_se_mlvamp}
    \FOR{$\ell=1,\ldots,L-1$}
        \STATE{$\wh{Q}^+_{k\ell} = h^+_\ell(P^0_{\lm1},P^+_{k,\lm1},Q_{k\ell}^-,
            W_\ell,\wb\theta^+_{k\ell}))$} \label{line:qhat_se_mlvamp}
        \STATE{$\alphabar_{k\ell}^+ = \Exp(\tfrac{\partial h^+_\ell}{\partial Q_{k\ell}^-}(P^0_{\lm1},P^+_{k,\lm1},Q_{k\ell}^-,
            W_\ell,\wb\theta^+_{k\ell})),\qquad \Lambdabar_{k\ell}^+ = (\alphabar^+_{k\ell},\wb\theta_{k\ell}^+)$}
%        \STATE{$ $}
            \label{line:lamp_se_mlvamp}
        \STATE{$Q_{k\ell}^+ = f^+_{\ell}(P^0_{\lm1},P^+_{k,\lm1},Q_{k\ell}^-,W_\ell,\Lambdabar^+_{k\ell})$}
            \label{line:qp_se_mlvamp}
        \STATE{$(P^0_\ell,P_{k\ell}^+) \sim \Norm(\zero,\Kbf_{k\ell}^+)$,
            $\qquad\Kbf_{k\ell}^+ := \Cov(Q^0_\ell,Q_{k\ell}^+) $}   \label{line:pp_se_mlvamp}
    \ENDFOR
%    \STATE{}
\vspace{10pt}

    \STATE{// \texttt{Backward Pass} }
    \STATE{$\wh{P}_{\kp1,\Lm1}^- = h^-_{L}(P^0_{\Lm1},P_{k,\Lm1}^+,W_L,\wb\theta^-_{\kp1,L})$}
        \label{line:phatL_se_mlvamp}
    \STATE{$\alphabar_{k+1,L}^- = \Exp(\tfrac{\partial h^-_{L}}{\partial P_{k,\Lm1}^+}(P^0_{\Lm1},P_{k,\Lm1}^+,W_L,\wb\theta^-_{\kp1,L})),\qquad\Lambdabar_{k+1,L}^- = (\alphabar^-_{k+1,L},\wb\theta_{k+1,L}^-)$}
    \STATE{$P_{\kp1,\Lm1}^- = f^-_{L}(P^0_{\Lm1},P_{k,\Lm1}^+,W_L,\Lambdabar^-_{\kp1,L})$}  \label{line:pL_se_mlvamp}
%    \STATE{$ $} \label{line:tauL_se_mlvamp}
    \STATE{$Q_{\kp1,\Lm1}^- \sim \Norm(0,\tau_{\kp1,\Lm1}^-),\qquad \tau_{\kp1,\Lm1}^- := \Exp(P^-_{\kp1,\Lm1})^2$} \label{line:qL_se_mlvamp}
    \FOR{$\ell=\Lm1,\ldots,1$}
        \STATE{$\wh{P}_{\kp1,\lm1}^- = h^-_{\ell}(P^0_{\lm1},P_{k,\lm1}^+,W_\ell,\wb\theta^-_{\kp1,\ell})$}
            \label{line:phatn_se_mlvamp}
        \STATE{$\alphabar_{k+1,\ell}^- = \Exp(\tfrac{\partial h^-_{\ell}}{\partial P_{k,\Lm1}^+}(P^0_{\lm1},P_{k,\lm1}^+,W_\ell,\wb\theta^-_{\kp1,\ell})),\qquad \Lambdabar_{k+1,\ell}^- = (\alphabar^-_{k+1,\ell},\wb\theta_{k+1,\ell}^-)$}
%        \STATE{$ $}
       \STATE{$P_{\kp1,\lm1}^- =
        f^-_{\ell}(P^0_{\lm1},P^+_{k,\lm1},Q_{\kp1,\ell}^-,W_\ell,\Lambdabar^-_{k\ell})$}
            \label{line:pn_se_mlvamp}
%        \STATE{$\tau_{\kp1,\lm1}^- = \Exp(P_{\kp1,\lm1}^-)^2$} \label{line:taun_se_mlvamp}
        \STATE{$Q_{\kp1,\lm1}^- \sim \Norm(0,\tau_{\kp1,\lm1}^-),\qquad \tau_{\kp1,\lm1}^- := \Exp(P_{\kp1,\lm1}^-)^2$}   \label{line:qn_se_mlvamp}
    \ENDFOR

\ENDFOR
\end{algorithmic}
\end{algorithm}
%%%%%% END appendix/state_evolution_equations %%%%%%%%%%%%%

%%%%%% START appendix/fixed_points %%%%%%%%%%%%%
%\begin{appendices}

\section{Proofs of ML-VAMP Fixed-Point Theorems}
\label{app:fixed_points}
\subsection{Proof of Theorem~\ref{thm:mapfix}}

The linear equalities in defining $\sbf_{k\ell}^\pm$ can be rewritten as,
\begin{align}\label{eq:defr}
    \rbf_{k\ell}^+ = \zbfhat^+_{k\ell} + \frac{1}{\alpha^{-}_{k\ell}}\sbf^{+}_{k\ell}, \qquad
    \rbf_{k+1,\ell}^- = \zbfhat^-_{k\ell} - \frac{1}{\alpha^{+}_{k\ell}}\sbf^{-}_{\kp1,\ell} %    \label{eq:definern}
\end{align}
Substituting \eqref{eq:defr}  in lines \ref{line:rp} and \ref{line:rn} of Algorithm \ref{algo:ml-vamp} give the updates \eqref{eq:admmsp} and \eqref{eq:admmsn} in Theorem \ref{thm:mapfix}. It remains to show that the optimization problem in updates
\eqref{eq:admmHp} and \eqref{eq:admmHn} is equivalent to
\eqref{eq:MAP_update}. It suffices to show that the terms dependent on $(\z_{\ell-1}^-,\z^{+}_\ell)$ in $b_\ell$ from  \eqref{eq:MAP_update}, and $\mc L_\ell$ from \eqref{eq:admmHp} and \eqref{eq:admmHn} are identical. This follows immediately on substituting \eqref{eq:defr} in \eqref{eq:def_belief}. Thus there exists a bijective mapping between the fixed points $\{\zhat,\r^{+},\r^{-}\}$ (of Algorithm \ref{algo:ml-vamp}) and $\{\zhat,\s\}$ (of Theorem \ref{thm:mapfix}).

It now remains to be shown that any fixed point of Algorithm~\ref{algo:ml-vamp} is a critical point of the augmented Lagrangian in \eqref{eq:Lagdef}.
%Since we are looking only at fixed points, we can drop the dependence
%on the iteration $k$.  So, for example, we can write $\rbf_\ell^+$ for $\rbf_{k\ell}^+$.
To that end, we need to show that
there exists dual parameters $\sbf_\ell$ such that for all $\ell=0,\ldots,\Lm1$,
% \begin{subequations}
\begin{gather}
% \label{eq:zcon}
    \zbfhat^+_\ell=\zbfhat^-_\ell,\qquad
     \label{eq:Laggrad}
    \partial_{\zbf^+} \mc L(\zbfhat^+,\zbfhat^-,\sbf)\owns\zero, \quad
    \partial_{\zbf^-} \mathcal{L} (\zbfhat^+,\zbfhat^-,\sbf)\owns \zero,
\end{gather}
% \end{subequations}
where $\mc L(\cdot)$ is the Lagrangian in \eqref{eq:Lagdef}. Primal feasibility or $\zbfhat^+_\ell=\zbfhat^-_\ell$ was already shown in \eqref{eq:z_equal}.
%At any fixed point of \eqref{eq:gamupdate}, we have
%\[
%    \eta_\ell=\gamma^+_\ell + \gamma^-_\ell = \frac{\gamma^+_\ell}{\alpha_\ell^-} =
%     \frac{\gamma_\ell^-}{\alpha_\ell^+}.
%\]
%Therefore,
%\beq\label{eq:alphaone}
%    \alpha_\ell^- = \frac{\gamma^+_\ell}{\gamma_\ell^++\gamma_\ell^-}
%        = 1- \frac{\gamma^-_\ell}{\gamma_\ell^++\gamma_\ell^-} = 1-\alpha^+_\ell.
%\eeq
%% Hence,
%% \beq \label{eq:alphaone}
%%     \alpha^+_\ell + \alpha^-_\ell = 1.
%% \eeq
%Now, from line~\ref{line:rp} in Algorithm~\ref{algo:ml-vamp},
%\begin{align}
%     \zbfhat^+_\ell &= (1-\alpha^+_\ell)\rbf^+_\ell + \alpha^+_\ell\rbf^-_\ell  = \alpha^-_\ell\rbf^+_\ell + \alpha^+_\ell\rbf^-_\ell, \label{eq:zhatfix1}
%\end{align}
%where the last step used \eqref{eq:alphaone}.  Similarly, from line~\ref{line:rn},
%\beq
%    \zbfhat^-_\ell = \alpha^-_\ell\rbf^+_\ell + \alpha^+_\ell\rbf^-_\ell. \label{eq:zhatfix2}
%\eeq
%Equations \eqref{eq:zhatfix1} and \eqref{eq:zhatfix2} prove \eqref{eq:zcon}.
As a consequence of the primal feasibility $\zbfhat^+_\ell=\zbfhat^-_\ell$, observe that
\beq
\sbf_\ell^+-\sbf_\ell^- = (\alpha_\ell^++\alpha_\ell^-)\zbfhat_\ell - \alpha^+_\ell\rbf_\ell^--\alpha^-_\ell\rbf_\ell^+=0,
\eeq
where we have used \eqref{eq:fixed_point}. Define $\sbf:=\sbf^+=\sbf^-$. To show the stationarity in \eqref{eq:Laggrad} it suffices to show that $\sbf_\ell$ is a valid dual parameter for which the following stationarity conditions hold,
% \begin{subequations}
% \begin{align} \label{eq:Laggrad_l}
%     \frac{\partial \mc L_\ell(\zbf^-_{\ell-1},\zbf^+_{\ell};\zbfhat^+_{\ell-1},\zbfhat^-_{\ell},\sbf_{\ell-1},\sbf_{\ell})}{\partial \zbf^-_{\ell-1}}\Bigg{\rvert}_{(\zbfhat_{\ell-1}^-,\zbfhat_{\ell}^+)}\owns\ &\ {\bf 0}, \\
%     \frac{\partial \mc L_\ell(\zbf^-_{\ell-1},\zbf^+_{\ell};\zbfhat^+_{\ell-1},\zbfhat^-_{\ell},\sbf_{\ell-1},\sbf_{\ell})}{\partial \zbf^+_\ell}\Bigg{\rvert}_{(\zbfhat_{\ell-1}^-,\zbfhat_{\ell}^+)}\owns\ &\ {\bf 0}.
% \end{align}
% \end{subequations}
\begin{align} \label{eq:Laggrad_l}
    \partial_{\zbf^-_{\ell-1}} \mc L_\ell(\zbfhat_{\ell-1}^-,\zbfhat_{\ell}^+;\zbfhat^+_{\ell-1},\zbfhat^-_{\ell},\sbf_{\ell-1},\sbf_{\ell})\owns\ {\bf 0}, \qquad
    \partial_{\zbf^+_\ell} \mc L_\ell(\zbfhat_{\ell-1}^-,\zbfhat_{\ell}^+;\zbfhat^+_{\ell-1},\zbfhat^-_{\ell},\sbf_{\ell-1},\sbf_{\ell})\owns\ {\bf 0},
\end{align}
Indeed the above conditions are the stationarity conditions of the optimization problem in \eqref{eq:admmHp} and \eqref{eq:admmHn}. Hence \eqref{eq:Laggrad} holds.

% \begin{proof}
%Recall the definition of the MMSE denoisers $\gbf_\ell^\pm$ given in equation \eqref{eq:MMSE_update}. Hence the fixed points $\{\r^+_\ell,\r^-_\ell,\zhat_\ell,\gamma_\ell^+,\gamma_\ell^-,\eta_\ell\}$ satisfy the following set of equations
%\begin{gather*}
%\zhat_\ell = \Exp[\z_\ell|b_\ell],\qquad{\rm and}\qquad
%\zhat_\ell = \Exp[\z_\ell|b_{\ell+1}].
%\end{gather*}
%We show that the quantities $\{\gamma_\ell^-\r_\ell^-,\gamma_{\ell}^+\r_{\ell}^+,\tfrac{\gamma_{\ell}^-}2,\tfrac{\gamma_{\ell}^+}2\}$ are valid Lagrange multipliers for the equality constraints \eqref{eq:moment_matching} for \eqref{eq:BFE_EC} for which the KKT conditions are satisfied.
\subsection{Proof of Theorem \ref{thm:fixed_point_mmse}}
Observe that the Lagrangian function for the constrained optimization problem \eqref{eq:BFE_EC} for this specific choice of Lagrange multipliers is given by
\begin{align*}
    \mc L(\{b_\ell\},\{q_\ell\},\{\r^+_\ell\},\{\r^-_\ell\},\{\gamma^+_\ell\},\{\gamma^-_\ell\}) = \sum_{\ell=0}^L D_{\mathsf{KL}}(f_\ell(\z_\ell,\z_{\ell-1})||p(\z_\ell|\z_{\ell-1}))+\sum_{\ell=0}^{L-1}H(q_{\ell})\\
    +\sum_{\ell=0}^{L-1}\gamma_\ell^-\r_\ell^{-\top}(\Exp[\z_\ell|f_\ell]-\Exp[\z_\ell|q_\ell])+\gamma_{\ell}^+\r_{\ell}^{+\top}(\Exp[\z_{\ell}|f_{\ell+1}]-\Exp[\z_{\ell-1}|q_{\ell}])\\
    +\sum_{\ell=0}^{L-1}\tfrac{\gamma_\ell^-}2(\Exp[\norm{\z_\ell}^2|f_\ell]-\Exp[\norm{\z_\ell}^2|q_\ell])+\tfrac{\gamma_{\ell}^+}2(\Exp[\norm{\z_{\ell}}^2|f_{\ell+1}]-\Exp[\norm{\z_{\ell}}^2|q_{\ell}])
\end{align*}
Notice that the stationarity KKT conditions $\nabla_{f_\ell} \mc L=\bm{0}$ and $\nabla_{q_\ell} \mc L=\bm{0}$ give us the relation
\begin{subequations}\label{eq:stationarity_KKT}
\begin{align}\label{eq:b_stationarity}
    f^*_\ell(\z_\ell,\z_{\ell-1}) &\propto p(\z_\ell|\z_{\ell-1})\mc \exp\left(-\tfrac{\gamma_\ell^-}2\norm{\z_\ell-\r_\ell^-}^2-\tfrac{\gamma_{\ell-1}^+}2\norm{\z_{\ell-1}-\r_{\ell-1}^+}^2\right)\\
    \label{eq:q_stationarity}
    q^*_\ell(\z_\ell) &\propto \mc \exp\left(-\tfrac{\gamma_\ell^- +\gamma_\ell^+}2\norm{\z_\ell-\tfrac{\gamma_\ell^-\r^-_\ell+\gamma_\ell^+\r_\ell^+}{\gamma_\ell^- +\gamma_\ell^+} }^2\right)
\end{align}
\end{subequations}
where notice that $f_\ell^*=b_\ell$ from  \eqref{eq:def_belief}. The primal feasibility KKT conditions \eqref{eq:moment_matching} result in
\begin{align*}
    \Exp\left[\z_\ell|f^*_\ell\right]=\Exp\left[\z_\ell|b_\ell\right]&=\frac{\gamma_\ell^-\r^-_\ell+\gamma_\ell^+\r_\ell^+}{\gamma_\ell^-+\gamma_\ell^+}\\
    \Exp\left[\norm{\z_\ell}^2|f^*_\ell\right]=\Exp\left[\norm{\z_\ell}^2|b_\ell\right]&=(\gamma_{\ell}^-+\gamma_{\ell}^+)^{-1}
\end{align*}
where we have used the Gaussianity of $q_{
\ell}$ from \eqref{eq:q_stationarity} and relation of $f^*_\ell=b_\ell$ from \eqref{eq:b_stationarity} and \eqref{eq:def_belief}.
%Due to equality constraints, we have no dual feasibility or complementary slackness conditions on the Lagrangian multipliers.
The quantity on the right is exactly $\zhat_\ell$ for any fixed point of MMSE-ML-VAMP as evident from \eqref{eq:z_equal}. The claim follows from the update \eqref{eq:MMSE_update}.
% \end{proof}
%%%%%% END appendix/fixed_points %%%%%%%%%%%%%

%%%%%% START appendix/gen_ml_recursions %%%%%%%%%%%%%
\section{General Multi-Layer Recursions}
\label{app:general_convergence}
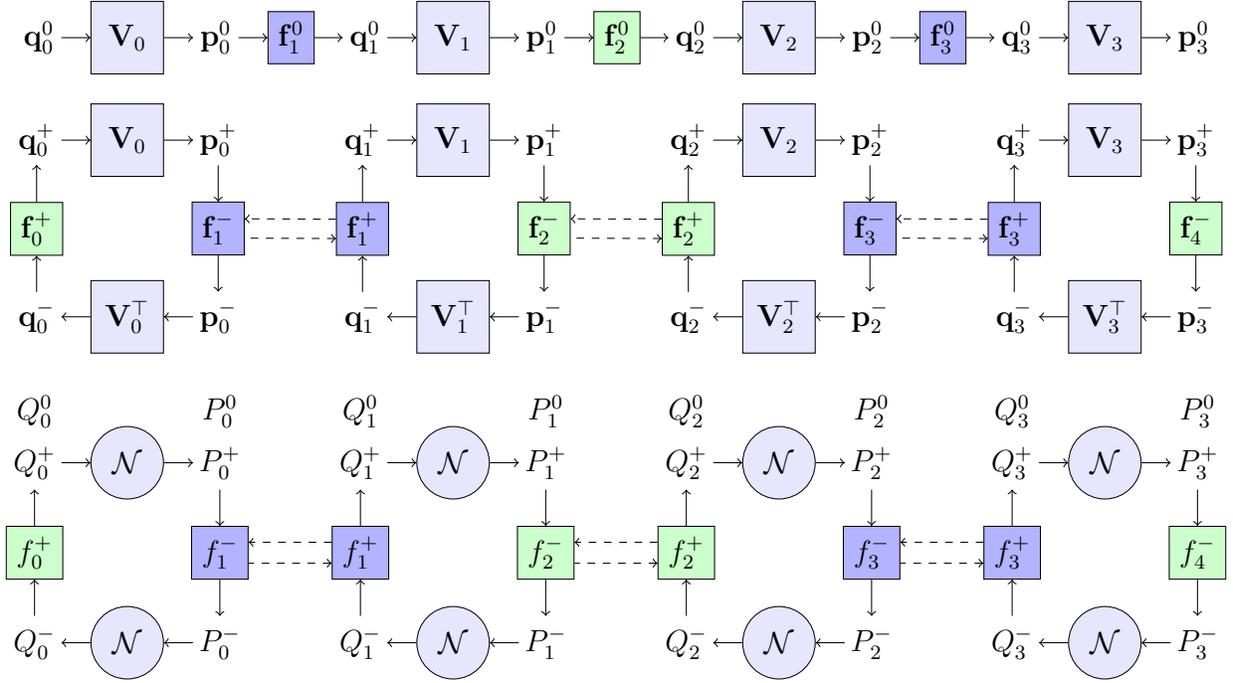
\begin{figure}[ht]
\resizebox{\textwidth}{!}{

%%%%% START tikz/mlvamp_equivalent_system %%%%%
\begin{tikzpicture}[scale = 0.9]

    \pgfmathsetmacro{\sep}{3};
    \pgfmathsetmacro{\yoff}{0.4};
    \pgfmathsetmacro{\xoffa}{0.3};
    \pgfmathsetmacro{\xoffb}{0.6};
    \tikzstyle{var}=[draw=white,circle,fill=white!100,node distance=0cm,inner sep=0cm];
    \tikzstyle{rot}=[draw,fill=blue!10, minimum size=1cm, node distance=\sep cm];
    \tikzstyle{Dnl}=[draw,fill=green!20, minimum size=.5cm, node distance=\sep cm];
    \tikzstyle{Dl}=[draw,fill=blue!30, minimum size=.5cm, node distance=\sep cm];
    \tikzstyle{ROT}=[draw,circle,fill=blue!10, minimum size=1cm, node distance=\sep cm];
    \tikzstyle{dnl}=[draw,fill=green!20, minimum size=.5cm, node distance=\sep cm];
    \tikzstyle{dl}=[draw,fill=blue!30, minimum size=.5cm, node distance=\sep cm];

\newcommand{\Vbm}{\bm V}
\newcommand{\pbm}{\bm p}
\newcommand{\qbm}{\bm q}
\newcommand{\ubm}{\bm u}
\newcommand{\tbm}{\bm t}
\newcommand{\Abm}{\bm A}
\def\xoff{.4cm}
\def\xoffblock{1.3cm}
\def\yoff{.5cm}
\def\yspace{.4cm}
\def\Yspace{1cm}
\def\yvec{0mm}

    % Dynamical System for general convergence result
        \node [var] (q0p) {$\q^+_0$};
        
        \def\loopend{4}
\foreach \i/\j in {0/1,1/2,2/3,3/\loopend} {
    \node [rot, right =\xoff of q\i p] (V\i ) {$\V_\i $};
        {
        \node [rot, above = \yspace of V\i] (V\i star){$\V_\i$};
        \node [var, right=\xoff of V\i star] (p\i star) {$\p_\i ^0$};
        \ifthenelse{\j<\loopend}{
            \ifthenelse{\intcalcMod{\j}{2}>0}{
                \node [Dl, right =\xoff of p\i star] (phi\j) {$\fbf_\j^0$};
                }
                {
                \node [Dnl, right =\xoff of p\i star] (phi\j) {$\fbf_\j^0$};
            }
        }{}
        \node [var, left=\xoff of V\i star] (q\i star) {$\q_\i ^0$};
        }
        
    \node [var, right=\xoff of V\i ] (p\i p) {$\p_\i ^+$};
    
    \ifthenelse{\intcalcMod{\j}{2}>0}{
    \node [Dl, below =\yoff of p\i p] (F\j m) {$\fbf_\j^-$};
    }
    {
    \node [Dnl, below =\yoff of p\i p] (F\j m) {$\fbf_\j^-$};
    }
    
    \node [var, below=\yoff of F\j m] (p\i m) {$\p_\i ^-$};
	\node [rot, left =\xoff of p\i m] (V\i T) {$\V_\i \T$};
	    {
	    \node [ROT, below = \Yspace of V\i T] (N\i){$\mc N$};
	    \node [var, right=\xoff of N\i ] (P\i p) {$P_\i ^+$};
	    \node [var, above=\yvec of P\i p](P\i 0){$P_\i ^0$};
        \ifthenelse{\intcalcMod{\j}{2}>0}{
        \node [dl, below =\yoff of P\i p] (f\j m) {$f_\j^-$};
        }
        {
        \node [dnl, below =\yoff of P\i p] (f\j m) {$f_\j^-$};
        }
        
        \node [var, below=\yoff of f\j m] (P\i m) {$P_\i ^-$};
    	\node [ROT, left =\xoff of P\i m] (N\i T) {$\mcN$};
        \node [var, left =\xoff of N\i T] (Q\i m) {$Q^-_\i $};
        \ifthenelse{\intcalcMod{\j}{2}>0}{
        \node [dnl, above =\yoff of Q\i m] (f\i p) {$f_\i^+$};
        }
        {	
        \node [dl, above =\yoff of Q\i m] (f\i p) {$f_\i^+$};
        }
    	\node [var, above =\yoff of f\i p] (Q\i p) {$Q^+_\i $};
    	\node [var, above = \yvec of Q\i p](Q\i 0){$Q^0_\i$};
	    }

    \node [var, left =\xoff of V\i T] (q\i m) {$\q^-_\i $};
    \ifthenelse{\intcalcMod{\j}{2}>0}{
    \node [Dnl, above =\yoff of q\i m] (F\i p) {$\fbf_\i^+$};
    }
    {	
    \node [Dl, above =\yoff of q\i m] (F\i p) {$\fbf_\i^+$};
    }
    \ifthenelse{\j<\loopend}{
	\node [var, right =\xoffblock of p\i p] (q\j p) {$\q^+_\j$};
	}{}
}

% Arrows

\foreach \i/\j in {0/1,1/2,2/3,3/\loopend} {

\draw[->] (q\i star.east) -- (V\i star.west);
\draw[->] (V\i star.east) -- (p\i star.west);
\ifthenelse{\j<\loopend}{
    \draw[->] (p\i star.east) -- (phi\j.west);
    \draw[->] (phi\j.east) -- (q\j star.west);
}

\draw[->] (q\i p.east) -- (V\i.west);
\draw[->] (V\i.east) -- (p\i p.west);
\draw[->] (p\i p.south) -- (F\j m.north);
\draw[->] (F\j m.south) -- (p\i m.north);
\draw[->] (p\i m.west) -- (V\i T.east);
\draw[->] (V\i T.west) -- (q\i m.east);
\draw[->] (q\i m.north) -- (F\i p.south);
\draw[->] (F\i p.north) -- (q\i p.south);

\ifthenelse{\i>0}{
\draw[->,dashed] (F\i p.160) -- (F\i m.20);
\draw[<-,dashed] (F\i p.-160) -- (F\i m.-20);
\draw[->,dashed] (f\i p.160) -- (f\i m.20);
\draw[<-,dashed] (f\i p.-160) -- (f\i m.-20);
}{}

\draw[->] (Q\i p.east) -- (N\i.west);
\draw[->] (N\i.east) -- (P\i p.west);
\draw[->] (P\i p.south) -- (f\j m.north);
\draw[->] (f\j m.south) -- (P\i m.north);
\draw[->] (P\i m.west) -- (N\i T.east);
\draw[->] (N\i T.west) -- (Q\i m.east);
\draw[->] (Q\i m.north) -- (f\i p.south);
\draw[->] (f\i p.north) -- (Q\i p.south);

}
\end{tikzpicture}

%%%%% END tikz/mlvamp_equivalent_system %%%%%

}
\caption{\label{fig:mlvamp_error_system}
(TOP) The equations \eqref{eq:nntrue} with equivalent quantities defined in \eqref{eq:pq0}. $\fbf_\ell^0$ defined using \eqref{eq:f0nonlin} and \eqref{eq:f0lin}.\newline
(MIDDLE) The GEN-ML recursions in Algorithm \ref{algo:gen}. These are also equivalent to ML-VAMP recursions from Algorithm \ref{algo:ml-vamp} (See Lemma \ref{lem:special_case}) if $\p^\pm,\q^\pm$ are as defined in equations \eqref{eq:pqdef} and $\fbf_\ell^\pm$ given by equations (\ref{eq:fh0p}-\ref{eq:fhn}).\newline
(BOTTOM) Quantities in the GEN-ML-SE recursions. These are also equivalent to ML-VAMP SE recursions from Algorithm \ref{algo:mlvamp_se} (See Lemma \ref{lem:special_case})
}
\end{figure}

To analyze Algorithm~\ref{algo:ml-vamp}, we consider a more general class
of recursions as given in Algorithm~\ref{algo:gen} and depicted in Fig. \ref{fig:mlvamp_error_system}.
The Gen-ML recursions generates
(i) a set of \textit{true vectors} $\q_\ell^0$ and $\p_\ell^0$ 
and (ii) \textit{iterated vectors} $\qbf^{\pm}_{k\ell}$ and $\pbf^{\pm}_{k\ell}$. The true vectors are generated by a single forward pass, whereas the iterated vectors are generated
via a sequence of forward and backward passes through a multi-layer system.
In proving the State Evolution for the ML-VAMP algorithm, one would then associate the terms $\qbf^{\pm}_{k\ell}$ and $\pbf^{\pm}_{k\ell}$
with certain error quantities in the ML-VAMP recursions. To account for the effect of the parameters $\gamma^{\pm}_{k\ell}$ and $\alpha^{\pm}_{k\ell}$
in ML-VAMP, the Gen-ML algorithm describes the parameter update through a sequence of
\emph{parameter lists} $\Lambda^{\pm}_{k\ell}$.
The parameter lists are ordered lists of parameters that accumulate as the
algorithm progresses. The true and iterated vectors from Algorithm \ref{algo:gen} are depicted in the signal flow graphs on the (TOP) and (MIDDLE) panel of Fig. \ref{fig:mlvamp_error_system} respectively. The iteration index $k$ for the iterated vectors $\q_{k\ell},\p_{k\ell}$ has been dropped for simplifying notation.

The functions $\fbf_\ell^0(\cdot)$ that produce the true vectors $\q_\ell^0,\p_\ell^0$ are called \textit{initial vector functions} and use the initial parameter list $\Lambda_{01}^-$. The functions $\fbf_{k\ell}^{\pm}(\cdot)$ that produce the vectors
$\qbf^{\pm}_{k\ell}$ and $\pbf^{\pm}_{k\ell}$ are  called the \emph{vector update functions} and use parameter lists $\Lambda_{kl}^\pm$.
The parameter lists are initialized with $\Lambda^-_{01}$ in line~\ref{line:laminit_gen}. 
As the algorithm progresses, new parameters $\lambda^{\pm}_{k\ell}$
are computed and then added to the lists in lines~\ref{line:lamp0_gen}, \ref{line:lamp_gen}, \ref{line:lamL_gen}
and \ref{line:lamn_gen}.  The vector update functions $\fbf_{k\ell}^{\pm}(\cdot)$ may depend on any sets of parameters accumulated in the parameter list. 
In lines~\ref{line:mup0_gen}, \ref{line:mup_gen}, \ref{line:muL_gen} and \ref{line:mun_gen},
the new parameters $\lambda_{k\ell}^{\pm}$ are computed by:
(1) computing average values $\mu_{k\ell}^{\pm}$ of componentwise functions $\varphibf^{\pm}_{k\ell}(\cdot)$;
and (2) taking functions $T^{\pm}_{k\ell}(\cdot)$ of the average values $\mu_{k\ell}^{\pm}$.
Since the average values $\mu_{k\ell}^{\pm}$ represent statistics on the components of
$\varphibf^{\pm}_{k\ell}(\cdot)$, we will call $\varphibf^{\pm}_{k\ell}(\cdot)$ the \emph{parameter statistic
functions}.  We will call the $T^{\pm}_{k\ell}(\cdot)$ the \emph{parameter update functions}.
The functions $\fbf_\ell^0,\fbf_{k\ell}^\pm,\varphibf^\pm_\ell$ also take as input some perturbation vectors $\w_\ell$.

\old{We will show below that the updates for the parameters $\gamma^{\pm}_{k\ell}$ and $\alpha^{\pm}_{k\ell}$
can be written in this form.}

%%%%%%%%%%%%%%%%%%%%%%%%%%%%%%%%%%%

%%%%% START algo/gen_ml_recursions %%%%%
\begin{algorithm}[t]
\setstretch{1.1}
\caption{General Multi-Layer (Gen-ML) Recursion }
\begin{algorithmic}[1]  \label{algo:gen}
\REQUIRE{Initial vector functions $\fbf_\ell^0$, vector update functions $\fbf^\pm_{k\ell}(\cdot)$,
parameter statistic functions $\varphibf^\pm_{k\ell}(\cdot)$,
parameter update functions $T^{\pm}_{k\ell}(\cdot)$,
orthogonal matrices $\Vbf_\ell$,
disturbance vectors $\wbf^\pm_\ell$. }

\STATE{// \texttt{Initialization} }
\STATE{Initialize parameter list $\Lambda_{01}^-$ and vectors $\pbf_0^0$ and $\qbf_{0\ell}^-$ for $\ell=0,\ldots,\Lm1$}
    \label{line:laminit_gen}
\STATE{$\qbf^0_0 = \fbf^0_0(\wbf_0), \quad \pbf^0_0 = \Vbf_0\qbf^0_0$} \label{line:q00init_gen}
\FOR{$\ell=1,\ldots,\Lm1$}
    \STATE{$\qbf^0_\ell = \fbf^0_\ell(\pbf^0_{\lm1},\wbf_\ell, \Lambda_{01}^-)$ }
    \label{line:q0init_gen}
    \STATE{$\pbf^0_\ell = \Vbf_\ell\qbf^0_\ell$ }  \label{line:p0init_gen}
\ENDFOR \label{line:end_initial_for}
%\STATE{Initialize parameter list:  $\Lambda_{01}^-$} \label{line:laminit_gen}
%\STATE{Initialize vectors:  $\qbf_{0\ell}^-$, $\ell=0,\ldots,\Lm1$} \label{line:qinit_gen}
\STATE{}
\FOR{$k=0,1,\dots$}\label{line:start_algo_for}
    \STATE{// \texttt{Forward Pass} }
    \STATE{$\lambda^+_{k0} = T_{k0}^+(\mu^+_{k0},\Lambda_{0k}^-), \quad
        \mu^+_{k0} = \bkt{\varphibf_{k0}^+(\qbf_{k0}^-,\wbf_0,\Lambda_{0k}^-)}$}    \label{line:mup0_gen}
    \STATE{$\Lambda_{k0}^+ = (\Lambda_{k1}^-,\lambda^+_{k0})$} \label{line:lamp0_gen}
    \STATE{$\qbf_{k0}^+ = \fbf^+_{k0}(\qbf_{k0}^-,\wbf_0,\Lambda^+_{k0})$}  \label{line:q0_gen}
    \STATE{$\pbf_{k0}^+ = \Vbf_0\qbf_{k0}^+$} \label{line:p0_gen}
    \FOR{$\ell=1,\ldots,L-1$}
        \STATE{$\lambda^+_{k\ell} = T_{k\ell}^+(\mu^+_{k\ell},\Lambda_{k,\lm1}^+), \quad
            \mu^+_{k\ell} = \bkt{\varphibf_{k\ell}^+(\pbf^0_{\lm1},\pbf^+_{k,\lm1},\qbf_{k\ell}^-,\wbf_\ell,\Lambda_{k,\lm1}^+)}$}    \label{line:mup_gen}
        \STATE{$\Lambda_{k\ell}^+ = (\Lambda_{k,\lm1}^+,\lambda^+_{k\ell})$}
            \label{line:lamp_gen}
        \STATE{$\qbf_{k\ell}^+ = \fbf^+_{k\ell}(\pbf^0_{\lm1},\pbf^+_{k,\lm1},\qbf_{k\ell}^-,\wbf_\ell,\Lambda^+_{k\ell})$}
            \label{line:qp_gen}
        \STATE{$\pbf_{k\ell}^+ = \Vbf_{\ell}\qbf_{k\ell}^+$}   \label{line:pp_gen}
    \ENDFOR
    \STATE{}

    \STATE{// \texttt{Backward Pass} }
    \STATE{$\lambda^-_{\kp1,L} = T_{kL}^-(\mu^-_{kL},\Lambda_{k,\Lm1}^+), \quad
        \mu^-_{kL} = \bkt{\varphibf_{kL}^-(\pbf_{k,\Lm1}^+,\wbf_L,\Lambda_{k,\Lm1}^+)}$}    \label{line:muL_gen}
    \STATE{$\Lambda_{\kp1,L}^- = (\Lambda_{k,\Lm1}^+,\lambda^+_{\kp1,L})$} \label{line:lamL_gen}
    \STATE{$\pbf_{\kp1,\Lm1}^- = \fbf^-_{kL}(\pbf^0_{\Lm1},\pbf_{k,\Lm1}^+,\wbf_L,\Lambda^-_{\kp1,L})$}  \label{line:pL_gen}
    \STATE{$\qbf_{\kp1,\Lm1}^- = \Vbf_{\Lm1}\tran\pbf_{\kp1,\Lm1}$} \label{line:qL_gen}
    \FOR{$\ell=\Lm1,\ldots,1$}
        \STATE{$\lambda^-_{\kp1,\ell} = T_{k\ell}^-(\mu^-_{k\ell},\Lambda_{\kp1,\lp1}^-), \quad
            \mu^-_{k\ell} =
            \bkt{\varphibf_{k\ell}^-(\pbf_{\lm1}^0,\pbf_{k,\lm1}^+,\qbf_{\kp1,\ell}^-,\wbf_\ell,\Lambda_{\kp1,\lp1}^-)}$}    \label{line:mun_gen}
        \STATE{$\Lambda_{\kp1,\ell}^- = (\Lambda_{\kp1,\lp1}^-,\lambda^-_{\kp1,\ell})$} \label{line:lamn_gen}
        \STATE{$\pbf_{\kp1,\lm1}^- =
        \fbf^-_{k\ell}(\pbf_{\lm1}^0,\pbf^+_{k,\lm1},\qbf_{\kp1,\ell}^-,\wbf_\ell,\Lambda^-_{k+1,\ell})$}
            \label{line:pn_gen}
        \STATE{$\qbf_{\kp1,\lm1}^- = \Vbf_{\lm1}\tran\pbf_{\kp1,\lm1}^-$}   \label{line:qn_gen}
    \ENDFOR

\ENDFOR\label{line:end_algo_for}
\end{algorithmic}
\end{algorithm}
%%%%% END algo/gen_ml_recursions %%%%%
%%%%%%%%%%%%%%%%%%%%%%%%%%%%%%%%%%%

Similar to our analysis of the ML-VAMP Algorithm,
we consider the following large-system limit (LSL) analysis of Gen-ML.
Specifically, we consider a sequence of runs of the recursions indexed by $N$.
For each $N$, let $N_\ell = N_\ell(N)$ be the dimension of the signals $\pbf_\ell^{\pm}$ and $\qbf_\ell^\pm$
as we assume that $\displaystyle\lim_{N \arr \infty} \tfrac{N_\ell}N = \beta_\ell\in(0,\infty)$ is a constant so that $N_\ell$ scales linearly with $N$.
We then make the following assumptions. See Appendix \ref{sec:empirical} for an overview of empirical convergence of sequences which we use in the assumptions.

\begin{assumption}\label{as:gen} For vectors in the Gen-ML Algorithm (Algorithm~\ref{algo:gen}),
we assume:
\begin{enumerate}[(a)]
\item\label{as1:a} The matrices $\Vbf_\ell$ are Haar distributed on the set of $N_\ell \times N_\ell$ orthogonal matrices and are
independent from one another and from the vectors $\q^0_0$,
$\qbf_{0\ell}^-$, perturbation vectors $\wbf_\ell$.\item\label{as1:b} The components of the initial conditions
$\qbf_{0\ell}^-$, and perturbation vectors $\wbf_\ell$ converge jointly empirically with limits,
\beq \label{eq:qwinitlim}
    \lim_{N \arr \infty} \{q_{0\ell,n}^-\} \PLeq Q_{0\ell}^-, \quad
    \lim_{N \arr \infty} \{ w_{\ell,n} \} \PLeq W_\ell,
\eeq
where $Q_{0\ell}^-$ and $W_\ell$ are random variables such that $(Q_{00}^-,\cdots,Q^-_{0,\Lm1})$
is a jointly Gaussian random vector.  Also, for $\ell=0,\ldots,\Lm1$, the random variables $W_\ell, P_{\ell-1}^0$ and $Q_{0\ell}^-$ are independent.
We also assume that the initial parameter list converges as
\beq \label{eq:Lambar01lim}
    \lim_{N \arr \infty} \Lambda_{01}^-(N) \xrightarrow{a.s.} \Lambdabar_{01}^-,
\eeq
to some list $\Lambdabar_{01}^-$.  The limit \eqref{eq:Lambar01lim} 
means that every element in the list $\lambda(N) \in \Lambda_{01}^-(N)$ converges to a limit
$\lambda(N) \arr \lambdabar$ as $N \arr \infty$ almost surely.

\item\label{as1:c} The \textit{vector update functions} $\fbf_{k\ell}^\pm(\cdot)$
and \textit{parameter update functions} $\varphibf_{k\ell}^\pm(\cdot)$ act componentwise.  For e.g.,
in the $k^{\rm th}$ forward pass, at  stage $\ell$, we assume that for each output component $n$,
\begin{align*}
    \left[ \fbf^+_{k\ell}(\pbf^0_{\lm1},\pbf^+_{k,\lm1},\qbf_{k\ell}^-,\wbf_\ell,\Lambda^+_{k\ell}) \right]_n
    = f^+_{k\ell}(p^0_{\lm1,n},p^+_{k,\lm1,n},q_{k\ell,n}^-,w_{\ell,n},\Lambda^+_{k\ell}) \\
    \left[ \varphibf^{+}_{k\ell}(\pbf^0_{\lm1},\pbf^+_{k,\lm1},\qbf_{k\ell}^-,\wbf_\ell,\Lambda^+_{k\ell}) \right]_n
    = \varphi^+_{k\ell}(p^0_{\lm1,n},p^+_{k,\lm1,n},q_{k\ell,n}^-,w_{\ell,n},\Lambda^+_{k\ell}),
\end{align*}
for some scalar-valued functions $f^+_{k\ell}(\cdot)$ and $\varphi^+_{k\ell}(\cdot)$.
Similar definitions apply in the reverse directions and for the initial vector functions $\fbf^0_\ell(\cdot)$.
We will call $f^{\pm}_{k\ell}(\cdot)$ the \emph{vector update component
functions} and $\varphi^{\pm}_{k\ell}(\cdot)$ the \emph{parameter update component functions}.
\end{enumerate}
\end{assumption}

%%%%%%%%%%%%%%%%%%%%%%%%%%%%%%%%%%%
\begin{algorithm}[t]
\setstretch{1.1}
\caption{Gen-ML State Evolution (SE)}
\begin{algorithmic}[1]  \label{algo:gen_se}

\REQUIRE{Vector update component functions $f^0_\ell(\cdot)$ and $f^\pm_{k\ell}(\cdot)$,
parameter statistic component functions $\varphi^\pm_{k\ell}(\cdot)$,
parameter update functions $T^{\pm}_{k\ell}(\cdot)$, initial parameter list limit:  $\Lambdabar_{01}^-$, initial random variables  $W_\ell$, $Q_{0\ell}^-$, $\ell=0,\ldots,\Lm1$.}
%random variables $W_\ell$, $P^0_0$, $Q_{0\ell}^-$ and
%an initial parameter list limit $\Lambdabar_{01}^-$. }

\STATE{// \texttt{Initial pass}}
%\STATE{}
%    \label{line:qinit_se_gen}
%\STATE{} \label{line:laminit_se_gen}
\STATE{$Q^0_0 = f^0_0(W_0,\Lambdabar_{01}^-), \quad P^0_0 \sim \Norm(0,\tau^0_0),
    \quad \tau^0_0 = \Exp(Q^0_0)^2$} \label{line:q0init_se_gen}
\FOR{$\ell=1,\ldots,\Lm1$}
    \STATE{$Q^0_\ell=f^0_\ell(P^0_{\lm1},W_\ell,\Lambdabar_{01}^-)$, \quad
            $P^0_\ell \sim \Norm(0,\tau^0_\ell)$, \quad
            $\tau^0_\ell = \Exp(Q^0_\ell)^2$} \label{line:p0init_se_gen}
\ENDFOR
\STATE{}

\FOR{$k=0,1,\dots$}
    \STATE{// \texttt{Forward Pass }}
    \STATE{$\lambdabar^+_{k0} = T_{k0}^+(\mubar^+_{k0},\Lambdabar_{0k}^-), \quad
        \mubar^+_{k0} = \Exp(\varphi_{k0}^+(Q_{k0}^-,W_0,\Lambdabar_{0k}^-))$}    \label{line:mup0_se_gen}
    \STATE{$\Lambdabar_{k0}^+ = (\Lambdabar_{k1}^-,\lambdabar^+_{k0})$} \label{line:lamp0_se_gen}
    \STATE{$Q_{k0}^+ = f^+_{k0}(Q_{k0}^-,W_0,\Lambdabar^+_{k0})$}  \label{line:q0_se_gen}
    \STATE{$(P^0_0,P_{k0}^+) \sim \Norm(\zero,\Kbf_{k0}^+),
        \quad \Kbf_{k0}^+ = \Cov(Q^0_0,Q_{k0}^+)$} \label{line:p0_se_gen}
    \FOR{$\ell=1,\ldots,L-1$}
        \STATE{$\lambdabar^+_{k\ell} = T_{k\ell}^+(\mubar^+_{k\ell},\Lambdabar_{k,\lm1}^+), \quad
            \mubar^+_{k\ell} = \Exp(\varphi_{k\ell}^+(P^0_{\lm1},P^+_{k,\lm1},Q_{k\ell}^-,W_\ell,\Lambdabar_{k,\lm1}^+))$}    \label{line:mup_se_gen}
        \STATE{$\Lambdabar_{k\ell}^+ = (\Lambdabar_{k,\lm1}^+,\lambdabar^+_{k\ell})$}
            \label{line:lamp_se_gen}
        \STATE{$Q_{k\ell}^+ = f^+_{k\ell}(P^0_{\lm1},P^+_{k,\lm1},Q_{k\ell}^-,W_\ell,\Lambdabar^+_{k\ell})$}
            \label{line:qp_se_gen}
        \STATE{$(P^0_\ell,P_{k\ell}^+) \sim \Norm(\zero,\Kbf_{k\ell}^+), \quad
            \Kbf_{k\ell}^+ = \Cov(Q^0_\ell,Q_{k\ell}^+) $}   \label{line:pp_se_gen}
    \ENDFOR
    \STATE{}

    \STATE{// \texttt{Backward Pass }}
    \STATE{$\lambdabar^-_{\kp1,L} = T_{kL}^-(\mubar^-_{kL},\Lambdabar_{k,\Lm1}^+), \quad
        \mubar^-_{kL} = \Exp(\varphi_{kL}^-(P^0_{\Lm1},P_{k,\Lm1}^+,W_L,\Lambdabar_{k,\Lm1}^+))$}    \label{line:muL_se_gen}
    \STATE{$\Lambdabar_{\kp1,L}^- = (\Lambdabar_{k,\Lm1}^+,\lambdabar^+_{\kp1,L})$} \label{line:lamL_se_gen}
    \STATE{$P_{\kp1,\Lm1}^- = f^-_{kL}(P^0_{\Lm1},P_{k,\Lm1}^+,W_L,\Lambdabar^-_{\kp1,L})$}  \label{line:pL_se_gen}
    \STATE{$Q_{\kp1,\Lm1}^- \sim \Norm(0,\tau_{\kp1,\Lm1}^-), \quad
        \tau_{\kp1,\Lm1}^- = \Exp(P^-_{\kp1,\Lm1})^2$} \label{line:qL_se_gen}
    \FOR{$\ell=\Lm1,\ldots,1$}
        \STATE{$\lambdabar^-_{\kp1,\ell} = T_{k\ell}^-(\mubar^-_{k\ell},\Lambdabar_{\kp1,\lp1}^-), \quad
            \mubar^-_{k\ell} =
                \Exp(\varphi_{k\ell}^-(P^0_{\lm1},P_{k,\lm1}^+,Q_{\kp1,\ell}^-,W_\ell,\Lambdabar_{\kp1,\lp1}^-))$}    \label{line:mun_se_gen}
        \STATE{$\Lambdabar_{\kp1,\ell}^- = (\Lambdabar_{\kp1,\lp1}^-,\lambdabar^-_{\kp1,\ell})$} \label{line:lamn_se_gen}
        \STATE{$P_{\kp1,\lm1}^- =
        f^-_{k\ell}(P^0_{\lm1},P^+_{k,\lm1},Q_{\kp1,\ell}^-,W_\ell,\Lambdabar^-_{k+1,\ell})$}
            \label{line:pn_se_gen}
        \STATE{$Q_{\kp1,\lm1}^- \sim \Norm(0,\tau_{\kp1,\lm1}^-), \quad
        \tau_{\kp1,\lm1}^- = \Exp(P_{\kp1,\lm1}^-)^2$}   \label{line:qn_se_gen}
    \ENDFOR

\ENDFOR
\end{algorithmic}
\end{algorithm}
%%%%%%%%%%%%%%%%%%%%%%%%%%%%%%%%%%%

Next we define a set of \textit{deterministic} constants $\{\Kbf_{k\ell}^+,\tau_{k\ell}^-,\wb{\mu}_{k\ell}^\pm,\wb\Lambda_{kl}^\pm,\tau_\ell^0\}$ and \textit{scalar} random variables $\{Q_\ell^0,P_\ell^0,Q_{k\ell}^\pm,P_{\ell}^\pm\}$ which are recursively defined through Algorithm~\ref{algo:gen_se}, which we call the \textit{Gen-ML State Evolution} (SE).
These recursions in Algorithm closely
mirror those in the Gen-ML algorithm (Algorithm~\ref{algo:gen}).  The vectors
$\qbf^\pm_{k\ell}$ and $\pbf^\pm_{k\ell}$ are replaced by random variables
$Q^\pm_{k\ell}$ and $P^\pm_{k\ell}$; the vector and parameter update functions
$\fbf^\pm_{k\ell}(\cdot)$ and $\varphibf^\pm_{k\ell}(\cdot)$ are replaced by their
component functions $f^\pm_{k\ell}(\cdot)$ and $\varphi^\pm_{k\ell}(\cdot)$;
and the parameters $\lambda_{k\ell}^\pm$ are replaced
by their limits $\lambdabar_{k\ell}^\pm$. We refer to $\{Q_\ell^0,P_\ell^0\}$ as \textit{true random variables} and $\{Q_{k\ell}^\pm,P_{kl}^\pm\}$ as iterated random variables. The signal flow graph for the true and iterated random variables in Algorithm \ref{algo:gen_se} is given in the (BOTTOM) panel of Fig. \ref{fig:mlvamp_error_system}. The iteration index $k$ for the iterated random variables $\{Q_{k\ell}^\pm,P_{kl}^\pm\}$ to simplify notation.

We also assume the following about the behaviour of component functions around the quantities defined in Algorithm \ref{algo:gen_se}.  The iteration index $k$ has been dropped for simplifying notation.

\begin{assumption} \label{as:gen2} For component functions $f,\varphi$ and parameter update functions $T$ we assume:
\begin{enumerate}[(a)]
\item\label{as2:a} $T^\pm_{k\ell}(\mu_{k\ell}^\pm,\cdot)$ are continuous at
$\mu_{k\ell}^\pm = \mubar_{k\ell}^\pm$ %where $\mubar_{k\ell}^\pm$ is the output of
% Algorithm~\ref{algo:gen_se}.
\item\label{as2:b} $f^+_{k\ell}(p_{\ell-1}^0,p^+_{k,\lm1},q_{k\ell}^-,w_\ell,\Lambda^+_{k\ell})$, $\tfrac{\partial f^+_{k\ell}}{\partial q_{k\ell}^-}(p_{\ell-1}^0,p^+_{k,\lm1},q_{k\ell}^-,w_\ell,\Lambda^+_{k\ell})$ and $\varphi^+_{k\ell}(p_{\ell-1}^0,p^+_{k,\lm1},q_{k\ell}^-,w_\ell,\Lambda^+_{k,\lm1})$ are uniformly Lipschitz continuous in $(p_{\ell-1}^0,p^+_{k,\lm1},q_{k\ell}^-,w_\ell)$ at
$\Lambda^+_{k\ell} = \Lambdabar^+_{k\ell}$, $\Lambda^+_{k,\lm1} = \Lambdabar^+_{k,\lm1}$. 
\newline Similarly, $f^-_{k+1,\ell}(p_{\ell-1}^0,p^+_{k,\lm1},q_{k+1,\ell}^-,w_\ell,\Lambda^-_{k\ell}),$ $\tfrac{\partial f_{k\ell}^-}{\partial p_{k,\ell-1}^+}(p_{\ell-1}^0,p^+_{k,\lm1},q_{k+1,\ell}^-,w_\ell,\Lambda^-_{k\ell}),$ and \newline $\varphi^-_{k\ell}(p_{\ell-1}^0,p^+_{k,\lm1},q_{k+1,\ell}^-,w_\ell,\Lambda^-_{k+1,\ell+1})$ are uniformly Lipschitz continuous in 
\newline
$(p_{\ell-1}^0,p^+_{k,\lm1},q_{k+1,\ell}^-,w_\ell)$ at $\Lambda^-_{k\ell} = \Lambdabar^-_{k\ell}$, $\Lambda^-_{k+1,\ell+1} = \Lambdabar^-_{k+1,\ell+1}$. 
\item\label{as2:c} $f^0_\ell(p^0_{\lm1},w_\ell,\Lambda^-_{01})$ are uniformly Lipschitz
continuous in $(p^0_{k,\lm1},w_\ell)$ at $\Lambda^-_{\kp1,\ell} = \Lambdabar^-_{\kp1,\ell}$.
\item\label{as2:d} Vector update functions $\fbf^\pm_{k\ell}$ are \emph{asymptotically divergence free} meaning
\beq \label{eq:fdivfree}
    \lim_{N \arr \infty} \bkt{\tfrac{\partial \fbf^+_{k\ell}}{
        \partial \qbf_{k\ell}^-}(\pbf^+_{k,\lm1},\qbf_{k\ell}^-,\wbf_\ell,\Lambdabar^+_{k\ell})} = 0,
    \quad
    \lim_{N \arr \infty} \bkt{\tfrac{\partial \fbf^-_{k\ell}}{
        \partial \pbf_{k,\lm1}^+} (\pbf^+_{k,\lm1},\qbf_{{k+1},\ell}^-,\wbf_\ell,\Lambdabar^-_{k\ell})} = 0 \\
\eeq
% \item The parameter update component functions in the forward direction
% $\varphi^+_{k\ell}(p^+_{k,\lm1},q_{k\ell}^-,w_\ell,\Lambda^+_{k,\lm1})$
% are uniformly Lipschitz continuous in $(p^+_{k,\lm1},q_{k\ell}^-,w_\ell)$ at
% $\Lambda^+_{k,\lm1} = \Lambdabar^+_{k,\lm1}$.
%  Analogous conditions apply to the reverse functions $\varphi^-_{k\ell}(\cdot)$
\end{enumerate}
\end{assumption}

\medskip
We are now ready to state the general result regarding the empirical convergence of the true and iterated vectors from Algorithm \ref{algo:gen} in terms of random variables defined in Algorithm \ref{algo:gen_se}.

\begin{theorem} \label{thm:general_convergence}  Consider the iterates of the Gen-ML recursion (Algorithm~\ref{algo:gen})
and the corresponding random variables and parameter limits
defined by the SE recursions (Algorithm~\ref{algo:gen_se}) under Assumptions~\ref{as:gen} and \ref{as:gen2}.
Then,
\begin{enumerate}[(a)]
\item For any fixed $k\geq 0$ and fixed $\ell=1,\ldots,\Lm1$,
the parameter list $\Lambda_{k\ell}^+$ converges as
\beq \label{eq:Lamplim}
    \lim_{N \arr \infty} \Lambda_{k\ell}^+ = \Lambdabar_{k\ell}^+
\eeq
almost surely.
Also, the components of
$\wbf_\ell$, $\pbf^0_{\lm1}$, $\qbf^0_{\ell}$, $\pbf_{0,\lm1}^+,\ldots,\pbf_{k,\lm1}^+$ and $\qbf_{0\ell}^\pm,\ldots,\qbf_{k\ell}^\pm$
almost surely jointly converge empirically  with limits,
\beq \label{eq:PQplim}
    \lim_{N \arr \infty} \left\{
        (p^0_{\lm1,n},p^+_{i,\lm1,n},q^-_{j\ell,n},q^0_{\ell,n},q^+_{j\ell,n}) \right\} \overset{PL(2)}=
        (P^0_{\lm1},P^+_{i,\lm1},Q^-_{j\ell},Q^0_{\ell}, Q^+_{j\ell}),
\eeq
for all $0\leq i,j\leq k$, where the variables
$P^0_{\lm1}$, $P_{i,\lm1}^+$ and $Q_{j\ell}^-$
are zero-mean jointly Gaussian random variables independent of $W_\ell$ and with covariance matrix given by
\beq \label{eq:PQpcorr}
    \Cov(P^0_{\lm1},P_{i,\lm1}^+) = \Kbf_{i,\lm1}^+, \quad \Exp(Q_{j\ell}^-)^2 = \tau_{j\ell}^-, \quad \Exp(P_{i,\lm1}^+Q_{j\ell}^-)  = 0,
    \quad \Exp(P^0_{\lm1}Q_{j\ell}^-)  = 0,
\eeq
and $Q^0_\ell$ and $Q^+_{j\ell}$ are the random variable in line~\ref{line:qp_se_gen}:
\beq \label{eq:Qpf}
    Q^0_\ell = f^0_\ell(P^0_{\lm1},W_{\ell}), \quad
    Q^+_{j\ell} =
    f^+_{\ell}(P^0_{\lm1},P^+_{j,\lm1},Q^-_{j\ell},W_\ell,\Lambdabar_{j\ell}^+).
\eeq
An identical result holds for $\ell=0$ with all the variables $\pbf_{i,\lm1}^+$ and $P_{i,\lm1}^+$ removed.

\item For any fixed $k \geq 1 $ and fixed $\ell=1,\ldots,\Lm1$,
the parameter lists $\Lambda_{k\ell}^-$ converge as
\beq \label{eq:Lamnlim}
    \lim_{N \arr \infty} \Lambda_{k\ell}^- = \Lambdabar_{k\ell}^-
\eeq
almost surely.
Also, the components of
$\wbf_\ell$, $\pbf^0_{\lm1}$, $\pbf_{0,\lm1}^\pm,\ldots,\pbf_{\km1,\lm1}^\pm$,  and $\qbf_{0\ell}^-,\ldots,\qbf_{k\ell}^-$
almost surely jointly converge empirically  with limits,
\beq \label{eq:PQnlim}
    \lim_{N \arr \infty} \left\{
        (p^0_{\lm1,n},p^+_{i,\lm1,n},q^-_{j\ell,n},p^-_{j,\ell-1,n}) \right\} \overset{PL(2)}=
        (P^0_{\lm1},P^+_{i,\lm1},Q^-_{j\ell},P_{j,\ell-1}^-),
\eeq
for all $0\leq i\leq \km1$ and $0\leq j\leq k$, where the variables
$P^0_{\lm1}$, $P_{i,\lm1}^+$ and $Q_{j\ell}^-$
are zero-mean jointly Gaussian random variables independent of $W_\ell$ and with covariance matrix given by equation \eqref{eq:PQpcorr}
% \beq \label{eq:PQncorr}
%     \Cov(P^0_{\lm1},P_{i,\lm1}^+) = \Kbf_{i,\lm1}^+,\quad
%     \quad \Exp(Q_{j\ell}^-)^2 = \tau_{j\ell}^-, \quad \Exp(P_{i,\lm1}^+Q_{j\ell}^-)  = 0,
%     \quad \Exp(P^0_{\lm1}Q_{j\ell}^-)  = 0,
% \eeq
and $P^-_{j\ell}$ is the random variable in line~\ref{line:pn_se_gen}:
\beq \label{eq:Pnf}
    P^-_{j\ell} = f^-_{\ell}(P^0_{\lm1},P^+_{j-1,\lm1},
                Q^-_{j\ell},W_\ell,\Lambdabar_{j\ell}^-).
\eeq
An identical result holds for $\ell=L$ with all the variables $\qbf_{j\ell}^-$ and $Q_{j\ell}^-$ removed.

For $k=0$, $\Lambda_{01}^-\rightarrow \Lambdabar_{01}^-$ almost surely, and $\{(w_{\ell,n},p_{\ell-1,n}^0,q_{j\ell,n}^-)\}$ empirically converge to independent random variables $(W_\ell,P_{\ell-1}^0,Q_{0\ell}^-)$.
\end{enumerate}
\end{theorem}
\begin{proof}  Appendix~\ref{app:proof_of_general_convergence} in the supplementary materials is dedicated to proving this result.
\end{proof}

%%%%%% END appendix/gen_ml_recursions %%%%%%%%%%%%%

%%%%%% START appendix/gen_ml_se_proof %%%%%%%%%%%%%
\section{Proof of Theorem~\ref{thm:general_convergence}} \label{app:proof_of_general_convergence}

\subsection{Overview of the Induction Sequence}

The proof is similar to that of \cite[Theorem 4]{rangan2019vamp},
which provides a SE analysis for VAMP on a single-layer network.
The critical challenge here is to extend that proof
to multi-layer recursions.
Many of the ideas in the two proofs are similar, so we highlight only the
key differences between the two.

Similar to the SE analysis of VAMP in \citep{rangan2019vamp},
we use an induction argument.  However, for the multi-layer proof,
we must index over both the iteration index $k$ and layer index $\ell$. To this end,
let $\mathcal{H}_{k\ell}^+$ and $\mathcal{H}_{k\ell}^-$ be the hypotheses:
\begin{itemize}
\item $\mathcal{H}_{k\ell}^+$:  The hypothesis that Theorem~\ref{thm:general_convergence}(a)
is true for a given $k$ and $\ell$, where $0\leq \ell\leq L-1$.
\item $\mathcal{H}_{k\ell}^-$:  The hypothesis that Theorem~\ref{thm:general_convergence}(b)
is true for a given $k$ and $\ell$, where $1\leq \ell\leq L$.
\end{itemize}
We prove these hypotheses by induction via a sequence of implications,
\beq \label{eq:induc}
    \{\mc H^-_{0\ell}\}_{\ell=1}^L\cdots \Rightarrow \mathcal{H}_{k1}^- \Rightarrow \mathcal{H}_{k0}^+ \Rightarrow \cdots \Rightarrow  \mathcal{H}_{k,\Lm1}^+
    \Rightarrow \mathcal{H}_{\kp1,L}^- \Rightarrow \cdots \Rightarrow \mathcal{H}_{\kp1,1}^- \Rightarrow \cdots,
\eeq
beginning with the hypotheses $\{\mathcal{H}^-_{0\ell}\}$ for all $\ell=1,\ldots,\Lm1$. 

\subsection{Base Case: Proof of \texorpdfstring{$\{\mc H_{0\ell}^-\}_{\ell=1}^L$}{H0l-}}
The base case corresponds to the Hypotheses $\{\mc H_{0\ell}^-\}_{\ell=1}^L.$ Note that Theorem \ref{thm:general_convergence}(b) states that for $k=0$, we need $\Lambda_{01}^-\rightarrow \Lambdabar_{01}^-$ almost surely, and $\{(w_{\ell,n},p_{\ell-1,n}^0,q_{j\ell,n}^-)\}$ empirically converge to independent random variables $(W_\ell,P_{\ell-1}^0,Q_{0\ell}^-)$. These follow directly from equations \eqref{eq:qwinitlim} and \eqref{eq:Lambar01lim} in Assumption 1 (a).

\subsection{Inductive Step: Proof of \texorpdfstring{$\mc H_{k,\ell+1}^+$}{H0l+}}

Fix a layer index $\ell=1,\ldots,\Lm1$ and an iteration index $k=0,1,\ldots$. We show the implication $\cdots\implies \mc H^+_{k,\ell+1}$ in \eqref{eq:induc}. All other implications can be proven similarly using symmetry arguments.
\begin{definition}[Induction hypothesis] The hypotheses {prior} to $\mathcal{H}^+_{k,\lp1}$ in the sequence \eqref{eq:induc},
but not including $\mathcal{H}^+_{k,\lp1}$, are true.  
\end{definition}
The inductive step then corresponds to the following result.
\begin{lemma}\label{lem:pqconvinduc}
Under the induction hypothesis, $\mc H_{k,\ell+1}^+$ holds
\end{lemma}

Before proving the inductive step in Lemma \ref{lem:pqconvinduc}, we prove two intermediate lemmas. Let us start by defining some notation. Define $\Pbf_{k\ell}^+ := \left[ \pbf_{0\ell}^+ \cdots \pbf_{k\ell}^+ \right] \in \R^{N_\ell \times (\kp1)},
$ be a matrix whose columns are the first $\kp1$ values of the vector $\pbf^+_{\ell}$. We define the matrices $\Pbf_{k\ell}^-$, $\Qbf_{k\ell}^+$ and $\Qbf_{k\ell}^-$ in a similar manner with values of $\pbf_{\ell}^-,\qbf_\ell^+$ and $\qbf_\ell^-$ respectively.

Note that except the initial vectors $\{\wbf_\ell,\qbf_{0\ell}^-\}_{\ell=1}^L$, all later iterates in Algorithm \ref{algo:gen} are random due to the randomness of $\V_\ell$.
Let $\Gset_{k\ell}^\pm$ denote the collection of random variables associated with the hypotheses,
$\mathcal{H}^{\pm}_{k\ell}$.  That is, for $\ell=1,\ldots,\Lm1$,
\beq \label{eq:Gsetdef}
    \Gset_{k\ell}^+ := \left\{ \wbf_{\ell},\pbf^0_{\lm1},\Pbf^+_{k,\lm1},\qbf^0_\ell,\Qbf^-_{k\ell},\Qbf_{k\ell}^+ \right\}, \quad
    \Gset_{k\ell}^- := \left\{ \wbf_{\ell},\pbf^0_{\lm1},\Pbf^+_{\km1,\lm1},\qbf^0_\ell,
        \Qbf^-_{k\ell},\Pbf^-_{k,\lm1} \right\}.
\eeq
For $\ell=0$ and $\ell=L$ we set, $
    \Gset_{k0}^+ := \left\{ \wbf_{0},\Qbf^-_{k0},\Qbf_{k0}^+ \right\}, \quad
    \Gset_{kL}^- := \left\{ \wbf_L,\pbf^0_{\Lm1},\Pbf^+_{\km1,\Lm1},\Pbf^-_{k,\Lm1} \right\}.$
    
Let $\Gsetbar_{k\ell}^+$ be the sigma algebra generated by the union of all the sets $\Gset_{k'\ell'}^\pm$
as they have appeared in the sequence \eqref{eq:induc} up to and including the final
set $\Gset_{k\ell}^+$.
Thus, the sigma algebra $\Gsetbar_{k\ell}^+$ contains all \textit{information} produced by Algorithm~\ref{algo:gen}
immediately \emph{before} line~\ref{line:pp_gen} in layer $\ell$ of iteration $k$. Note also that the random variables in Algorithm \ref{algo:gen_se} immediately before defining $P_{k,\ell}^+$ in line \ref{line:pp_se_gen} are all $\Gsetbar_{k\ell}^+$ measurable.

Observe that the matrix $\Vbf_\ell$ in Algorithm~\ref{algo:gen}
appears only during matrix-vector multiplications in lines~\ref{line:pp_gen} and \ref{line:pn_gen}.
If we define the matrices,
$
    \Abf_{k\ell} := \left[ \pbf^0_\ell, \Pbf_{\km1,\ell}^+ ~ \Pbf_{k\ell}^- \right], \quad
    \Bbf_{k\ell} := \left[ \qbf^0_\ell, \Qbf_{\km1,\ell}^+ ~ \Qbf_{k\ell}^- \right],
$
all the vectors in the set $\Gsetbar_{k\ell}^+$ will be unchanged for all
matrices $\Vbf_\ell$ satisfying the linear constraints
\beq \label{eq:ABVconk}
    \Abf_{k\ell} = \Vbf_\ell\Bbf_{k\ell}.
\eeq
Hence, the conditional distribution of $\Vbf_\ell$ given $\Gsetbar_{k\ell}^+$ is precisely
the uniform distribution on the set of orthogonal matrices satisfying
\eqref{eq:ABVconk}.  The matrices $\Abf_{k\ell}$ and $\Bbf_{k\ell}$ are of dimensions
$N_\ell \times 2k+2$.
From \cite[Lemmas 3,4]{rangan2019vamp}, this conditional distribution is given by
\beq \label{eq:Vconk}
    \left. \Vbf_\ell \right|_{\Gsetbar_{k\ell}^+} \eqd
    \Abf_{k\ell}(\Abf\tran_{k\ell}\Abf_{k\ell})^{-1}\Bbf_{k\ell}\tran + \Ubf_{\Abf_{k\ell}^\perp}\wt{\Vbf}_\ell\Ubf_{\Bbf_{k\ell}^\perp}\tran,
\eeq
where $\Ubf_{\Abf_{k\ell}^\perp}$ and $\Ubf_{\Bbf_{k\ell}^\perp}$ are $N_\ell \times (N_\ell-(2k+2))$ matrices
whose columns are an orthonormal basis for $\Range(\Abf_{k\ell})^\perp$ and $\Range(\Bbf_{k\ell})^\perp$.
The matrix $\wt{\Vbf}_\ell$ is  Haar distributed on the set of $(N_\ell-(2k+2))\times (N_\ell-(2k+2))$
orthogonal matrices and is independent of $\Gsetbar_{k\ell}^+$.

Next, similar to the proof of \cite[Thm. 4]{rangan2019vamp},
we can use \eqref{eq:Vconk} to write the conditional distribution of $\pbf_{k\ell}^+$ (from line~\ref{line:pp_gen} of Algorithm \ref{algo:gen}) given $\Gsetbar_{k\ell}^+$ as a sum of two terms
\begin{subequations}
\begin{align}
    \label{eq:ppart}
    \pbf_{k\ell}^+|_{\Gsetbar_{k\ell}^+} &= \Vbf_\ell|_{\Gsetbar_{k\ell}^+}\ \qbf_{k\ell}^+ \overset{d}= \pbf_{k\ell}^{\rm +det} + \pbf_{k\ell}^{\rm +ran},\\
    \label{eq:pdet}
    \pbf_{k\ell}^{\rm +det} &:= \Abf_{k\ell}(\Bbf\tran_{k\ell}\Bbf_{k\ell})^{-1}\Bbf_{k\ell}\tran\qbf_{k\ell}^+\\
    \label{eq:pran}
    \pbf_{k\ell}^{\rm +ran} &:= \Ubf_{\Bbf_k^\perp}\wt{\Vbf}_\ell\tran \Ubf_{\Abf_k^\perp}\tran \qbf_{k\ell}^+.
\end{align}
\end{subequations}
where we call $\pbf_{k\ell}^{\rm +det}$ the \emph{deterministic} term and
$\pbf_{k\ell}^{\rm +ran}$ the \emph{random} term. The next two lemmas characterize the limiting distributions
of the deterministic and random terms.

\begin{lemma} \label{lem:pconvdet}
Under the induction hypothesis, the components of the ``deterministic" term
$\pbf_{k\ell}^{+\rm det}$ along with the components
of the vectors in $\Gsetbar_{k\ell}^+$  converge empirically.
In addition, there exists constants $\beta_{0\ell}^+,\ldots,\beta^+_{\km1,\ell}$ such that
\beq \label{eq:pconvdet}
    \lim_{N \arr \infty} \{ p_{k\ell,n}^{\rm +det} \} \PLeq P_{k\ell}^{\rm +det}
    := \beta^0_\ell P^0_\ell +  \sum_{i=0}^{\km1}    \beta_{i\ell} P_{i\ell}^+,
\eeq
where $P_{k\ell}^{+\rm det}$ is the limiting random variable for the components of $\pbf_{k\ell}^{\rm det}$.
\end{lemma}
\begin{proof}
The proof is similar that of \cite[Lem. 6]{rangan2019vamp}, but we go over the details
as there are some important differences in the multi-layer case.
Define
$ \label{eq:PQaug}
    \wt{\Pbf}_{\km1,\ell}^+ = \left[ \pbf^0_\ell, ~ \Pbf_{\km1,\ell}^+ \right],
    \wt{\Qbf}_{\km1,\ell}^+ = \left[ \qbf^0_\ell, ~ \Qbf_{\km1,\ell}^+ \right],
$
which are the matrices in $\Real^{N_\ell\times (k+1)}$.
We can then write $\Abf_{k\ell}$ and $\Bbf_{k\ell}$ from \eqref{eq:ABVconk} as
\beq \label{eq:ABdef2}
    \Abf_{k\ell} := \left[ \wt{\Pbf}_{\km1,\ell}^+ ~ \Pbf_{k\ell}^- \right], \quad
    \Bbf_{k\ell} := \left[ \wt{\Qbf}_{\km1,\ell}^+ ~ \Qbf_{k\ell}^- \right],
\eeq
We first evaluate the asymptotic values of various terms in \eqref{eq:pdet}.
By definition of $\Bbf_{k\ell}$ in \eqref{eq:ABdef2},
\[
    \Bbf\tran_{k\ell}\Bbf_{k\ell} = \begin{bmatrix}
        (\wt{\Qbf}_{\km1,\ell}^+)\tran\wt{\Qbf}_{\km1,\ell}^+ & (\wt{\Qbf}_{\km1,\ell}^+)\tran\Qbf_{k\ell}^- \\
        (\Qbf_{k\ell}^-)\tran\wt{\Qbf}_{\km1,\ell}^+ & (\Qbf_{k\ell}^-)\tran\Qbf_{k\ell}^-
        \end{bmatrix}
\]
We can then evaluate the asymptotic values of these
terms as follows:  For $0\leq i,j\leq k-1$ the asymptotic value of the
$(i+2,j+2)^{\rm nd}$ entry of the matrix $(\wt{\Qbf}_{\km1,\ell}^+)\tran\wt{\Qbf}_{\km1,\ell}^+$ is given by
\begin{align*}
    \MoveEqLeft \lim_{N \arr \infty} \tfrac{1}{N_\ell} \left[ (\wt{\Qbf}_{\km1,\ell}^+)\tran\wt{\Qbf}_{\km1,\ell}^+ \right]_{i+2,j+2}
        \stackrel{(a)}{=} \lim_{N \arr \infty}
        \frac{1}{N_\ell} (\qbf_{i\ell}^+)\tran\qbf_{j\ell}^+ 
        &= \lim_{N \arr \infty} \tfrac{1}{N_\ell} \sum_{n=1}^{N_\ell} q_{i\ell,n}^+q_{j\ell,n}^+
        \stackrel{(b)}{=} \Exp\left[ Q_{i\ell}^+Q_{j\ell}^+ \right]
\end{align*}
where (a) follows since the $(i+2)^{\rm th}$ column of $\wt{\Qbf}_{\km1,\ell}^+$
is $\qbf_{i\ell}^+$, and (b) follows due to the empirical convergence assumption in \eqref{eq:PQplim}.
Also, since the first column of $\wt{\Qbf}_{\km1,\ell}^+$ is $\qbf^0_\ell$,
we obtain that
\[
    \lim_{N_\ell \arr \infty}  \tfrac{1}{N_\ell}
        (\wt{\Qbf}_{k-1,\ell}^+)\tran\wt{\Qbf}_{k-1,\ell}^+ = \Rbf^+_{k-1,\ell}\qquad{\rm and}\qquad \lim_{N_\ell \arr \infty}  \tfrac{1}{N_\ell}  (\Qbf_{k\ell}^-)\tran\Qbf_{k\ell}^- = \Rbf^-_{k\ell},
\]
where $\Rbf^+_{k-1,\ell}$ is the covariance matrix of
$(Q^0_\ell,Q_{0\ell}^+,\ldots,Q_{k-1,\ell}^+)$,
and $\Rbf^-_{k\ell}$ is the covariance matrix of the vector
$(Q_{0\ell}^-,\ldots,Q_{k\ell}^-)$.
For the matrix $(\wt\Qbf_{\km1,\ell}^+)\tran\Qbf_{k\ell}^-$,
first observe that the limit of the divergence free condition \eqref{eq:fdivfree} implies
\beq \label{eq:fpdivfree}
    \Exp\left[ \frac{\partial f_{i\ell}^+(P_{i,\lm1}^+,Q_{i\ell}^-,W_\ell,\Lambdabar_{i\ell})}{\partial q_{i\ell}^-} \right]
    = \lim_{N_\ell \arr \infty}  \bkt{\frac{\partial \fbf^+_{i\ell}(\pbf^+_{i,\lm1},\qbf_{i\ell}^-,\wbf_\ell,\Lambdabar^+_{i\ell})}{
        \partial \qbf_{i\ell}^-} }  = 0,
\eeq
for any $i$.  Also, by the induction hypothesis $\mathcal{H}_{k\ell}^+$,
\beq \label{eq:pqxcorrpf}
    \Exp(P_{i,\lm1}^+Q_{j\ell}^-)  = 0, \quad
    \Exp(P_{\lm1}^0 Q_{j\ell}^-) = 0,
\eeq
for all $0\leq i,j \leq k$.
Therefore using \eqref{eq:Qpf}, the cross-terms $\Exp(Q_{i\ell}^+Q_{j\ell}^-)$ are given by
\begin{align}
    &\Exp(f_{i\ell}^+(P^0_{\lm1},P_{i,\lm1}^+,Q_{i\ell}^-,W_\ell,\Lambdabar_{i\ell})Q_{j\ell}^-) 
    \stackrel{(a)}{=}  \Exp\left[ \tfrac{\partial f_{i\ell}^+(P^0_{\lm1},P_{i,\lm1}^+,Q_{i\ell}^-,W_\ell,\Lambdabar^+_{i\ell})}
        {\partial P_{\lm1}^0} \right]
        \Exp(P_{\lm1}^0Q_{j\ell}^-)    \nonumber \\
     &+\Exp\left[ \tfrac{\partial f_{i\ell}^+(P^0_{\lm1},P_{i,\lm1}^+,Q_{i\ell}^-,W_\ell,\Lambdabar^+_{i\ell})}{\partial P_{i,\lm1}^+} \right]
        \Exp(P_{i,\lm1}^+Q_{j\ell}^-)
     + \Exp\left[ \tfrac{\partial f_{i\ell}^+(P^0_{\lm1},P_{i,\lm1}^+,Q_{i\ell}^-,W_\ell,\Lambdabar^+_{i\ell})}
        {\partial Q_{i\ell}^-} \right]
        \Exp(Q_{i\ell}^-Q_{j\ell}^-)
     \stackrel{(b)}{=} 0, \label{eq:Qijstein}
\end{align}
(a) follows from Stein's Lemma; and (b) follows from \eqref{eq:fpdivfree}, and \eqref{eq:pqxcorrpf}.
Consequently,
\begin{align} \label{eq:BBlim_and_Bqlim}
    \lim_{N_\ell \arr \infty} \tfrac{1}{N_\ell} \Bbf\tran_{k\ell}\Bbf_{k\ell} &= \begin{bmatrix}
        \Rbf_{\km1,\ell}^+ & \zero \\
        \zero & \Rbf_{k\ell}^-
        \end{bmatrix} ,\quad{\rm and}\quad
     \lim_{N_\ell \arr \infty} \tfrac{1}{N_\ell} \Bbf_{k\ell}\tran\qbf_{k\ell}^+= \
    \begin{bmatrix} \bbf^+_{k\ell} \\ \zero \end{bmatrix} ,
\end{align}
where $\bbf^+_{k\ell} := \left[\Exp(Q_{0\ell}^+Q_{k\ell}^+), ~ \Exp(Q_{1\ell}^+Q_{k\ell}^+),
        ~\cdots, \Exp(Q_{\km1,\ell}^+Q_{k\ell}^+) \right]\tran,$ is the vector of correlations. We again have $\zero$ in the second term because $\Exp[Q_{i\ell}^+Q_{j\ell^-}]=0$ for all $0\leq i,j\leq k$. Hence we have
\beq \label{eq:Bqmult}
    \lim_{N_\ell \arr \infty} (\Bbf\tran_{k\ell}\Bbf_{k\ell})^{-1}\Bbf_{k\ell}\tran\qbf_{k\ell}^+ =
    \begin{bmatrix}  \betabf_{k\ell}^+ \\ \mathbf{0} \end{bmatrix}, \quad \betabf_{k\ell}^+ := \begin{bmatrix} \Rbf^+_{\km1,\ell} \end{bmatrix}^{-1}\bbf^+_{k\ell}.
\eeq
Therefore, $\pbf_{k\ell}^{+\rm det}$ equals
\begin{align}
    \Abf_{k\ell}(\Bbf\tran_{k\ell}\Bbf_{k\ell})^{-1}\Bbf_{k\ell}\tran\qbf_{k\ell}^+
    = \left[ \wt{\Pbf}_{\km1,\ell}^+ ~ \Pbf_{k,\ell}^- \right]
%   \left[ 
   \begin{bmatrix}%\begin{array}{c} 
   \mathbf{\beta}_{k\ell}^+  \\ \zero %\end{array} \right]
   \end{bmatrix}
    + O\left(\tfrac{1}{N_\ell}\right) \nonumber 
    = \beta^0_\ell \pbf^0_\ell +
    \sum_{i=0}^{\km1} \beta_{i\ell}^+\pbf_{i\ell}^+ + O\left(\tfrac{1}{N_\ell}\right),
\end{align}
where $\beta^0_\ell$ and $\beta_{i\ell}^+$ are the components of $\betabf_{k\ell}^+$ and
the term $O(\tfrac1{N_\ell})$ means a vector sequence, $\xibf(N) \in \R^{N_{\ell}}$ such that $\lim_{N \arr\infty} \tfrac{1}{N} \|\xibf(N)\|^2 = 0.$
A continuity argument then shows the empirical convergence \eqref{eq:pconvdet}.
\end{proof}

\begin{lemma} \label{lem:pconvran}
Under the induction hypothesis, the components of
the ``random" term $\pbf_{k\ell}^{+\rm ran}$ along with the components
of the vectors in $\Gsetbar_{k\ell}^+$ almost surely converge empirically.
The components of $\pbf_{k\ell}^{+\rm ran}$ converge as
\beq \label{eq:pconvran}
     \lim_{N \arr \infty} \{ p_{k\ell,n}^{+\rm ran} \} \PLeq U_{k\ell},
\eeq
where $U_{k\ell}$ is a zero mean Gaussian random variable
independent of the limiting random variables corresponding to the variables
in $\Gsetbar_{k\ell}^+$.
\end{lemma}
\begin{proof}
The proof is very similar to that of \cite[Lemmas 7,8]{rangan2019vamp}.
\end{proof}

We are now ready to prove Lemma \ref{lem:pqconvinduc}.

\old{
\begin{lemma} \label{lem:pqconvinduc}
Under the induction hypothesis, the parameter list $\Lambda_{k,\lp1}^+$ almost surely converges as
\beq \label{eq:Lampliminduc}
    \lim_{N_{\lp1} \arr \infty} \Lambda_{k,\lp1}^+ = \Lambdabar_{k,\lp1}^+,
\eeq
where $\Lambdabar_{k,\lp1}$ is the parameter list generated from the SE recursion, Algorithm~\ref{algo:gen_se}.
Also, the components of
$\wbf_{\lp1}$, $\pbf^0_{\ell}$, $\qbf^0_{\lp1}$, $\pbf_{0,\ell}^+,\ldots,\pbf_{k,\ell}^+$ and $\qbf_{0,\lp1}^\pm,\ldots,\qbf_{k,\lp1}^\pm$
almost surely empirically converge jointly with limits,
\beq \label{eq:PQpliminduc}
    \lim_{N \arr \infty} \left\{
        (p^0_{\ell,n},p^+_{i\ell,n},q^0_{\lp1,n},q^-_{j,\lp1,n},q^+_{j,\lp1,n}) \right\} =
        (P^0_{\ell},P^+_{i\ell},Q^0_{\lp1},Q^-_{j,\lp1}, Q^+_{j,\lp1}),
\eeq
for all $i,j=0,\ldots,\kp1$, where the variables
\beq \label{eq:pqvecinduc}
    (P^0_\ell,P_{0\ell}^+,\ldots,P_{k,\ell}^+,Q_{0,\lp1}^-,\ldots,Q_{k,\lp1}^-),
\eeq
are zero-mean jointly Gaussian random variables independent of $W_\ell$ with
\beq \label{eq:PQpcorrinduc}
    \Cov(P^0_{\ell},P_{i,\ell}^+) = \Kbf_{i\ell}^+, \quad \Exp(Q_{j,\lp1}^-)^2 = \tau_{j,\lp1}^-, \quad \Exp(P_{i,\ell}^+Q_{j,\lp1}^-)  = 0,
    \quad \Exp(P^0_{\ell}Q_{j,\lp1}^-)  = 0,
\eeq
and $Q^0_{\lp1}$ and $Q^+_{j,\lp1}$ are the random variables in line~\ref{line:qp_se_gen}:
\beq \label{eq:Qpfinduc}
    Q^0_{\lp1} = f^0_{\lp1}(P^0_{\ell},W_{\lp1}), \quad
    Q^+_{j,\lp1} =
    f^+_{j,\lp1}(P^0_{\ell},P^+_{i\ell},Q^-_{j,\lp1},W_{\lp1},\Lambdabar_{k,\lp1}^+).
\eeq
\end{lemma}
}

\begin{proof}[Proof of Lemma \ref{lem:pqconvinduc}]
Using the partition \eqref{eq:ppart} and Lemmas~\ref{lem:pconvdet} and \ref{lem:pconvran},
we see that the components of the
vector sequences in $\Gsetbar_{k\ell}^+$ along with $\pbf^+_{k\ell}$
almost surely converge jointly empirically, where the components of $\pbf^+_{k\ell}$
have the limit
\beq \label{eq:pklim}
    \lim_{N_\ell \arr \infty} \left\{ p^+_{k\ell,n} \right\}
    = \lim_{N_\ell \arr \infty} \left\{ p^{\rm det}_{k\ell,n} + p^{\rm ran}_{k\ell,n} \right\}
    \PLeq  \beta^0_\ell P^0_\ell + \sum_{i=0}^{\km1}  \beta_{i\ell}^+ P_{i\ell}^+ + U_{k\ell} =: P_{k\ell}^+.
\eeq
{Note that the above PL(2) convergence can be shown using the same arguments involved in showing that if $X_N|\mc F\overset{d}{\implies} X|\mc F,$ and $Y_N|\mc F\overset{d}{\implies} c,$ then $(X_N,Y_N)|\mc F\overset{d}{\implies} (X,c)|\mc F$ for some constant $c$ and sigma-algebra $\mc F$.}

We first establish the Gaussianity of $P_{k\ell}^+$. Observe that by the induction hypothesis, $\mathcal{H}_{k,\lp1}^-$ holds whereby $(P_\ell^0,P_{0\ell}^+,\ldots,P_{\km1,\ell}^+,Q_{0,\lp1}^-,\ldots,Q_{k,\lp1}^-),$
is jointly Gaussian. Since $U_k$ is Gaussian and independent of $(P_\ell^0,P_{0\ell}^+,\ldots,P_{k-1,\ell}^+,Q_{0,\lp1}^-,\ldots,Q_{k,\lp1}^-),$ we can conclude from  \eqref{eq:pklim} that
\beq
(P_\ell^0,P_{0\ell}^+,\ldots,P_{\km1,\ell}^+,P_{k\ell}^+,Q_{0,\lp1}^-,\ldots,Q_{k,\lp1}^-){\rm \ is\ jointly\ Gaussian}.
\eeq
We now need to prove the correlations of this jointly Gaussian random vector as claimed by $\mc H_{k,\ell+1}^+$.  Since $\mathcal{H}_{k,\lp1}^-$ is true, we know
that \eqref{eq:PQpcorr} is true for
all $i=0,\ldots,\km1$ and $j=0,\ldots,k$ and $\ell=\ell+1$.  Hence,
we need only to prove the additional identity for $i=k$,
namely the equations:
$% \label{eq:PQcorrpf2}
    \Cov(P^0_\ell,P_{k\ell}^+)^2 = \Kbf_{k\ell}^+
    $ and $
    \Exp(P_{k\ell}^+Q_{j,\lp1}^-) = 0.
$
First observe that
\[
    \Exp(P_{k\ell}^+)^2  \stackrel{(a)}{=} \lim_{N_\ell \arr \infty} \frac{1}{N_\ell}
        \|\pbf_{k\ell}^+\|^2
          \stackrel{(b)}{=} \lim_{N_\ell \arr \infty} \frac{1}{N_\ell}
        \|\qbf_{k\ell}^+\|^2 \stackrel{(c)}{=}  \Exp\left( Q_{k\ell}^+ \right)^2
\]
where (a) follows from the fact that the components of $\pbf^+_{k\ell}$ converge empirically
to $P_{k\ell}^+$;
(b) follows from line \ref{line:pp_gen} in Algorithm~\ref{algo:gen} and the fact that $\Vbf_\ell$ is orthogonal;
and
(c) follows from the fact that the components of $\qbf^+_{k\ell}$ converge empirically
to $Q_{k\ell}^+$ from hypothesis $\mc H_{k,\ell}^+$.  Since $\pbf^0_\ell = \Vbf_\ell \qbf^0$, we similarly obtain that
$
    \Exp(P^0_\ell P_{k\ell}^+) = \Exp(Q^0_\ell Q_{k\ell}^+), \quad
    \Exp(P^0_\ell)^2 = \Exp(Q^0_\ell)^2,
$
from which we conclude
\beq \label{eq:PQcorr3}
    \Cov(P^0_\ell, P_{k\ell}^+) = \Cov(Q^0_\ell, Q_{k\ell}^+) =: \Kbf^+_{k\ell},
\eeq
where the last step follows from the definition of $\Kbf^+_{k\ell}$ in line~\ref{line:pp_se_gen} of Algorithm \ref{algo:gen_se}.
Finally, we observe that for $0\leq j\leq k$
\beq \label{eq:PQcorr4}
    \Exp(P_{k\ell}^+Q_{j,\lp1}^-) \stackrel{(a)}{=}
     \beta^0_\ell\Exp(P_{\ell}^0Q_{j,\lp1}^-) + \sum_{i=0}^{\km1} \beta_{i\ell}^+ \Exp(P_{i\ell}^+Q_{j,\lp1}^-)
        + \Exp(U_{k\ell}Q_{j,\lp1}^-) \stackrel{(a)}{=} 0,
\eeq
where (a) follows from \eqref{eq:pklim} and, in (b), we used the fact that
$\Exp(P_{\ell}^0Q_{j,\lp1}^-) = 0$ and
$\Exp(P_{i\ell}^+Q_{j,\lp1}^-) = 0$ since \eqref{eq:PQpcorr} is true for $i\leq \km1$ corresponding to $\mc H_{k,\ell+1}^-$ and
$\Exp(U_{k\ell}Q_{j,\lp1}^-) = 0$ since $U_{k\ell}$ is independent of $\Gsetbar_{k\ell}^+$, and $Q_{j,\lp1}^-$ is $\Gsetbar_{k\ell}^+$ measurable.
Thus, with \eqref{eq:PQcorr3} and \eqref{eq:PQcorr4}, we have proven all the correlations in
\eqref{eq:PQpcorr} corresponding to $\mc H_{k,\ell+1}^+$.

Next, we prove the convergence of the parameter lists $\Lambda_{k,\ell+1}^+$ to $\Lambdabar_{k,\ell+1}^+$.  Since $\Lambda^+_{k\ell} \arr \Lambdabar_{k\ell}^+$ due to hypothesis $\mc H_{k\ell}^+$,
and $\varphi_{k,\lp1}^+(\cdot)$ is uniformly Lipschitz continuous,
we have that $\lim_{N \arr \infty} \mu^+_{k,\lp1}$ from line~\ref{line:mup_gen} in Algorithm~\ref{algo:gen}
converges almost surely as
\begin{align}
    \MoveEqLeft  \lim_{N \arr \infty}
        \bkt{\varphibf_{k,\lp1}^+(\pbf^0_\ell,\pbf^+_{k\ell},\qbf_{k,\lp1}^-,\wbf_{\lp1},\Lambdabar_{k\ell}^+)}
    =    \Exp\left[ \varphi_{k,\lp1}^+(P^0_\ell,P^+_{k\ell},Q_{k,\lp1}^-,W_{\lp1},\Lambdabar_{k\ell}^+)
        \right]  = \mubar^+_{k,\lp1},
\end{align}
where $\mubar^+_{k,\lp1}$ is the value in line~\ref{line:mup_se_gen} in Algorithm~\ref{algo:gen_se}.
Since $T^+_{k,\lp1}(\cdot)$ is continuous, we have that $\lambda_{k,\lp1}^+$ in
line~\ref{line:lamp_gen} in Algorithm~\ref{algo:gen} converges as
$
    \lim_{N \arr \infty} \lambda_{k,\lp1}^+ 
    % =\lim_{N \arr \infty} T_{k,\lp1}^+(\mu_{k,\lp1}^+)
    = T_{k,\lp1}^+(\mubar_{k,\lp1}^+,\Lambdabar_{k\ell}^+) =: \lambdabar_{k,\lp1}^+,
$
from
line~\ref{line:lamp_se_gen} in Algorithm~\ref{algo:gen_se}. Therefore, we have the limit
\beq\label{eq:Lampliminduc}
    \lim_{N \arr \infty} \Lambda_{k,\lp1}^+ =
    \lim_{N \arr \infty} (\Lambda_{k,\ell}^+,\lambda_{k,\lp1}^+)
    = (\Lambdabar_{k,\ell}^+,\lambdabar_{k,\lp1}^+) = \Lambdabar_{k,\lp1}^+,
\eeq
which proves the convergence of the parameter lists stated in $\mc H_{k,\ell+1}^+$.
Finally, using \eqref{eq:Lampliminduc}, the empirical convergence of the vector
sequences $\pbf^0_\ell$, $\pbf_{k\ell}^+$ and $\qbf_{k,\lp1}^-$ and the uniform Lipschitz continuity of
the update function $f_{k,\lp1}^+(\cdot)$ we obtain that $\lim_{N \arr \infty} \left\{ q_{k,\lp1,n}^+ \right\}$ equals
\begin{align*}
    \left\{ f_{k,\lp1}^+(p^0_{\ell,n},p_{k\ell,n}^-, q_{k,\lp1,n}^-,w_{\lp1,n},\Lambda_{k,\lp1}^+) \right\}
    = f_{k,\lp1}^+(P^0_\ell,P_{k\ell}^-, Q_{k,\lp1}^-,W_{\lp1},\Lambdabar_{k,\lp1}^+) =: Q^+_{k,\lp1},
\end{align*}
which proves the claim \eqref{eq:Qpf} for $\mc H_{k,\ell+1}^+$.  This completes the proof.
\end{proof}

%%%%%% END appendix/gen_ml_se_proof %%%%%%%%%%%%%

%%%%%% START appendix/proof_main_result %%%%%%%%%%%%%

\section{Proofs of Main Results:  Theorems~\ref{thm:main_result} and \ref{thm:MMSE_SE}}
\label{app:proof_of_main_result}

Recall that the main result in Theorem \ref{thm:main_result} claims the empirical convergence of PL(2) statistics of iterates of the ML-VAMP algorithm \ref{algo:ml-vamp} to the expectations corresponding statistics of random variables given in Algorithm \ref{algo:mlvamp_se}. 
We prove this result by applying the general convergence result stated in Theorem \ref{thm:general_convergence} which shows that under Assumptions \ref{as:gen} and \ref{as:gen2}, the PL(2) statistics of iterates of Algorithm \ref{algo:gen} empirically converge to expectations of corresponding statistics of appropriately defined scalar random variables defined in Algorithm \ref{algo:gen_se}.

The proof of Theorem \ref{thm:main_result} proceeds in two steps. First, we show that the ML-VAMP iterations are a special case of the iterations of Algorithm \ref{algo:gen}, and similarly Algorithm \ref{algo:mlvamp_se} is a special case of \ref{algo:mlvamp_se}, for specific choices of vector update functions, parameter statistic functions and parameter update functions, and their componentwise counterparts. The second step is to show that all assumptions required in Theorem \ref{thm:general_convergence} are satisfied, and hence the conclusions of Theorem \ref{thm:general_convergence} hold.

\begin{subsection}{Proof of Theorem \ref{thm:main_result}}
We start by showing that the ML-VAMP iterations from Algorithm \ref{algo:ml-vamp} are a special case of the Gen-ML recursions from Algorithm \ref{algo:gen}.

%%%%% START sections/transformed_error_system %%%%%
Consider the singular value decompositions $\W_\ell = \V_\ell\diag(\s_\ell)\V_{\ell-1}$ from equation \eqref{eq:SVD}. Then the true signals $\z_\ell^0$ in equation \eqref{eq:nntrue} and the iterates $\{\r_\ell^\pm,\zhat_\ell^\pm\}$ of Algorithm \ref{algo:ml-vamp} can then be expressed via the \textit{transformed} true signals defined below,
\begin{align}  \label{eq:pq0}
\begin{split}
    \qbf^0_\ell &:= \zbf^0_\ell, \quad    \pbf^0_\ell :=  \Vbf_\ell \zbf^0_\ell \quad \ell=0,2,\ldots,L \\
    \qbf^0_\ell &:= \Vbf_\ell\tran\zbf^0_\ell, \quad \pbf^0_\ell := \zbf^0_\ell \quad \ell=1,3,\ldots,\Lm1.
\end{split}
\end{align}
These signals can be see in the (TOP) of Fig. \ref{fig:mlvamp_error_system}.
Next, for $\ell=0,2,\ldots,L-2$, define:
\begin{subequations} \label{eq:pqdef}
\begin{align}
    \qbfhat^{\pm}_{k\ell} := \zbfhat^{\pm}_{k\ell}, \quad
    &\qbf^{\pm}_{k\ell} := \rbf_{k\ell}^\pm - \zbf^0_\ell, 
    &\qbfhat^{\pm}_{k,\lp1} := \Vbf_{\lp1}\tran\zbfhat^{\pm}_{k,\lp1}, \quad
    &\qbf^{\pm}_{k,\lp1} := \Vbf_{\lp1}\tran(\rbf_{k,\lp1}^{\pm} - \zbf^0_{\lp1}) \label{eq:qdef} \\
    \pbfhat^{\pm}_{k\ell} := \Vbf_\ell\zbfhat^{\pm}_{k\ell}, \quad
    &\pbf^{\pm}_{k\ell} := \Vbf_\ell(\rbf_{k\ell}^\pm - \zbf^0_\ell),
    &\pbfhat^{\pm}_{k,\lp1} := \zbfhat^{\pm}_{k,\lp1}, \quad
    &\pbf^{\pm}_{k,\lp1} := \rbf_{k,\lp1}^{\pm} - \zbf^0_{\lp1},  \label{eq:pdef}
\end{align}
\end{subequations}
The vectors $\qbfhat^{\pm}_{k\ell}$ and $\pbfhat^{\pm}_{k\ell}$ represent the estimates
of $\qbf^0_\ell$ and $\pbf^0_\ell$ defined in \eqref{eq:pq0}. These are outputs of the estimators $\gbf_\ell^\pm$ and $\Gbf_\ell^\pm$. Similarly, the vectors $\qbf^{\pm}_{k\ell}$ and $\pbf^{\pm}_{k\ell}$
are the differences $\rbf_{k\ell}^{\pm}-\zbf^0_\ell$ or their transforms.  These
represent errors on the \emph{inputs} $\rbf_{k\ell}^\pm$ to the estimators $\gbf^{\pm}_\ell(\cdot)$ (even $\ell$) and $\Gbf_\ell^\pm$ (odd $\ell$). These vectors can be seen in the (MIDDLE) panel of Fig. \ref{fig:mlvamp_error_system}

% It is important to note that during the computation, only the quantities $\{\qbfhat_{k\ell}^\pm,\qbf_{kl}^\pm+\qbf_\ell^0,\pbfhat_{k\ell}^\pm,\pbf_{kl}^\pm+\pbf_\ell^0\}$ or equivalently $\{\zhat_{k\ell}^\pm,\r_{k\ell}^\pm\}$ are observable during Algorithm \ref{algo:ml-vamp}. Nonetheless, we can still comment on bulk error quantities such as in the LHS of \eqref{eq:smooth_empirical_convergence}.

%%%%% END sections/transformed_error_system %%%%%

\begin{lemma}[ML-VAMP as a special case of Gen-ML]\label{lem:special_case}
Consider Algorithms \ref{algo:gen} and \ref{algo:gen_se} with 
\begin{enumerate}
    \item Initial functions $\fbf_\ell^0$ and vector update functions $\fbf_\ell^\pm$ given by componentwise extensions of $f_\ell^0$ and $f_\ell^\pm$ respectively from equation \eqref{eq:se_update_functions}.
    Parameter statistic functions $\bm{\varphi}_\ell^+$ and $\bm{\varphi}_\ell^-$ be given by componentwise extensions of $\tfrac{\partial f_\ell^+}{\partial q_\ell^-}$ and $\tfrac{\partial \fbf_\ell^+}{\partial p_{\ell-1}^+}$ respectively.
    Parameter updates $T_{k\ell}^\pm(\cdot)$ applied so that $\mu_{k\ell}^\pm=\alpha_{k\ell}^\pm$ and $\Lambda_{k\ell}^\pm =\theta_{k\ell}^\pm$, with $\theta_{k\ell}^\pm$ given in equation \eqref{eq:thetagam}.
    \item Perturbation vectors $\w_\ell$ given by
    $\w_0 = \z_0^0$, $\w_{2\ell} = \bm{\xi}_{2\ell}$ and $\w_{2\ell-1}=(\s_{2\ell-1},\wb\bbf_{2\ell-1},\wb{\bm{\xi}}_{2\ell-1})$ for $\ell=1,2,\ldots \frac{L}{2}.$
    Perturbation random variables $W_\ell$ given by \eqref{eq:perturbations}.
\end{enumerate}
Then we have that
\begin{enumerate}
    \item 
Lines \ref{line:q00init_gen}-\ref{line:end_initial_for} of Algorithm \ref{algo:gen} are equivalent to equation \eqref{eq:nntrue} with definitions of $\p_\ell^0,\q_\ell^0$ given in equation \eqref{eq:pq0}. 
Lines \ref{line:start_algo_for}-\ref{line:end_algo_for} of Algorithm \ref{algo:gen} are equivalent to the ML-VAMP iterations in Algorithm \ref{algo:ml-vamp} with definitions of $\p_\ell^\pm,\phat_\ell^\pm,\q_\ell^\pm,\qhat_\ell^\pm,$ given in equation \eqref{eq:pqdef}.
\item Algorithm \ref{algo:gen_se} is equivalent to Algorithm \ref{algo:mlvamp_se}.
\end{enumerate}
\end{lemma}

\begin{lemma}\label{lem:Ass1}
Assumptions \ref{as:gen} and \ref{as:gen2} are satisfied by the conditions in Theorem \ref{thm:main_result}.
\end{lemma}

The lemmas follow from the direct substitution of the quantities keeping in mind \eqref{eq:SVD}. As a consequence of the lemmas, we can apply the result of Theorem \ref{thm:general_convergence} under the conditions given in Theorem \ref{thm:main_result}. 
The convergence of $(\alpha_{k\ell}^\pm,\gamma_{k\ell}^\pm,\eta_{k\ell}^\pm)$ follows from the convergence of $\Lambda_{k\ell}^\pm$.

Theorem \ref{thm:general_convergence} leads to the conclusion that the following triplets are asymptotically normal
\begin{align*}
(\z_{\ell-1}^0,\r_{\ell-1}^+-\z_{\ell-1}^0,\r_\ell^--\z_{\ell}^0)\equiv(\p_{\ell-1}^0,\p_{\ell-1}^+,\q_\ell^-),\qquad\forall\ \ell\ {\rm even},\\
\left(\V_{\ell-1}\z_{\ell-1}^0,\V_{\ell-1}(\r_{\ell-1}^+-\z_{\ell-1}^0),\V_{\ell}\T(\r_\ell^--\z_{\ell}^0)\right)\equiv(\p_{\ell-1}^0,\p_{\ell-1}^+,\q_\ell^-),\qquad\forall\ \ell\ {\rm odd}.
\end{align*}
The results in Theorem \ref{thm:main_result} follows from the argument definition of PL(2) convergence defined in Appendix \ref{app:empirical_convergence}
\end{subsection}

\begin{subsection}{Proof of Theorem \ref{thm:MMSE_SE}}

Recall the update equations for $(\wb\alpha_{k\ell}^\pm,\wb\gamma_{k\ell}^\pm,\wb\eta_{k\ell}^\pm)$ analogous to \eqref{eq:gamupdate}. Fix the iteration index $k$ and let $\ell$ be even. We showed earlier after stating Theorem \ref{thm:main_result} that \begin{align*}
    \tfrac{1}{N_\ell}\norm{\zhat_{k\ell}^+-\z_\ell^0}\xrightarrow{a.s.}\Exp\left(g_\ell^+(\mathsf{C+A,B+A},\wb\gamma_{k\ell}^-,\wb\gamma_{k,\ell-1}^+)-\phi_\ell(\mathsf{A},\Xi_\ell)\right)^2=:\mc E^+_\ell(\wb\gamma_{k\ell}^-,\wb\gamma_{k,\ell-1}^+)
\end{align*}
We also know that $\eta_{k\ell}^+\xrightarrow{a.s.}\wb\eta_{k\ell}^+=\frac{\wb\gamma_{k\ell}^-}{\wb\alpha_{k\ell}^+}.$ We need to show that the two limits coincide or equivalently $\frac{\wb\alpha_{k\ell}^+}{\wb\gamma_{k\ell}^-}=\mc E^+_\ell(\wb\gamma_{k\ell}^-,\wb\gamma_{k,\ell-1}^+)$.
In case of MMSE estimation, where $g_{\ell,\mathsf{mmse}}^\pm$ from \eqref{eq:mmse_estimator} is applied, we can simplify $\wb\alpha_{k\ell}^\pm$. From line \ref{line:alphap} of Algorithm \ref{algo:mlvamp_se}, then we have
\begin{align*}
    \MoveEqLeft\wb\alpha_{k\ell}^+ = \Exp\frac{\partial h_\ell^+}{\partial Q_{\ell}^-}(P^0_{\lm1},P^+_{k,\lm1},Q_{k\ell}^-,W_\ell,\wb\theta^+_{k\ell})= \Exp\frac{\partial g_\ell^+}{\partial Q_{\ell}^-}(Q_{k\ell}^-+Q_{\ell}^0,P_{k,\ell-1}^++P_{\ell-1}^0,\wb\theta_{k\ell}^\pm)\\
    \MoveEqLeft=
    \Exp\frac{\partial}{\partial Q_{\ell}^-} \int \tfrac{p(z_\ell|z_{\ell-1})}{Z}\exp\big(-\frac{\wb\gamma_{k\ell}^-}2(z_\ell-Q_{\ell}^--Q_\ell^0)^2-\frac{\wb\gamma_{k,\ell-1}^+}2(z_{\ell-1}-P_{k,\ell-1}^+-P_{\ell-1}^0)^2\big)z_\ell dz_\ell dz_{\ell-1},
\end{align*}
for a normalizing factor $Z$. The last expectation above is with respect to the density of $(P_{\ell-1}^0,P_{k,\ell-1}^+,Q_{k\ell}^-)$ which are Gaussian and $Q_\ell^0=\phi_\ell(P_{\ell-1},\Xi_\ell)$. 
Exchanging the order of the integration and the partial derivative, gives the desired expression for $\mc E_\ell^+$.
\end{subsection}

% \begin{proof}[Proof of Lemma \ref{lem:special_case}]
% \end{proof}

%%%%%% END appendix/proof_main_result %%%%%%%%%%%%%

\bibliographystyle{IEEEtran}
\bibliography{bibl_mlvamp}

\end{document}